\documentclass[11pt]{article}

\usepackage[utf8]{inputenc}
\usepackage[T1]{fontenc}
\usepackage{newtxtext}                    
\usepackage[varbb]{newtxmath}             
\usepackage[letterpaper,margin=1.5in]{geometry}
\usepackage{amsmath}

\usepackage{amssymb}

\usepackage{amsthm}
\usepackage{algorithm}
\usepackage{algpseudocode}
\usepackage{graphicx}
\usepackage[breaklinks=true,colorlinks=true,linkcolor=blue!60!black,citecolor=green!50!black,urlcolor=blue!60!black]{hyperref}
\usepackage{booktabs}
\usepackage{xcolor}
\usepackage{listings}
\usepackage[numbers,sort&compress]{natbib}
\usepackage{caption}
\usepackage{subcaption}
\usepackage{float}
\usepackage{enumitem}
\usepackage{microtype}

\newtheorem{theorem}{Theorem}[section]
\newtheorem{definition}[theorem]{Definition}
\newtheorem{lemma}[theorem]{Lemma}
\newtheorem{proposition}[theorem]{Proposition}

\lstdefinestyle{lean}{
  language={},
  basicstyle=\ttfamily\small,
  keywordstyle=\bfseries\color{blue!70!black},
  commentstyle=\itshape\color{green!50!black},
  stringstyle=\color{red!60!black},
  morekeywords={def,theorem,structure,instance,where,let,if,then,else,fun,by,exact,inductive,deriving,import,open,return,do,for,end,match,with},
  morecomment=[l]{--},
  frame=single,
  framerule=0.4pt,
  rulecolor=\color{gray!50},
  backgroundcolor=\color{gray!5},
  breaklines=true,
  columns=fullflexible,
  keepspaces=true,
  showstringspaces=false,
  tabsize=2,
  xleftmargin=1em,
  framexleftmargin=0.5em,
  aboveskip=8pt,
  belowskip=8pt,
  literate={∈}{$\in$}1 {∀}{$\forall$}1 {∃}{$\exists$}1
           {≤}{$\leq$}1 {≥}{$\geq$}1 {→}{$\rightarrow$}1
           {⇒}{$\Rightarrow$}1 {∧}{$\wedge$}1 {∨}{$\vee$}1
           {¬}{$\neg$}1 {·}{$\cdot$}1,
}
\algrenewcommand\algorithmicrequire{\textbf{Input:}}
\algrenewcommand\algorithmicensure{\textbf{Output:}}

\begin{document}

\begin{center}
{\LARGE\bfseries Formally Verified Patent Analysis\\[4pt]
via Dependent Type Theory:\\[4pt]
Machine-Checkable Certificates from a\\[4pt]
Hybrid AI + Lean~4 Pipeline\par}
\vspace{14pt}
{\large George Koomullil}\\[4pt]
{\itshape Ascendr, Inc.}\\[2pt]
{\small \texttt{george@ascendr.ai}}\\[14pt]
{\small \today}
\end{center}

\begin{center}\textbf{Abstract}\end{center}
\begin{quote}\small\itshape
We present a formally verified framework for patent analysis, realized as a hybrid AI + Lean~4 pipeline in which the DAG-coverage core (Algorithm~1b) is fully machine-verified once bounded match scores are fixed, while freedom-to-operate, claim-construction sensitivity, cross-claim consistency, and doctrine-of-equivalents analyses are formalized at the specification level and supported by kernel-checked candidate certificates. Patent analysis is a high-stakes legal task that existing approaches address either by manual expert analysis (slow, non-scalable) or by machine learning/NLP methods (probabilistic, opaque, non-compositional). To the best of our knowledge, this is the first framework that applies interactive theorem proving based on dependent type theory to intellectual property analysis. Claims are encoded as directed acyclic graphs (DAGs) in Lean~4, match strengths are modeled as elements of a formally verified complete lattice, and confidence scores propagate through dependency structures using proven-correct monotone functions. We formalize five IP use cases (patent-to-product mapping, freedom-to-operate, claim construction sensitivity, cross-claim consistency, and doctrine of equivalents) via six algorithms (one main algorithm per use case plus a Kleene fixed-point alternative; Algorithm~1b is a shared helper variant of Algorithm~1), together with complexity analysis and correctness arguments whose verification status is classified per Table~\ref{tab:proofstatus}: structural lemmas, the coverage-core generator (Proposition~\ref{prop:gensound}), and the closed-path coverage identity $\texttt{coverage} = W_{\text{cov}}$ are machine-verified in Lean~4; higher-level algorithmic theorems for use cases other than Algorithm~1b remain informal proof sketches, and the corresponding proof-generation functions are architecturally mitigated (untrusted generators whose outputs are kernel-checked and subjected to a \texttt{sorry}-free axiom audit before acceptance; the doctrine-of-equivalents analysis carries additional layers of conditionality beyond the other use cases). The formal guarantees are conditional on the ML layer: they certify the mathematical correctness of all computations downstream of ML-produced scores, while the semantic accuracy of the scores themselves remains the responsibility of the ML layer. This conditional guarantee, correctness of computation given the inputs, is nonetheless qualitatively different from anything purely probabilistic systems can offer. A case study on a synthetic memory-module claim (no adjudicated outcome, no real prosecution history, no real product documentation) demonstrates the framework's application, including weighted coverage computation and claim-construction sensitivity analysis; validation against adjudicated cases is identified as future work.
\end{quote}

\section{Introduction}
\label{sec:intro}

Patent analysis sits at the intersection of natural language understanding, logical reasoning, and legal doctrine. The core task across all intellectual property (IP) use cases is to decompose a patent claim into atomic limitations, map each limitation to evidence, and draw logically valid conclusions about coverage, infringement, or validity. Despite the inherently formal structure of patent law (claims are conjunctions of limitations governed by well-defined legal rules such as the all-elements rule, the doctrine of equivalents, and prosecution history estoppel), no prior work has applied interactive theorem provers or formal verification techniques to this domain.

Current automated approaches fall into two categories, each with fundamental limitations:

\textit{Manual expert analysis} uses claim charts, tabular mappings of claim elements to product evidence. While legally authoritative, this process can require tens to hundreds of expert hours per claim-product pair~\cite{lefstin2012formal}, is non-reproducible (two experts may reach different conclusions), and cannot scale to portfolios of thousands of patents.

\textit{Machine learning and NLP approaches}, including recent LLM-based tools, compute probabilistic similarity scores between claim elements and product features. However, these systems provide no correctness guarantees (an ``87\% confidence'' score reflects training distribution, not truth), produce opaque reasoning chains, suffer from non-compositional confidence ($P(A \wedge B) \neq P(A) \cdot P(B)$ in general), and are vulnerable to adversarial manipulation and hallucination.

We propose a fundamentally different approach: encoding claim structure in a dependent type theory (Lean~4) and discharging the downstream reasoning obligations with an interactive theorem prover. The Curry-Howard correspondence, the mathematical equivalence between propositions and types, proofs and programs, serves as the conceptual inspiration for this work: it suggests that structural properties of claim analysis (dependency satisfaction, coverage bounds, consistency) can be framed as propositions and discharged by proof terms rather than estimated by statistical similarity. The technical development we actually present is narrower and more concrete than a literal ``infringement as proof term'' formalization: we encode the claim DAG and match-strength lattice in Lean~4, we define and argue about propagation, coverage, and classification algorithms over those structures, and we generate proof certificates whose validity the Lean~4 kernel independently type-checks. When a certificate type-checks and passes a \texttt{sorry}-free axiom audit, the downstream computations (dependency propagation, bounds, classification) are mathematically verified against the ML-supplied inputs; when it fails, the failure point is localized. The Curry-Howard lens motivates this decomposition but the paper's contribution is the concrete architecture, not a direct type-theoretic reduction of infringement itself.

An important clarification on scope: the proposed framework is a \textit{hybrid} pipeline. The ML/NLP layer (claim decomposition, text extraction, and match scoring) produces match scores using probabilistic methods and is subject to all the standard ML limitations: hallucination, adversarial vulnerability, and training data dependency. The formal verification layer (DAG encoding, lattice propagation, and certificate generation) takes these scores as unverified inputs and certifies that all downstream computations (dependency propagation, score composition, coverage bounds) are mathematically correct. The formal guarantees are therefore \textit{conditional}: they certify the correctness of computation given the inputs, not the correctness of the inputs themselves. This is analogous to a formally verified calculator: it guarantees the arithmetic is correct, while the accuracy of the input figures remains the user's responsibility. Despite this limitation, the conditional guarantee is qualitatively stronger than anything provided by purely probabilistic systems, and the explicit trust boundary makes the scope of the guarantee transparent.

Beyond the specific domain of patent analysis, this work demonstrates a general and transferable architectural principle with significant implications for how AI workflows are designed in high-stakes settings. Current AI pipelines for structured reasoning tasks (such as legal analysis, regulatory compliance, and clinical decision support) implement their entire workflow probabilistically, not because every component requires probabilistic reasoning, but because formal alternatives have not been developed. This paper demonstrates that a principled decomposition is possible: the components of an AI workflow that perform natural language understanding genuinely require ML, while the components that perform structured computation (dependency propagation, score composition, coverage bounds, consistency checking) can be replaced by formally verified equivalents without sacrificing the pipeline's ability to process natural language inputs. The result is not an AI system with a verification layer added on top, but a restructured workflow in which the trust boundary is pushed as far into the pipeline as possible, converting what would otherwise be probabilistic black-box computations into machine-checkable proof certificates. We believe this decomposition principle, identify which components of an AI workflow can be formalized, and replace them, represents a broadly applicable template for improving the trustworthiness of AI systems in domains where logical rigor is required.

\medskip\noindent\textbf{A concrete motivation: why kernel type-checking alone is not enough.} To motivate why this paper defines certificate validity with both a kernel type-check and a \texttt{sorry}-free axiom audit, consider a na\"{\i}ve proof generator that discharges every proof obligation with Lean~4's \texttt{sorry} primitive:

\pagebreak
\smallskip
\noindent\colorbox{orange!20}{\small\textsc{Lean 4: illustrative, not compiled (counter example sketch)}}\nopagebreak
\begin{lstlisting}
def fraudulentCertificate : ProofCertificate := {
  coverage  := 100,     -- asserts 100% with no basis
  p_acyclic := sorry,   -- no proof
  p_lattice := sorry,
  p_propag  := sorry,
  p_bounded := sorry,
  p_deps    := sorry
}
\end{lstlisting}
\noindent This term \textit{type-checks} in the Lean~4 kernel because \texttt{sorry} is implemented as a primitive axiom inhabiting any type; a kernel-only check would accept it. The paper's verification pipeline rejects such artifacts by additionally running a \texttt{\#print axioms} audit and requiring that the transitive axiom set be a subset of $\Omega = \{\texttt{propext}, \texttt{Classical.choice}, \texttt{Quot.sound}\}$ (Definition~\ref{def:validcert}). This two-step check (type-check plus axiom audit) is what we mean by ``machine-checkable'' throughout the paper. Section~\ref{sec:framework} develops the full treatment, including why the Lean~4 prover itself is untrusted while the kernel is trusted, and what the \texttt{native\_decide} axiom would look like in the audit output.

\subsection{Contributions}
\label{sec:contributions}

This paper makes the following contributions:
\begin{enumerate}[leftmargin=2em]
\item To the best of our knowledge, we present the first formalization of patent analysis in dependent type theory, encoding claims as DAGs with typed nodes and directed dependency edges in Lean~4.
\item We define match strength as a complete lattice whose \texttt{CompleteLattice} instance is machine-verified in Mathlib, and we establish lattice-theoretic monotonicity, idempotency, and bounded propagation properties for the framework's propagation operators (verification status per property: see Table~\ref{tab:proofstatus}; higher-level $W_{\text{cov}}$-level consequences are currently informal proof sketches).
\item We provide six formal algorithms with pseudocode, correctness arguments, and complexity analysis for five IP use cases (Algorithms~1--5 address the five use cases directly; Algorithm~6 is a Kleene fixed-point alternative to Algorithm~1, offering different convergence guarantees rather than a new use case). We also define Algorithm~1b (\textsc{WeightedDAGCoverageGivenScores}), a helper variant of Algorithm~1 that separates DAG propagation from score computation and is invoked by Algorithm~3 (construction sensitivity); it is not a distinct use case.
\item We introduce \textit{proof certificates} for patent analysis, machine-checkable Lean~4 proof terms, bundled with claim-to-evidence mappings (the traditional claim chart structure), that provide independently verifiable evidence of analysis correctness. Certificate validity requires both kernel type-checking and a \texttt{sorry}-free axiom audit (Section~\ref{sec:framework}), ensuring that no proof obligation is bypassed. For the coverage-core path of Algorithm~\ref{alg:dagcov-given}, we additionally provide a concrete Lean~4 certificate type, a compiled generator whose every proof field is discharged without \texttt{sorry}, and a machine-checked generator-soundness proposition (Proposition~\ref{prop:gensound}, Section~\ref{sec:closedpath}).
\item We present a restructured AI pipeline in which components that prior systems implement probabilistically (dependency propagation, score composition, coverage computation, and consistency checking) are replaced by formally verified equivalents, pushing the trust boundary further into the workflow than existing systems assume.
\item We demonstrate a general decomposition principle for AI workflows in high-stakes domains: natural language understanding genuinely requires ML, while structured computation over its outputs can be formally verified. This shows that not every component of an AI pipeline needs to remain probabilistic, a principle applicable beyond patent analysis to any domain requiring structured reasoning over natural language.
\end{enumerate}

\noindent\textbf{Scope limitation.} The formal guarantees presented in this paper (machine-verified structural lemmas plus informal proof sketches for higher-level algorithmic theorems; see Table~\ref{tab:proofstatus}) are conditional in two senses. First, they are conditional on the ML layer providing correct input scores, as discussed throughout. Second, with the exception of the coverage-core generator for Algorithm~\ref{alg:dagcov-given} (Section~\ref{sec:closedpath}, Proposition~\ref{prop:gensound}), the proof generation functions that construct actual Lean~4 proof certificates are designed for the prototype described in Section~\ref{sec:implementation} but are not formally specified or verified in this paper. Every certificate is independently checked by the Lean~4 kernel, which catches any generation failure. Providing complete formal specifications of the generation functions is a primary direction for future work.

\medskip
\begin{center}
\fbox{\parbox{0.92\textwidth}{\small\itshape
\noindent\textbf{Verification-status ledger.} Every formal claim in this paper is qualified by its verification-status classification in Table~\ref{tab:proofstatus}: \emph{machine-verified} (kernel-checked in Lean~4), \emph{informal sketch} (mathematically rigorous argument describing the intended Lean~4 proof but not compiled), \emph{architecturally enforced} (encoded in type signatures), or \emph{architecturally mitigated} (not formally specified but constrained by the trust model; see Section~\ref{sec:framework}). Readers should consult Table~\ref{tab:proofstatus} before interpreting any ``theorem'', ``proof certificate'', or correctness claim elsewhere in the paper.
}}
\end{center}

\begin{table}[H]
\centering
\caption{Verification status of formal results.}
\label{tab:proofstatus}
\small
\begin{tabular}{@{}p{4.2cm}p{2.0cm}p{6.0cm}@{}}
\toprule
\textbf{Result} & \textbf{Status} & \textbf{Notes} \\
\midrule
\texttt{dag\_acyclic} (Section~\ref{sec:dagencoding}) & Machine-verified & Lean \texttt{decide} tactic; exhaustive check \\
Lattice axioms (Section~\ref{sec:framework}) & Machine-verified & Mathlib \texttt{CompleteLattice} instance \\
$M$ boundedness (Def.~\ref{def:matchscore}) & Architecturally enforced & Type signature $M : T \times T \to [0,1]$ encodes boundedness; only $M$-property used in Theorems~\ref{thm:algcorrect}--\ref{thm:convergence} \\
$M$ symmetry, identity (Def.~\ref{def:matchscore}) & Desirable, not required & Hold in the normal case; identity can fail for zero TF-IDF vector; not used in any proof \\
\texttt{coverage\_in\_range} & Machine-verified & Both endpoints of the $[0,100]$ bound; uses \texttt{ScoreValid} and $0 \leq \theta$ only (Appendix~\ref{app:lean}) \\
Coverage-semantic equality ($\texttt{weightedCoverage}(\beta,\theta) = W_{\text{cov}}$) & Machine-verified by definitional unfolding & \texttt{weightedCoverage} unfolds to the weighted average of $\texttt{computeEff}(\beta,\theta)$; certificate's \texttt{p\_deps} field discharged by \texttt{rfl} \\
Theorem~\ref{thm:weakest} (Weakest Link) & Informal sketch & Proof sketch describes Lean proof structure \\
Theorem~\ref{thm:monotone} (Monotonicity) & Informal sketch & Proof sketch describes Lean proof structure \\
Theorem~\ref{thm:algcorrect}(i) (Alg.~1 bounds) & Machine-verified via \texttt{coverage\_in\_range} & For Algorithm~\ref{alg:dagcov-given}; Alg.~\ref{alg:dagcov} inherits (Proposition~\ref{prop:gensound}) \\
Theorem~\ref{thm:algcorrect}(ii)--(iv) and Theorems~\ref{thm:fto}--\ref{thm:convergence} (Alg.\ correctness, parts beyond bounds) & Informal sketch & Mathematical arguments; not compiled \\
\texttt{generateCertificate} (Alg.~\ref{alg:dagcov-given}, \S\ref{sec:closedpath}) & Machine-verified & Compiled certificate type and generator for Algorithm~\ref{alg:dagcov-given} coverage core; Def.~\ref{def:validcert}(ii) discharged by $\Omega$-audit (Appendix~\ref{app:axiomaudit}) \\
Remaining proof gen.\ functions (Algs.~2--6) & Architecturally mitigated & Not formally specified; kernel type-checks every certificate (Def.~\ref{def:validcert}) \\
\texttt{sorry}-free guarantee & Architecturally mitigated & Enforced by \texttt{\#print axioms} audit \\
\bottomrule
\end{tabular}
\end{table}

\section{Background and Related Work}
\label{sec:related}

\subsection{Patent Claim Structure and Legal Doctrine}

A patent claim consists of three parts: a preamble setting the technical field, a transitional phrase (``comprising,'' ``consisting of,'' or ``consisting essentially of'') determining openness of scope, and a body listing the limitations (elements) of the invention. Under the \textit{all-elements rule}, infringement requires that the accused product satisfy every limitation of the claim, formally, a universal conjunction over all elements. Non-infringement requires demonstrating the absence of at least one limitation, which by De Morgan's laws is an existential statement.

Claim construction, the interpretation of claim terms, is governed by \textit{Phillips v.\ AWH Corp.}~\cite{phillips2005}. The doctrine of equivalents extends infringement to non-literal matches satisfying the function-way-result test (\textit{Graver Tank}~\cite{gravertank1950}), subject to prosecution history estoppel (\textit{Festo}~\cite{festo2002}). These legal rules have precise logical structure that naturally lends itself to formalization.

\subsection{Formal Verification and Dependent Type Theory}

Interactive theorem provers (ITPs) such as Lean~4~\cite{demoura2021lean4}, Rocq (formerly Coq), and Isabelle/HOL provide languages for stating mathematical propositions and constructing machine-checkable proofs. Lean~4 implements the Calculus of Inductive Constructions (CIC), a dependent type theory in which types may depend on values. The Curry-Howard correspondence ensures that propositions are types and proofs are programs: if a term $t$ of type $P$ type-checks, then $t$ is a valid proof of proposition $P$.

Lean~4 has been adopted for industrial verification: AWS uses it to formally verify the Cedar authorization policy language~\cite{aws2024cedar}, achieving sub-millisecond authorization latency with proven soundness and completeness. Harmonic AI~\cite{harmonic2025} uses Lean~4 to formally verify AI outputs. The Mathlib library~\cite{mathlib2025} formalizes over 210,000 theorems, including comprehensive lattice theory and fixed-point theorems relevant to our approach.

\subsection{NLP and ML for Patent Analysis}

Commercial tools (Relecura, Patsnap, Clarivate, Patlytics, XLSCOUT, IPRally, Questel) and academic work employ NLP for patent analysis. The FLAN Graph model~\cite{gao2024flan} builds claim dependency graphs but feeds them into graph neural networks for approval prediction, not formal verification. PatentScore~\cite{patentscore2025} uses LLM-based multi-dimensional scoring. All existing approaches produce probabilistic outputs without formal correctness guarantees.

\subsection{Formal Methods in Legal Reasoning}

Catala~\cite{merigoux2021catala} compiles statutory tax law using F*. L4L~\cite{chen2025l4l} combines LLMs with SMT solvers for criminal law. The LogiKEy workbench~\cite{benzmuller2024logikey} provides deontic logic in Isabelle/HOL. Reis et al.~\cite{reis2018contributions} used ALEN Description Logic to model patent claims but without machine-checkable proofs or weighted scoring. None of this prior work applies interactive theorem provers to patent analysis.

\subsection{The Prior Art Gap}

A review of the prior literature reveals four non-overlapping silos, summarized in Table~\ref{tab:priorartsilos}; the present work bridges all four for the first time.

\begin{table}[H]
\centering
\caption{Four non-overlapping silos in the prior literature and representative entries.}
\label{tab:priorartsilos}
\small
\begin{tabular}{@{}p{3.3cm}p{3.9cm}p{5.2cm}@{}}
\toprule
\textbf{Silo} & \textbf{What it does} & \textbf{Representative entries} \\
\midrule
(A) Formal methods for general law (not patents) & ITP/SMT/deontic formalization of tax, criminal, or contract law &
Catala~\cite{merigoux2021catala} (statutory tax in F*); LogiKEy~\cite{benzmuller2024logikey} (deontic logic in Isabelle/HOL); L4L~\cite{chen2025l4l} (LLM + SMT for criminal law); Lawsky~\cite{lawsky2018logic}; Grimmelmann~\cite{grimmelmann2022programming}. \\
(B) Description logic for patents (not theorem provers) & ALEN/DL encodings of claim structure, no machine-checkable proofs &
Reis et al.~\cite{reis2018contributions} (ALEN DL for claim modeling). \\
(C) NLP/AI patent analysis (no formal methods) & LLMs, embeddings, graph models for retrieval, matching, scoring &
FLAN-Graph~\cite{gao2024flan}; PatentScore~\cite{patentscore2025}; PatentMatch~\cite{risch2021patentmatch}; DAPFam~\cite{dapfam2025}; Casadio et al.~\cite{casadio2024nlp}; Sadowski et al.~\cite{sadowski2025verifiable}. \\
(D) Lean~4 for non-legal applications & Industrial / mathematical uses of Lean without IP content &
AWS Cedar~\cite{aws2024cedar}; Mathlib~\cite{mathlib2025}; Harmonic AI~\cite{harmonic2025}; Lean4Lean~\cite{carneiro2024lean4lean}; Park et al.~\cite{park2024dslean}. \\
\bottomrule
\end{tabular}
\end{table}

\noindent Silos (A), (B), (C), and (D) are individually well-populated, but none covers the intersection: a machine-checkable coverage certificate for a patent claim, produced by a hybrid ML + Lean~4 pipeline, is absent from all four silos to the best of our knowledge.

\section{Problem Formulation}
\label{sec:formulation}

\subsection{Formal Definitions}

\noindent\textbf{Notation guide (read before Definition~\ref{def:textspaces}).} The paper uses several pairs of symbols that share a base letter but denote distinct objects; we summarize them here so readers do not have to infer the distinction from math alphabets alone.
\begin{center}
{\small
\begin{tabular}{@{}p{1.9cm}p{4.2cm}p{6.6cm}@{}}
\toprule
\textbf{Base letter} & \textbf{Meaning 1 (plain)} & \textbf{Meaning 2 (alternate alphabet or subscript)} \\
\midrule
$T$ / $\mathcal{T}$ & $T$: master text space (Def.~\ref{def:textspaces}) & $\mathcal{T}$: finite set of node types (Def.~\ref{def:claimdag}) \\
$E$ / $\mathcal{E}$ & $E$: DAG edge set (Def.~\ref{def:claimdag}) & $\mathcal{E}$: evidence text space (Def.~\ref{def:textspaces}) \\
$I$ / $\mathcal{I}$ & $I_j$: a claim construction (Problem~3) & $\mathcal{I}$: implementation-description text space (Def.~\ref{def:textspaces}) \\
$L$ / $\mathcal{L}$ / $\mathbb{L}$ & (not used as a plain symbol) & $\mathcal{L}$: claim language text space (Def.~\ref{def:textspaces}); $\mathbb{L}$: match-strength lattice (\S\ref{sec:framework}) \\
\bottomrule
\end{tabular}%
}
\end{center}
\noindent In every subsequent definition, theorem, and algorithm, these distinctions are preserved by font (plain vs.\ $\mathcal{X}$ vs.\ $\mathbb{X}$) or by subscript ($I$ alone never appears without the subscript $j$; node-type $\tau(v) \in \mathcal{T}$ is always calligraphic).

\begin{definition}[Text Spaces]
\label{def:textspaces}
We distinguish three disjoint spaces of text representations, all subsets of a master space $T$ of all finite Unicode text strings:
\begin{itemize}[leftmargin=2em, topsep=2pt, itemsep=1pt]
\item $\mathcal{L} \subset T$: the space of \textit{claim language texts}, texts appearing in patent claims, amendments, and prosecution arguments. Claim elements $V \subset \mathcal{L}$ and all prosecution amendment texts live here.
\item $\mathcal{I} \subset T$: the space of \textit{implementation description texts}, texts describing concrete technical implementations of inventions. Elements of $\Phi_v$ (Definition~\ref{def:scopespace}) live here.
\item $\mathcal{E} \subset T$: the space of \textit{evidence texts}, segments extracted from product documentation. Evidence segments $S \subset \mathcal{E}$ live here.
\end{itemize}
$\mathcal{L}$, $\mathcal{I}$, and $\mathcal{E}$ are pairwise disjoint by construction in the formal framework: each text admitted to the framework carries a role tag (claim language, implementation description, or evidence), and the three spaces are the corresponding role-tagged subsets. A given Unicode string may appear with different tags in different contexts (a datasheet phrase borrowed from claim language does not collapse the two categories) because disjointness is a property of the tagged representation, not of the underlying string. The master space $T$ continues to carry the untagged strings for type-theoretic uniformity of operations such as $M : T \times T \to [0,1]$, which acts on the underlying strings and is blind to the tag. Texts in $\mathcal{L}$ describe what a claim covers. Texts in $\mathcal{I}$ describe how an invention could be implemented. Texts in $\mathcal{E}$ describe what an accused product does.
\end{definition}

\medskip\noindent\textbf{Running example (3-node prefix of the case study).} To keep the abstract definitions in this section concrete, we carry through a minimal three-node example that is a prefix of the 15-node memory-module DAG used in the full case study (Section~\ref{sec:implementation}, Figure~\ref{fig:claimdag}):
\begin{itemize}[leftmargin=2em, topsep=2pt, itemsep=1pt]
\item $C_1$ (\textit{preamble}, ``printed circuit board''), weight $w(\textit{preamble}) = 0.3$, no dependencies;
\item $C_2$ (\textit{structural}, ``a plurality of DDR memory devices mounted to the printed circuit board''), weight $w(\textit{structural}) = 1.0$, depends on $C_1$;
\item $C_3$ (\textit{quantitative}, ``organized into a first number of ranks''), weight $w(\textit{quantitative}) = 2.0$, depends on $C_2$.
\end{itemize}
Each of $C_1, C_2, C_3 \in \mathcal{L}$ (claim language texts). Each subsequent definition in this subsection is annotated with how the running example instantiates it; full numerical values for this three-node subgraph are taken from Table~\ref{tab:scores} (Appendix~\ref{app:lean}, memory-module case study, construction $I_1$).

\begin{definition}[Patent Claim DAG]
\label{def:claimdag}
A patent claim is represented as a DAG $G = (V, E, \tau, w)$ where $V = \{C_1, \ldots, C_n\}$ is the finite set of atomic claim limitations ($|V| \geq 1$), $E \subseteq V \times V$ is a set of directed edges representing logical dependencies, $\tau : V \to \mathcal{T}$ maps each node to a type, where $\mathcal{T} = \{\textit{preamble}, \textit{structural}, \textit{functional}, \textit{quantitative}, \textit{wherein}, \textit{coupling}, \textit{signal}\}$ is the set of element types, and $w : \mathcal{T} \to \mathbb{R}^+$ assigns positive weights based on legal significance. The dependency function $\textit{deps} : V \to \mathcal{P}(V)$ maps each node to its set of direct dependencies as encoded in $E$: $\textit{deps}(v) = \{u \in V \mid (u,v) \in E\}$, i.e., edges point from dependency to dependent. Since $w(\tau(v)) > 0$ for all $v \in V$, the total weight $\sum_{v \in V} w(\tau(v)) > 0$. (Note: the Lean~4 encoding in Section~\ref{sec:dagencoding} reverses this orientation for implementation convenience; see the edge orientation convention discussion there.)

\textit{Running example.} $V = \{C_1, C_2, C_3\}$, $E = \{(C_1, C_2), (C_2, C_3)\}$, $\tau(C_1) = \textit{preamble}$, $\tau(C_2) = \textit{structural}$, $\tau(C_3) = \textit{quantitative}$, and so $w(\tau(C_1)) = 0.3$, $w(\tau(C_2)) = 1.0$, $w(\tau(C_3)) = 2.0$. The dependency function yields $\textit{deps}(C_1) = \emptyset$, $\textit{deps}(C_2) = \{C_1\}$, $\textit{deps}(C_3) = \{C_2\}$; the total weight is $0.3 + 1.0 + 2.0 = 3.3$.
\end{definition}

\begin{definition}[Evidence Set]
\label{def:evidence}
Given a product document $D$, the evidence set $S = \{s_1, \ldots, s_m\}$ with $m \geq 1$ is a non-empty, canonically ordered collection of text segments obtained by extracting and segmenting $D$ into semantically coherent chunks of bounded length. The canonical ordering on $S$ (e.g., by document position) is used for tie-breaking in subsequent definitions.

\textit{Running example.} For the memory-module case study, $S$ consists of segments extracted from a product datasheet (for instance, board-layout, memory-topology, and control-signal sections), each $s_i \in \mathcal{E}$. The concrete values of $\beta(C_j) = \max_{s \in S} M(C_j, s)$ for the three-node subgraph under construction $I_1$ are $\beta(C_1) = 0.90$, $\beta(C_2) = 0.87$, $\beta(C_3) = 0.82$ (Table~\ref{tab:scores}).
\end{definition}

\begin{definition}[Text Representation Spaces]
\label{def:textspace}
We work with the three distinct subspaces of the master text space $T$ established in Definition~\ref{def:textspaces}. Claim elements $V \subset \mathcal{L}$ are texts in the claim language space. Evidence segments $S \subset \mathcal{E}$ are texts in the evidence space. Implementation descriptions $\Phi_v \subset \mathcal{I}$ (Definition~\ref{def:scopespace}) are texts in the implementation space. The match score function $M : T \times T \to [0,1]$ is defined over the master space for uniformity, but in practice is applied as:
\begin{itemize}[leftmargin=2em, topsep=2pt, itemsep=1pt]
\item $M : \mathcal{L} \times \mathcal{E} \to [0,1]$ for claim-to-evidence matching (Definition~\ref{def:matchscore})
\item $M_{\text{scope}} : \mathcal{I} \times \mathcal{E} \to [0,1]$ for implementation-to-evidence matching (Definition~\ref{def:scopespace})
\end{itemize}
These are the same underlying function with potentially different configuration parameters, as described in Section~\ref{sec:architecture}.

\textit{Forward pointer.} A companion function $\textit{ann} : V \to \mathcal{P}(\Sigma)$ (Definition~\ref{def:annotation}) annotates each node $v \in V$ with the technical terms from a shared vocabulary $\Sigma$ that appear in $v$'s claim-language text in $\mathcal{L}$. $\textit{ann}$ is needed only for the cross-claim consistency problem (Problem~4) and is defined at the end of the present subsection after the DOE apparatus that precedes it; readers focused on Problems~1--3 may skip ahead to Section~\ref{sec:formulation} subsection ``Five IP Analysis Problems'' without loss.
\end{definition}

\begin{definition}[Match Score Function]
\label{def:matchscore}
The match score function $M : T \times T \to [0,1]$ computes the similarity between two text representations from the master space $T$. For claim-to-evidence matching, $M$ restricts to $M : \mathcal{L} \times \mathcal{E} \to [0,1]$. One property is formally required; two further properties are desirable for well-behaved implementations but not required by any downstream theorem.
\begin{itemize}[leftmargin=2em, topsep=2pt, itemsep=1pt]
\item \textit{Boundedness (required):} $M(t_1, t_2) \in [0,1]$ for all $t_1, t_2 \in T$. This is the only property of $M$ on which the formal guarantees of Theorems~\ref{thm:algcorrect}--\ref{thm:convergence} depend.
\item \textit{Symmetry (desirable):} $M(t_1, t_2) = M(t_2, t_1)$. Holds when the underlying similarity measure uses a shared corpus for IDF weight computation, so that both arguments receive identical weights. Not required by any downstream theorem.
\item \textit{Identity (desirable):} $M(t, t) = 1$ for all $t \in T$. Holds for standard implementations whenever both the lexical and semantic components return~1 for identical inputs. As noted in Section~\ref{sec:architecture}, the concrete implementation can produce $M(t,t) = 1 - \alpha \neq 1$ when a text's TF-IDF vector is zero while its embedding vector is non-zero, an edge case that is unlikely for patent text but cannot be excluded unconditionally. Because identity is not required by any downstream theorem, this edge case does not affect the formal guarantees.
\end{itemize}
$M$ is provided as an input to the formal framework and lies below the trust boundary. Its concrete implementation (combining lexical and semantic similarity) is described in Section~\ref{sec:architecture}.
\end{definition}

\begin{definition}[Best Match and Effective Score]
\label{def:bestmatch}
\label{def:effcov}
For each limitation $v \in V$, the best match score is $\beta(v) = \max_{s \in S} M(v, s)$, which is well-defined since $S$ is non-empty and finite (Definition~\ref{def:evidence}). The dependency satisfaction threshold $\theta \in (0,1]$ is a system parameter. The effective score $\textit{eff} : V \to [0,1]$ is defined by structural induction on the topological order of $G$:
\[
\textit{eff}(v) = \begin{cases} \beta(v) & \text{if } \forall u \in \textit{deps}(v),\; \textit{eff}(u) \geq \theta \\ 0 & \text{else} \end{cases}
\]
\textbf{Well-foundedness.} This is a recursive definition: computing $\textit{eff}(v)$ requires knowing $\textit{eff}(u)$ for every $u \in \textit{deps}(v)$. For such a definition to be unambiguous (meaning $\textit{eff}(v)$ has a unique, determinate value for every node $v$), the recursion must terminate rather than loop. If the graph had a cycle (say $v$ depends on $u$ and $u$ depends on $v$), then computing $\textit{eff}(v)$ would require $\textit{eff}(u)$, which would require $\textit{eff}(v)$, and neither would ever be defined. Termination is guaranteed by the acyclicity of $G$: because $G$ is a DAG, its nodes can be arranged in a topological order in which every node appears after all of its dependencies. Computing $\textit{eff}$ in this order means that when we reach node $v$, the values $\textit{eff}(u)$ for all $u \in \textit{deps}(v)$ have already been computed. There is no circularity. This topological order exists if and only if $G$ is acyclic, which is established by the machine-verified theorem \texttt{dag\_acyclic} in Section~\ref{sec:framework}. The topological order is computed explicitly in Algorithm~\ref{alg:dagcov} (\texttt{TopologicalSort}, line~7) and forms the basis of the propagation in all subsequent algorithms. All other recursive definitions in this paper that traverse the dependency structure (including $\textit{eff}_{\text{DOE}}$ (Definition~\ref{def:doeeff}) and $\textit{claimStrength}$ (Definition~\ref{def:claimstrength})) are well-founded by the same argument.

Checking $\textit{eff}(u) \geq \theta$ (rather than $\beta(u) \geq \theta$) ensures transitive dependency enforcement: if any ancestor of $v$ in the DAG fails threshold, $v$'s effective score is also zeroed. In the context of DOE analysis (Problem~5), Definition~\ref{def:doeeff} introduces a separate propagation threshold $\theta_{\text{prop}} \in \{\theta_{\text{lit}}, \theta_{\text{eq}}\}$ for DOE-adjusted dependency enforcement; $\theta$ in this definition governs the standard analysis (Algorithm~\ref{alg:dagcov}), while $\theta_{\text{prop}}$ governs Algorithm~\ref{alg:doe}'s Phase~2.

\smallskip\noindent\textit{Remark (sub-threshold but non-zero effective scores).} $\textit{eff}(v)$ can be strictly positive yet below $\theta$. This occurs when $v$'s own dependencies are satisfied (so Definition~\ref{def:bestmatch} does not zero $\textit{eff}(v)$) but $\beta(v) < \theta$. Such a node contributes $\beta(v)$ to $W_{\text{cov}}$ through Definition~\ref{def:depscov}, yet simultaneously fails the dependency check for any dependent $w$ with $v \in \textit{deps}(w)$, causing $\textit{eff}(w) = 0$. The memory-module case study in Section~\ref{sec:algorithms} exhibits exactly this pattern under construction $I_2$: $\textit{eff}(C_3, I_2) = 0.58 > 0$ is retained (because $C_3$'s dependency $C_2$ is met) but triggers cascade zeroing in $C_{11}, C_{12}, C_{13}$ because $0.58 < \theta = 0.65$.

\smallskip\noindent\textit{Running example.} Using the best-match values $\beta(C_1) = 0.90$, $\beta(C_2) = 0.87$, $\beta(C_3) = 0.82$ from the example in Definition~\ref{def:evidence} with $\theta = 0.65$: $C_1$ has no dependencies, so $\textit{eff}(C_1) = \beta(C_1) = 0.90$; $\textit{deps}(C_2) = \{C_1\}$ and $\textit{eff}(C_1) = 0.90 \geq \theta$, so $\textit{eff}(C_2) = \beta(C_2) = 0.87$; $\textit{deps}(C_3) = \{C_2\}$ and $\textit{eff}(C_2) = 0.87 \geq \theta$, so $\textit{eff}(C_3) = \beta(C_3) = 0.82$. All three effective scores equal their best-match scores because every dependency in this prefix meets threshold.
\end{definition}

\begin{definition}[Dependency Satisfaction and Weighted Coverage]
\label{def:depscov}
A limitation $v$ has its dependencies satisfied, written $\textit{deps\_met}(v)$, iff $\forall u \in \textit{deps}(v),\; \textit{eff}(u) \geq \theta$, where $\textit{eff}$ is from Definition~\ref{def:bestmatch}. The weighted DAG coverage is:
\[
W_{\text{cov}} = \frac{\sum_{v \in V} w(\tau(v)) \cdot \textit{eff}(v)}{\sum_{v \in V} w(\tau(v))} \times 100
\]
The denominator is strictly positive by Definition~\ref{def:claimdag}. Since $0 \leq \textit{eff}(v) \leq 1$ and $w(\tau(v)) > 0$, we have $0 \leq W_{\text{cov}} \leq 100$.

\smallskip\noindent\textit{Running example.} With the effective scores computed above and weights $w(\tau(C_1)) = 0.3$, $w(\tau(C_2)) = 1.0$, $w(\tau(C_3)) = 2.0$:
\[
W_{\text{cov}} \;=\; \frac{0.3 \cdot 0.90 + 1.0 \cdot 0.87 + 2.0 \cdot 0.82}{0.3 + 1.0 + 2.0} \times 100 \;=\; \frac{2.78}{3.3} \times 100 \;\approx\; 84.2\%,
\]
which is a contribution of the three-node subgraph to the larger case study's coverage computation (the full 15-node analysis yields $82.3\%$ under $I_1$, Section~\ref{sec:algorithms}).
\end{definition}

\begin{definition}[Claim Scope Space and Evidence Projection]
\label{def:scopespace}
\textbf{(a) Scope Space.} For each element $v \in V$, the claim scope space $\Phi_v$ is a finite, non-empty set of implementation description texts:
\[
\Phi_v \subset \mathcal{I}, \quad |\Phi_v| \geq 1, \quad |\Phi_v| < \infty
\]
where each $\phi \in \Phi_v$ represents a distinct concrete technical implementation that could plausibly satisfy element $v$ under some reasonable claim construction. The finiteness of $\Phi_v$ is guaranteed by imposing a word-length bound: $\Phi_v \subseteq \{t \in \mathcal{I} \mid |t| \leq L\}$ for a fixed bound $L$ (e.g., $L = 1000$ words). Since $\mathcal{I}$ is over a finite alphabet and texts are length-bounded, this set is finite, and any $\Phi_v$ satisfying the bound is automatically finite. Non-emptiness is a precondition on inputs to the formal framework, analogous to $S \neq \emptyset$. Since $\Phi_v$ is finite and non-empty, all argmax operations over $\Phi_v$ are well-defined, with ties broken by an arbitrary but fixed ordering on $\Phi_v$ (analogous to the canonical ordering on $S$ from Definition~\ref{def:evidence}).

\noindent\textbf{Trust boundary for $\Phi_v$ (important).} $\Phi_v$ is one of the single largest risk inputs to the DOE analysis, and its quality is \textit{below the trust boundary}. The formal framework makes no requirement on how $\Phi_v$ is constructed; it may be provided by human experts, structured technology knowledge bases, or automated tools, and the formal guarantees hold for any valid instantiation satisfying the preconditions above. But the framework also cannot detect three failure modes that affect $\Phi_v$ specifically: scope underrepresentation, scope misclassification, and evidence mismatch (see Limitation~2 in Section~\ref{sec:discussion} for the detailed failure-mode analysis and mitigations). In the prototype of Section~\ref{sec:implementation}, $\Phi_v$ is generated by the ML/NLP layer; under either ML or expert construction, expert review of $\Phi_v$ and its prosecution scope partitions $\Phi_v^{\text{orig},k}, \Phi_v^{\text{amend},k}$ before formal encoding is the primary mitigation.

\textbf{(b) Prosecution Scope Regions.} When element $v$ has been amended $K_v$ times during prosecution, the scope regions for amendment $k$ are subsets of $\Phi_v$:
\[
\Phi_v^{\text{amend},k} \subseteq \Phi_v^{\text{orig},k} \subseteq \Phi_v, \quad \text{for } k = 1, \ldots, K_v
\]
where $\Phi_v^{\text{orig},k} \subseteq \Phi_v$ is the subset of implementations covered by the original (pre-amendment $k$) claim language, and $\Phi_v^{\text{amend},k} \subseteq \Phi_v$ is the subset covered by the amended (post-amendment $k$) claim language. The surrendered subject matter for amendment $k$ is:
\[
\Phi_v^{\text{orig},k} \setminus \Phi_v^{\text{amend},k}
\]
which is the set of implementations that were within scope before amendment $k$ but are no longer within scope after it. The inclusion $\Phi_v^{\text{amend},k} \subseteq \Phi_v^{\text{orig},k}$ is required because amendments are narrowing, the amended language covers a subset of what the original language covered. When $v$ has a single amendment, the index $k$ is dropped. The classification of each $\phi \in \Phi_v$ into scope regions requires claim construction judgment and is provided as an input to the formal framework. In the prototype, this classification is performed by the ML/NLP layer (below the trust boundary).

\textbf{(c) Evidence Projection.} The evidence projection function $\pi_v : \mathcal{E} \to \Phi_v$ maps each evidence segment to the closest implementation description in $\Phi_v$:
\[
\pi_v(s) = \arg\max_{\phi \in \Phi_v} M_{\text{scope}}(\phi, s)
\]
where $M_{\text{scope}} : \mathcal{I} \times \mathcal{E} \to [0,1]$ is a scope-specific match score function satisfying the same axioms as $M$ (Definition~\ref{def:matchscore}) but potentially configured differently to reflect the greater importance of precise terminology in defining claim scope boundaries. When no scope-specific configuration is needed, $M_{\text{scope}} = M$. Since $\Phi_v$ is finite and non-empty, this argmax is always well-defined, with ties broken by the fixed ordering on $\Phi_v$. The projection asks: ``which of the candidate implementations of $v$ does this product evidence most closely resemble?'' When $v$ has no prosecution history, $\pi_v$ and the scope regions are not needed and are left undefined.
\end{definition}

\noindent\textbf{Remark (Claim Language vs.\ Implementation Descriptions).} It is essential to distinguish two conceptually different kinds of text in the DOE analysis:
\begin{itemize}[leftmargin=2em, topsep=2pt, itemsep=1pt]
\item The \textit{claim language} of element $v$ (in $\mathcal{L}$) is the text of the patent claim as written, e.g., ``a rank translation circuit configured to map a second number of ranks to a first number of ranks.'' It defines the scope of the element.
\item An \textit{implementation description} $\phi \in \Phi_v$ (in $\mathcal{I}$) is a text describing a concrete technical artifact, e.g., ``a hardware ASIC with a hardwired rank address decoder using combinational logic.'' It describes a specific realization that may or may not fall within the scope.
\end{itemize}
$\Phi_v$ is a set of the second kind, not the first kind. The claim language $v$ defines what $\Phi_v$'s elements should look like, it is the criterion for membership, but $v$ itself is not a member of $\Phi_v$ because $v$ and the elements of $\Phi_v$ carry different role tags by construction ($v$ is tagged as claim language, elements of $\Phi_v$ as implementation descriptions; see Definition~\ref{def:textspaces}), even when, as a string, the claim language phrasing happens to coincide with one of the implementation descriptions. This distinction is critical for the prosecution history estoppel analysis: the surrendered territory $\Phi_v^{\text{orig},k} \setminus \Phi_v^{\text{amend},k}$ is a set of concrete implementations, not a set of claim language paraphrases. Only by keeping these conceptually separate can the set difference have meaningful legal content.

\begin{definition}[Prosecution History]
\label{def:prosecution}
The prosecution history $\textit{PH} = (\textit{PH}_{\text{amend}},\; \textit{PH}_{\text{arg}})$ is a pair of sets:
\begin{itemize}[leftmargin=2em, topsep=2pt, itemsep=1pt]
\item $\textit{PH}_{\text{amend}} = \{(v_k, A_k^{\text{orig}}, A_k^{\text{amend}}, R_k)\}$ is the set of narrowing amendments, where $v_k \in V$ is the amended element (note: the same $v_k$ may appear in multiple entries if amended more than once), $A_k^{\text{orig}} \in \mathcal{L}$ is the original claim language, $A_k^{\text{amend}} \in \mathcal{L}$ is the amended claim language, and $R_k \in \{\textit{Patentability}, \S112, \textit{Other}\}$ is the reason for amendment. Under \textit{Festo Corp.\ v.\ Shoketsu}~\cite{festo2002}, all narrowing amendments trigger the presumption of prosecution history estoppel regardless of $R_k$. The reason $R_k$ determines which of Festo's three rebuttal exceptions (unforeseeability, tangentiality, other reason) may apply. For $R_k = \textit{Other}$, the scope of available rebuttal is unclear under current case law, so the presumption is treated conservatively as irrebuttable for such amendments.
\item $\textit{PH}_{\text{arg}} = \{(v_k, \textit{arg}_k, \textit{scope}_k)\}$ is the set of prosecution arguments, where $\textit{arg}_k \in \mathcal{L}$ is the argument text and $\textit{scope}_k \subseteq \Phi_{v_k}$ (Definition~\ref{def:scopespace}) is the scope disclaimed by the argument.
\end{itemize}
$\textit{PH}_{\text{amend}}$ is derived from the formal prosecution record and admits exact formalization. $\textit{PH}_{\text{arg}}$ requires semantic parsing of natural language text and lies below the trust boundary. Dedication to the public (disclosed-but-unclaimed embodiments) is identified as a future extension (see Section~\ref{sec:discussion}).
\end{definition}

\noindent\textbf{Remark (Computing Disclaimed Scope from Prosecution Arguments).} The disclaimed scope $\textit{scope}_k$ for each prosecution argument $\textit{arg}_k$ requires determining, from the natural language text of the argument, which concrete implementations in $\Phi_{v_k}$ were given up. This is among the most difficult tasks in the entire pipeline: it requires reading a legal argument, understanding its technical context, interpreting what it implies about claim scope, and mapping that implication onto the finite set $\Phi_{v_k}$. This task is substantially harder than claim decomposition or evidence matching, as it involves multi-step legal reasoning over prosecution history text. In the prototype, $\textit{scope}_k$ is computed by the ML/NLP layer (below the trust boundary). The formal layer certifies only that $\textit{ER}_{\text{arg}}$ is computed correctly given $\textit{scope}_k$ as provided; it cannot certify that $\textit{scope}_k$ correctly captures what the prosecution argument disclaimed. This conditionality should be understood as a significant limitation of the argument-based estoppel analysis.

\begin{definition}[Festo Rebuttal Predicate]
\label{def:festo_rebutted}
The Festo rebuttal predicate, reflecting the rebuttable presumption from \textit{Festo Corp.\ v.\ Shoketsu}~\cite{festo2002}, is
\begin{multline*}
\textit{Festo\_rebutted}(v, s, k, \textit{PH}) = \textit{Unforeseeable}(v, s, k, \textit{PH}) \\
\vee\; \textit{Tangential}(v, s, k, \textit{PH}) \;\vee\; \textit{AlternativeJustification}(v, s, k, \textit{PH})
\end{multline*}
The three component predicates are abstract user-supplied inputs to the formal layer with the following type signatures:
\begin{itemize}[leftmargin=2em, topsep=2pt, itemsep=1pt]
\item $\textit{Unforeseeable} : V \times S \times \mathbb{N} \times \textit{PH} \to \{\textit{true}, \textit{false}\}$
\item $\textit{Tangential} : V \times S \times \mathbb{N} \times \textit{PH} \to \{\textit{true}, \textit{false}\}$
\item $\textit{AlternativeJustification} : V \times S \times \mathbb{N} \times \textit{PH} \to \{\textit{true}, \textit{false}\}$ \hfill (the \emph{Festo} \textit{other reason} exception)
\end{itemize}
We adopt the predicate name $\textit{AlternativeJustification}$ in place of the legal term ``other reason'' to avoid collision with the amendment reason $R_k = \textit{Other}$ (Definition~\ref{def:prosecution}); the two concepts are unrelated. $R_k$ records the stated purpose of an amendment, while $\textit{AlternativeJustification}$ is a rebuttal predicate invoked \textit{after} an amendment triggers the Festo presumption. All three predicates require $k$ (the amendment index) because the rebuttal determination is specific to the rationale and context of each individual amendment: an equivalent that was unforeseeable given the state of the art at the time of amendment $k_1$ may have been foreseeable at the time of a later amendment $k_2$. Per Definition~\ref{def:prosecution}, when $R_k = \textit{Other}$ the presumption is treated as irrebuttable, and the user-supplied predicates are required to satisfy $\textit{Unforeseeable}(v,s,k,\textit{PH}) = \textit{Tangential}(v,s,k,\textit{PH}) = \textit{AlternativeJustification}(v,s,k,\textit{PH}) = \textit{false}$ for all $s$ when the applicable amendment has $R_k = \textit{Other}$, ensuring $\textit{Festo\_rebutted}$ returns $\textit{false}$ (no rebuttal available) in that case. Correctness of these predicates with respect to the \textit{Festo} legal standard is assumed and lies below the trust boundary; the formal layer certifies only that $\textit{Festo\_rebutted}$ is computed correctly from them.
\end{definition}

\begin{definition}[Amendment-Based Estopped Region]
\label{def:er_amend}
The amendment-based estopped region for element $v \in V$ under prosecution history $\textit{PH}_{\text{amend}}$ is
\begin{multline*}
\textit{ER}_{\text{amend}}(v, \textit{PH}_{\text{amend}}) = \bigl\{s \in S \;\big|\; \exists k: v_k = v \\
\wedge\; \pi_v(s) \in \Phi_v^{\text{orig},k} \setminus \Phi_v^{\text{amend},k}
\;\wedge\; \neg\textit{Festo\_rebutted}(v, s, k, \textit{PH})\bigr\}
\end{multline*}
where $\Phi_v^{\text{orig},k}$ and $\Phi_v^{\text{amend},k}$ are the prosecution scope regions for the $k$-th amendment to $v$ (Definition~\ref{def:scopespace}(b)), $\pi_v$ is the evidence projection (Definition~\ref{def:scopespace}(c)), and $\textit{Festo\_rebutted}$ is the predicate of Definition~\ref{def:festo_rebutted}. The condition handles multiple amendments to the same element via the existential over $k$.
\end{definition}

\begin{definition}[Argument-Based Estopped Region]
\label{def:er_arg}
The argument-based estopped region for element $v \in V$ under prosecution history $\textit{PH}_{\text{arg}}$ is
\[
\textit{ER}_{\text{arg}}(v, \textit{PH}_{\text{arg}}) = \{s \in S \mid \exists (v_k, \textit{arg}_k, \textit{scope}_k) \in \textit{PH}_{\text{arg}},\; v_k = v:\; \pi_v(s) \in \textit{scope}_k\}
\]
The condition $v_k = v$ restricts argument-based estoppel to prosecution arguments made specifically for element $v$, consistent with the element-specific scope used in Definition~\ref{def:er_amend}.
\end{definition}

\begin{definition}[Estopped Region]
\label{def:estoppel}
For each $v \in V$, the estopped region is the union of the amendment-based and argument-based components:
\[
\textit{EstoppedRegion}(v, \textit{PH}) = \textit{ER}_{\text{amend}}(v, \textit{PH}_{\text{amend}}) \cup \textit{ER}_{\text{arg}}(v, \textit{PH}_{\text{arg}})
\]
(Definitions~\ref{def:er_amend} and~\ref{def:er_arg}). Where $\textit{PH}$ contains no entry for $v$ under any component, $\textit{EstoppedRegion}(v, \textit{PH}) = \emptyset$; this default assumes absence of prosecution history implies no estoppel, a legal assumption whose correctness must be verified by the practitioner. Dedication to the public is not modeled and is identified as a future extension (Section~\ref{sec:discussion}). The formal layer certifies only that $\textit{EstoppedRegion}$ is computed correctly given the three components as inputs.
\end{definition}

\noindent\textbf{Remark (No Rebuttal for Argument-Based Estoppel).} The current formalization treats argument-based estoppel as absolute, no rebuttal mechanism analogous to $\textit{Festo\_rebutted}$ is provided for $\textit{ER}_{\text{arg}}$. This reflects one view in the case law, where argument-based estoppel has sometimes been treated as irrebuttable because the patentee voluntarily made the disclaimer without examiner compulsion. However, this is a contested legal question. Some courts have permitted limited rebuttal of argument-based estoppel where the argument was ambiguous or not clearly directed at the accused equivalent. Practitioners relying on this framework should be aware that the absolute treatment of argument-based estoppel may be overly conservative in jurisdictions or contexts where rebuttal is permitted. Adding a rebuttal mechanism for argument-based estoppel analogous to $\textit{Festo\_rebutted}$ is identified as future work.

\begin{definition}[Functional Decomposition]
\label{def:funcdecomp}
For each claim element $v \in V \subset \mathcal{L}$, the parsing functions $\textit{func}, \textit{way}, \textit{res} : \mathcal{L} \to \mathcal{L}$ extract the functional, way (manner), and result aspects of $v$ respectively, returning texts in the claim language space $\mathcal{L}$. Since $\textit{func}(v) \in \mathcal{L} \subset T$ and $s \in \mathcal{E} \subset T$, the application $M(\textit{func}(v), s) : T \times T \to [0,1]$ is type-correct. For element types where functional decomposition is not semantically meaningful (e.g., preamble), $\textit{func}(v) = \textit{way}(v) = \textit{res}(v) = v$ by convention. These parsing functions are abstract and lie below the trust boundary; their concrete computation is described in Section~\ref{sec:architecture}.
\end{definition}

\begin{definition}[Function-Way-Result Similarity]
\label{def:fwr}
The function, way, and result similarity measures $f_{\text{sim}}, w_{\text{sim}}, r_{\text{sim}} : V \times \mathcal{E} \to [0,1]$ are defined as:
\[
f_{\text{sim}}(v, s) = M(\textit{func}(v), s), \quad w_{\text{sim}}(v, s) = M(\textit{way}(v), s), \quad r_{\text{sim}}(v, s) = M(\textit{res}(v), s)
\]
where $M : T \times T \to [0,1]$ is the match score function from Definition~\ref{def:matchscore}, applied to parsed text representations from Definition~\ref{def:funcdecomp}. Since $\textit{func}(v) \in \mathcal{L} \subset T$ and $s \in \mathcal{E} \subset T$, the application $M(\textit{func}(v), s)$ is type-correct. These functions inherit the probabilistic nature of $M$ and lie below the trust boundary. The formal verification layer certifies threshold comparisons on these scores but cannot certify that they correctly measure the legal concepts of function, way, and result respectively.
\end{definition}

\begin{definition}[DOE Thresholds]
\label{def:doethresholds}
The literal match threshold $\theta_{\text{lit}} \in (0,1]$ is the minimum score above which element $v$ is considered literally satisfied. The equivalence threshold $\theta_{\text{eq}} \in (0, \theta_{\text{lit}})$ is the minimum score required on each prong of the function-way-result test; the strict inequality $\theta_{\text{eq}} > 0$ ensures that zero-similarity is never sufficient for equivalence. The constraint $\theta_{\text{eq}} < \theta_{\text{lit}}$ reflects the legal principle that equivalence is a weaker standard than literal satisfaction. Both thresholds are system parameters configurable by the user. Note: a uniform $\theta_{\text{eq}}$ across all three prongs and all element types is a simplifying assumption; alternative formulations may use separate thresholds $\theta_f, \theta_w, \theta_r$ per prong or per element type $\tau(v)$. The relationship between the dependency satisfaction threshold $\theta$ (Definition~\ref{def:bestmatch}, used in Algorithms~1, 2, and~6) and $\theta_{\text{lit}}$ should be explicitly configured. When DOE analysis is not enabled, $\theta$ and $\theta_{\text{lit}}$ are independent parameters. When DOE analysis is enabled, setting $\theta = \theta_{\text{lit}}$ ensures consistent dependency enforcement across the standard and DOE analyses and is the recommended configuration.
\end{definition}

\begin{definition}[Best Non-Estopped Evidence and Vitiation]
\label{def:bestevidence}
For each $v \in V$ with $\beta(v) < \theta_{\text{lit}}$, let $S_v = S \setminus \textit{EstoppedRegion}(v, \textit{PH})$ (Definition~\ref{def:estoppel}). The best non-estopped evidence segment is:
\[
s^*(v) = \begin{cases} \arg\max_{s \in S_v} M(v, s) & \text{if } S_v \neq \emptyset \\ \text{undefined} & \text{if } S_v = \emptyset \end{cases}
\]
with ties broken by the canonical ordering on $S$ (Definition~\ref{def:evidence}). The vitiation predicate $\textit{vitiated} : V \times S \to \{\textit{true}, \textit{false}\}$ is defined as:
\[
\textit{vitiated}(v, s) \;\Longleftrightarrow\; M(v, s) < \theta_{\text{vit}}
\]
where $\theta_{\text{vit}} \in [0, \theta_{\text{eq}})$ is a vitiation threshold below which the accused element is considered absent rather than substituted, consistent with the all-elements rule under \textit{Warner-Jenkinson}~\cite{warnerjenkinson1997}. This prevents DOE from being used to read out an entire claim limitation. The default $\theta_{\text{vit}} = 0$ disables the vitiation check: since $M(v,s) \geq 0$ by Definition~\ref{def:matchscore}, the strict inequality $M(v,s) < 0$ is never satisfied, so $\textit{vitiated}$ is vacuously false on every pair (the strict $<$ rather than $\leq$ is chosen precisely so that $\theta_{\text{vit}} = 0$ acts as a no-op sentinel). Both $s^*(v)$ and $\textit{vitiated}$ depend on ML-computed scores and lie below the trust boundary.
\end{definition}

\begin{definition}[Match Type]
\label{def:matchtype}
$\textsc{MatchType}$ is an inductive type with constructors ordered $\textsc{NoMatch} < \textsc{Equivalent} < \textsc{Literal}$. Classification follows a precedence rule: if $\beta(v) \geq \theta_{\text{lit}}$, classify as $\textsc{Literal}$ regardless of whether $\textsc{Equivalent}$ conditions are also met. The constructors and their proof obligations are:
\begin{itemize}[leftmargin=2em, topsep=2pt, itemsep=1pt]
\item $\textsc{Literal}(\beta(v))$ requires $\beta(v) \geq \theta_{\text{lit}}$
\item $\textsc{Equivalent}(s^*(v), f_{\text{sim}}(v, s^*), w_{\text{sim}}(v, s^*), r_{\text{sim}}(v, s^*))$ requires $\beta(v) < \theta_{\text{lit}}$, $s^*(v)$ defined (Definition~\ref{def:bestevidence}), all three prongs $\geq \theta_{\text{eq}}$, and $\neg\textit{vitiated}(v, s^*(v))$. (The non-estoppel condition $s^*(v) \notin \textit{EstoppedRegion}(v, \textit{PH})$ is automatically satisfied since $s^*(v)$ is selected from $S_v = S \setminus \textit{EstoppedRegion}$.)
\item $\textsc{NoMatch}$ has no proof obligations
\end{itemize}
Note: $\textsc{MatchType}$ forms a finite totally ordered set (and hence a distributive lattice with $\wedge = \min$, $\vee = \max$), but the propagation machinery operates on the complete lattice $\mathbb{L} = [0,1]$, not on $\textsc{MatchType}$ directly. The mapping $\textit{eff}_{\text{DOE}}$ (Definition~\ref{def:doeeff}) embeds each classification back into $\mathbb{L}$, where all lattice-theoretic properties (monotonicity, bounded propagation, fixed-point convergence) apply.
\end{definition}

\begin{definition}[DOE-Adjusted Effective Score]
\label{def:doeeff}
The DOE-adjusted effective score $\textit{eff}_{\text{DOE}} : V \to [0,1]$ incorporates both match classification and dependency satisfaction:
\[
\textit{eff}_{\text{DOE}}(v) = \begin{cases}
0 & \text{if } \neg\textit{deps\_met\_DOE}(v) \\[3pt]
\beta(v) & \text{if } \textit{deps\_met\_DOE}(v) \;\wedge\; \textit{match}(v) = \textsc{Literal} \\[3pt]
\delta \cdot \min(f_{\text{sim}},\, w_{\text{sim}},\, r_{\text{sim}}) & \text{if } \textit{deps\_met\_DOE}(v) \;\wedge\; \textit{match}(v) = \textsc{Equivalent} \\[3pt]
0 & \text{if } \textit{deps\_met\_DOE}(v) \;\wedge\; \textit{match}(v) = \textsc{NoMatch}
\end{cases}
\]
where the similarity arguments are evaluated at $(v, s^*(v))$, i.e., $f_{\text{sim}} = f_{\text{sim}}(v, s^*(v))$ and likewise for $w_{\text{sim}}$ and $r_{\text{sim}}$;
where $\delta \in (0,1]$ is a discount factor reflecting the weaker legal status of equivalent matches (default $\delta = 1$). Dependency enforcement under DOE is governed by a propagation threshold $\theta_{\text{prop}} \in \{\theta_{\text{lit}}, \theta_{\text{eq}}\}$ with default $\theta_{\text{prop}} := \theta_{\text{lit}}$; see the Note on Phase~2 threshold interaction following Theorem~\ref{thm:doe} for the tradeoff between conservative dependency enforcement ($\theta_{\text{prop}} = \theta_{\text{lit}}$) and maximized DOE coverage ($\theta_{\text{prop}} = \theta_{\text{eq}}$). The DOE-adjusted dependency satisfaction is:
\[
\textit{deps\_met\_DOE}(v) \;\Longleftrightarrow\; \forall u \in \textit{deps}(v),\; \textit{eff}_{\text{DOE}}(u) \geq \theta_{\text{prop}}
\]
This recursive definition is well-founded by the same argument as Definition~\ref{def:bestmatch}: the acyclicity of $G$, established by \texttt{dag\_acyclic}, guarantees a topological evaluation order in which all dependencies are computed before their dependents. The DOE-adjusted weighted coverage is:
\[
W_{\text{DOE}} = \frac{\sum_{v \in V} w(\tau(v)) \cdot \textit{eff}_{\text{DOE}}(v)}{\sum_{v \in V} w(\tau(v))} \times 100
\]
Since $0 \leq \textit{eff}_{\text{DOE}}(v) \leq 1$ and the denominator is strictly positive, $0 \leq W_{\text{DOE}} \leq 100$. Note that $W_{\text{DOE}}$ and $W_{\text{cov}}$ are not directly comparable in general: an element with $\beta(v) \in [\theta_{\text{eq}}, \theta_{\text{lit}})$ and $\textit{deps\_met}(v)$ contributes $\beta(v) \cdot w(\tau(v)) > 0$ to $W_{\text{cov}}$, but under DOE may be classified as $\textsc{NoMatch}$ (contributing~$0$) if the FWR test fails, if all non-estopped evidence is removed ($S_v = \emptyset$), or if the element is vitiated. Conversely, DOE can increase coverage by reclassifying such elements as $\textsc{Equivalent}$ when all conditions are met. The relationship $W_{\text{DOE}} \geq W_{\text{cov}}$ holds only under the additional assumption that for every $v$ with $\beta(v) < \theta_{\text{lit}}$, either $S_v \neq \emptyset$ and the FWR test passes with $\delta \cdot \min(f_{\text{sim}}, w_{\text{sim}}, r_{\text{sim}}) \geq \beta(v)$, or $\textit{eff}(v) = 0$ under the standard analysis.
\end{definition}

\begin{definition}[Node Text Annotation]
\label{def:annotation}
Given a patent claim DAG $G = (V, E, \tau, w)$ (Definition~\ref{def:claimdag}) and a shared technical vocabulary $\Sigma$, the annotation function
\[
\textit{ann} : V \to \mathcal{P}(\Sigma)
\]
maps each node $v \in V$ to the subset of $\Sigma$ consisting of technical terms that appear in the claim language text of $v$:
\[
\textit{ann}(v) = \{t \in \Sigma \mid t \text{ appears in the claim language text of } v\},
\]
where the claim language text of $v$ is the element of $\mathcal{L} \subset T$ (Definition~\ref{def:textspace}) corresponding to the atomic limitation represented by $v$, \textit{not} its type label $\tau(v) \in \mathcal{T}$. The type label records the grammatical and legal category of the limitation (one of seven values in $\mathcal{T}$); the annotation records its technical vocabulary content. These are entirely separate objects. We lift the annotation to the graph level by $\textit{terms}(G) = \bigcup_{v \in V} \textit{ann}(v)$, the set of all shared-vocabulary terms appearing anywhere in $G$.

Term identification within claim text is performed by the LLM-assisted extraction pipeline (below the trust boundary). The formal layer certifies consistency of extracted interpretations given $\textit{ann}$; it cannot certify that $\textit{ann}$ correctly identifies all relevant technical terms. In the Lean~4 encoding, $\textit{ann}$ is an abstract function provided as a parameter, analogous to $M$ in Definition~\ref{def:matchscore}.
\end{definition}

\subsection{Five IP Analysis Problems}

Based on these definitions, we formalize five distinct problems:

\textbf{Problem 1 (Patent-to-Product Mapping).} Given claim DAG $G$ and evidence set $S$, compute $W_{\text{cov}}$ and generate a proof certificate attesting to the score and its derivation.

\textbf{Problem 2 (Freedom-to-Operate).} Given $G$ and $S$, either construct a formal proof that no assignment $A : V \to S$ can achieve full coverage (Definition~\ref{def:fullcoverage}), or report a gap analysis identifying the weakest elements and constraining dependency paths.

\textbf{Problem 3 (Claim Construction Sensitivity).} Given $G$, $S$, and constructions $\{I_1, \ldots, I_k\}$, determine which constructions yield coverage and identify determinative terms.

\textbf{Problem 4 (Cross-Claim Consistency).} Given portfolio $\{G_1, \ldots, G_p\}$ with shared vocabulary $\Sigma$, and interpretation function $f : \{1,\ldots,p\} \times \Sigma \to \mathsf{Int}$ where $f(i, t)$ denotes the meaning assigned to term $t$ within claim DAG $G_i$ and $\mathsf{Int}$ is the abstract set of possible interpretations (deliberately distinct from the match-score symbol $M$ of Definition~\ref{def:matchscore}), verify interpretation consistency: $\forall i,j.\; \forall t \in \Sigma.\; \text{uses}(G_i, t) \wedge \text{uses}(G_j, t) \Rightarrow f(i,t) = f(j,t)$. Here $\text{uses}(G_i, t)$ holds iff technical term $t$ appears in the claim text of at least one element $v \in V_i$:
\[
\text{uses}(G_i, t) \;\Leftrightarrow\; \exists v \in V_i:\; t \in \textit{ann}_i(v),
\]
where $\textit{ann}_i : V_i \to \mathcal{P}(\Sigma)$ is the annotation function of Definition~\ref{def:annotation}. This is well-defined because each $v \in V_i$ corresponds to a specific claim limitation text in $\mathcal{L} \subset T$, and $\textit{ann}_i(v)$ records the shared-vocabulary terms identified within that text. Note that $\textit{ann}_i(v)$ is \textit{not} derived from the type label $\tau_i(v) \in \mathcal{T}$: the type label records the grammatical/legal category of the limitation, while the annotation records its technical vocabulary content. $\textit{ann}_i$ lies below the trust boundary.

\textbf{Problem 5 (Doctrine of Equivalents).} Given claim DAG $G = (V, E, \tau, w)$, evidence set $S$, prosecution history $\textit{PH}$ (Definition~\ref{def:prosecution}), thresholds $\theta_{\text{lit}}, \theta_{\text{eq}}$ with $\theta_{\text{eq}} < \theta_{\text{lit}}$ (Definition~\ref{def:doethresholds}), propagation threshold $\theta_{\text{prop}} \in \{\theta_{\text{lit}}, \theta_{\text{eq}}\}$ with default $\theta_{\text{prop}} := \theta_{\text{lit}}$ (Definition~\ref{def:doeeff}), similarity functions $f_{\text{sim}}, w_{\text{sim}}, r_{\text{sim}}$ (Definition~\ref{def:fwr}), and claim construction $I$ (from Problem~3, treated as fixed input):

For each $v \in V$, let $S_v = S \setminus \textit{EstoppedRegion}(v, \textit{PH})$ (Definition~\ref{def:estoppel}). If $S_v = \emptyset$, set $\textit{match}(v) = \textsc{NoMatch}$. Otherwise let $s^*(v) = \arg\max_{s \in S_v} M(v,s)$, breaking ties by canonical ordering on $S$. Classify $\textit{match}(v) \in \textsc{MatchType}$ (Definition~\ref{def:matchtype}) and compute $\textit{eff}_{\text{DOE}}(v)$ (Definition~\ref{def:doeeff}).

Output: (i) $\textit{match} : V \to \textsc{MatchType}$ classifying each element; (ii) DOE-adjusted weighted coverage $W_{\text{DOE}}$; (iii) proof certificate $\Pi$ attesting to threshold comparisons, non-vitiation conditions, estoppel constraint application, bounds on $W_{\text{DOE}}$, and propagation correctness for $\textit{eff}_{\text{DOE}}$ through the DAG. Note: $\Pi$ certifies mathematical correctness of classifications given the input scores but does not certify legal correctness of $f_{\text{sim}}, w_{\text{sim}}, r_{\text{sim}}$ as measures of function, way, and result (see Section~\ref{sec:discussion}).

\textit{Legal completeness note.} The current formalization captures amendment-based and argument-based estoppel (Definition~\ref{def:estoppel}, combining Definitions~\ref{def:er_amend} and~\ref{def:er_arg} over the prosecution-history record of Definition~\ref{def:prosecution}) and the function-way-result test (\textit{Graver Tank}). The Festo rebuttable presumption is present syntactically via the three rebuttal predicates $\textit{Unforeseeable}$, $\textit{Tangential}$, and $\textit{AlternativeJustification}$ (Definition~\ref{def:festo_rebutted}); what is \textit{missing} is the formal derivation and verification of these predicates from prosecution history and the state of the art, which are currently treated as user-supplied inputs below the trust boundary. Three aspects are modeled as future extensions: (i)~dedication to the public (\textit{Johnson \& Johnston}~\cite{johnsonjohnston2002}); (ii)~formal derivation of the Festo rebuttal predicates rather than accepting them as user-supplied inputs; (iii)~the insubstantial differences test as an alternative to function-way-result. Analyses relying on the current formalization should be understood as incomplete with respect to these doctrines.

\section{System Architecture}
\label{sec:architecture}

The proposed framework employs a hybrid pipeline combining machine learning for natural language processing with formal verification for correctness certification. Figure~\ref{fig:architecture} illustrates the architecture.

\begin{figure}[H]
\centering
\includegraphics[width=\textwidth]{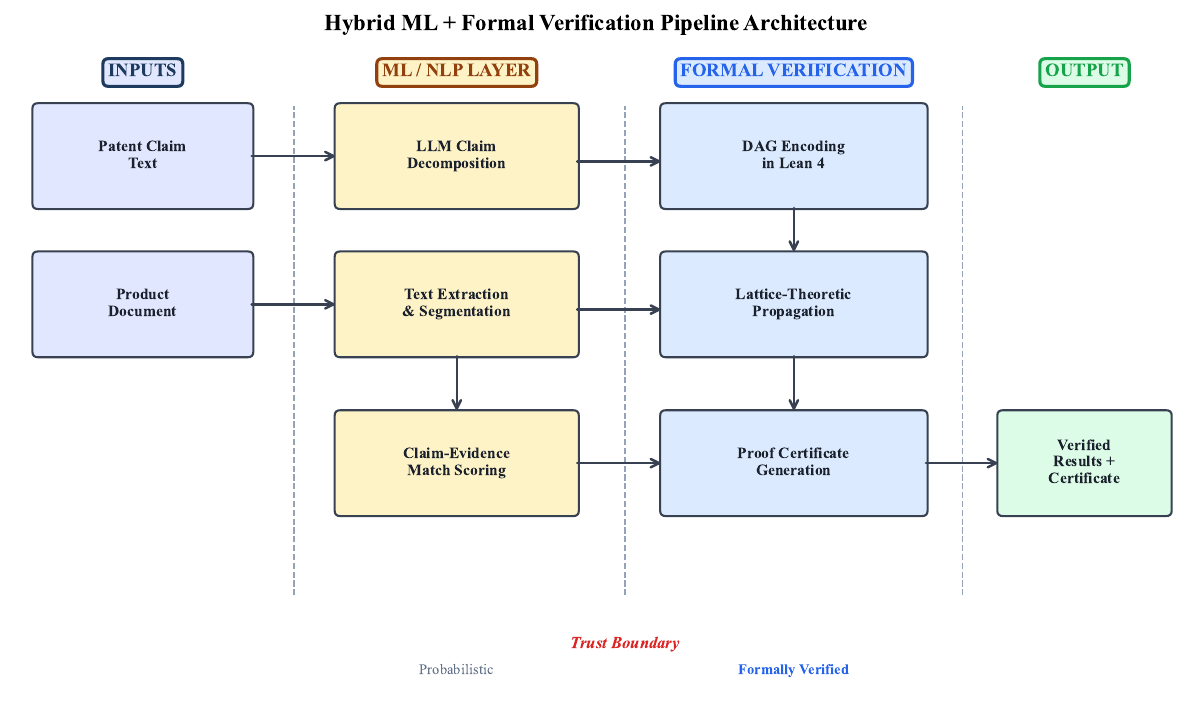}
\caption{Hybrid ML + Formal Verification Pipeline Architecture. The trust boundary separates probabilistic ML components (left) from formally verified components (right). For \textit{proof validity}, only the Lean~4 kernel (${\sim}5{,}000$ lines of C++) must be trusted. \textit{End-to-end legal or semantic correctness}, by contrast, additionally depends on the ML-layer outputs, $\Phi_v$ construction, functional decomposition, prosecution-history scope extraction, and other unverified inputs below the trust boundary.}
\label{fig:architecture}
\end{figure}

The pipeline consists of two layers separated by a trust boundary (Figure~\ref{fig:architecture}). The ML/NLP layer handles claim decomposition, text extraction and segmentation, and claim-evidence match scoring using standard techniques including LLM-based decomposition, TF-IDF, and embedding-based matching. The formal verification layer encodes the claim DAG in Lean~4, performs lattice-theoretic propagation of effective scores, and generates the proof certificate.

\subsection{Match Score Implementation}
The abstract match score function $M : T \times T \to [0,1]$ (Definition~\ref{def:matchscore}) is implemented in the claim-evidence match scoring component as a convex combination of lexical and semantic similarity:
\[
M(t_1, t_2) = \alpha \cdot M_{\text{lex}}(t_1, t_2) + (1-\alpha) \cdot M_{\text{sem}}(t_1, t_2)
\]
where $M_{\text{lex}}$ computes TF-IDF cosine similarity using a unified vocabulary built from the combined corpus of claim text and product documentation (ensuring symmetric IDF weights), $M_{\text{sem}}$ computes embedding cosine similarity (using $d$-dimensional dense vectors), and $\alpha \in [0,1]$ is a tunable fusion parameter. For scope-specific matching (Definition~\ref{def:scopespace}), a separate parameter $\alpha_{\text{scope}}$ may be configured to weight lexical similarity more heavily, reflecting the precision of legal claim language. Both $M_{\text{lex}}$ and $M_{\text{sem}}$ lie below the trust boundary. The computational cost of evaluating $M$ once is $O(d)$ where $d$ is the embedding dimension. The fused score $M$ satisfies the one formally required property of Definition~\ref{def:matchscore}: \textit{boundedness} holds because $M$ is a convex combination of values in $[0,1]$, unconditionally for all inputs. The two desirable properties hold in the normal case: \textit{symmetry} follows because cosine similarity is symmetric and the shared corpus ensures identical IDF weights for both arguments; \textit{identity} holds whenever both components return~1 for identical inputs, which is the case for non-empty text with in-vocabulary terms. However, when a text's TF-IDF vector is zero (all terms out of vocabulary) while its embedding vector is non-zero, the implementation yields $M(t,t) = 1 - \alpha \neq 1$. This violates identity but not boundedness, so the formal guarantees of Theorems~\ref{thm:algcorrect}--\ref{thm:convergence} are unaffected.

\subsection{Trust Model}
The system's trust model follows \textit{de Bruijn's criterion}~\cite{barendregt2005challenge}, a foundational principle in formal verification: separate proof \textit{generation} (complex, potentially buggy, untrusted) from proof \textit{checking} (simple, small, trusted). Just as verifying a completed Sudoku is trivial while solving one may require complex heuristics, checking a formal proof requires only a small, trusted kernel, regardless of how the proof was constructed.

The system establishes a clear trust boundary (Figure~\ref{fig:architecture}, dashed line). The ML/NLP components and proof generator (the claim decomposition, text extraction, and match scoring components, together with the Lean~4 prover) are \textit{untrusted}, their outputs are treated as unverified inputs. The only trusted component is the Lean~4 proof checker kernel, which has been independently self-verified by the Lean4Lean project~\cite{carneiro2024lean4lean} and consists of approximately 5,000 lines of C++.

\textit{Scope of the ``only the kernel is trusted'' claim.} This reduction applies to \textbf{proof validity}: whether a proof term delivered to the kernel in fact establishes the stated proposition in the Calculus of Inductive Constructions. For that question, the system's trustworthiness reduces to trusting the kernel, a reduction in trusted computing base of several orders of magnitude compared to trusting an entire ML pipeline. It does \textbf{not} mean that end-to-end legal or semantic correctness reduces to kernel trust. End-to-end correctness additionally depends on: (i)~ML-supplied match scores $M(v,s)$ being semantically faithful to the claim language and accused product; (ii)~the claim decomposition $(V, E, \tau, w)$ faithfully reflecting the patent claim; (iii)~$\Phi_v$ and its prosecution scope regions $\Phi_v^{\text{orig},k}, \Phi_v^{\text{amend},k}$ being correctly constructed; (iv)~functional, way, and result parses $\textit{func}, \textit{way}, \textit{res}$ correctly extracting the DOE prongs; and (v)~argument-based disclaimed scope $\textit{scope}_k$ being correctly computed from prosecution argument text. All of these lie below the trust boundary and are the responsibility of the ML/NLP layer or human domain experts. The proof certificate certifies that all downstream computation over these inputs is mathematically correct; it does not certify the inputs themselves.

\noindent\textbf{Trust Boundary and $\Phi_v$.} The claim scope space $\Phi_v$ is defined mathematically as a finite non-empty set of implementation descriptions (Definition~\ref{def:scopespace}). The formal framework makes no requirement on how $\Phi_v$ is constructed, it treats $\Phi_v$ as a given input. In the prototype implementation, $\Phi_v$ is generated by the ML/NLP layer, placing its construction below the trust boundary. However, in principle $\Phi_v$ could be provided by human domain experts or structured knowledge bases, in which case the trust boundary for $\Phi_v$ would be determined by the trustworthiness of those sources rather than the ML layer. The formal guarantees are identical regardless of how $\Phi_v$ is constructed, as long as the preconditions (non-emptiness, finiteness, correct scope region classifications) are satisfied.

\medskip
\noindent\textbf{Guarantee map (summary).} The following consolidated map is referenced throughout the paper; later sections that reiterate the trust-boundary placement of particular objects are pointers to this table rather than independent treatments.

\smallskip
\begin{center}
{\small
\begin{tabular}{@{}p{3.2cm}p{5cm}p{5cm}@{}}
\toprule
\textbf{Object} & \textbf{Status} & \textbf{Where treated in detail} \\
\midrule
Claim DAG $(V, E, \tau, w)$ & Below trust boundary (LLM/expert) & Section~\ref{sec:discussion} Limitation~1 \\
Match scores $\beta(v)$ & Below trust boundary (ML) & Section~\ref{sec:architecture}; Def.~\ref{def:matchscore} \\
Propagated $\textit{eff}$ and $W_{\text{cov}}$ & Above trust boundary (formal; Alg.~\ref{alg:dagcov-given}) & Section~\ref{sec:closedpath}, Prop.~\ref{prop:gensound} \\
Coverage certificate $\Pi$ (Alg.~\ref{alg:dagcov-given}) & Above trust boundary (machine-verified, $\Omega$-audited) & Section~\ref{sec:closedpath}, App.~\ref{app:closedpath}, App.~\ref{app:axiomaudit} \\
Certificates for Algs.~2--6 & Architecturally mitigated (untrusted generator; kernel-checked output) & Table~\ref{tab:proofstatus}; Section~\ref{sec:discussion} Limitation~7 \\
$\Phi_v$ and scope regions & Below trust boundary (ML/expert) & Def.~\ref{def:scopespace}; Section~\ref{sec:discussion} Limitation~2 \\
$\textit{func},\textit{way},\textit{res}$, $f_{\text{sim}}, w_{\text{sim}}, r_{\text{sim}}$ & Below trust boundary (ML) & Def.~\ref{def:funcdecomp}--\ref{def:fwr} \\
Festo rebuttal predicates; $\textit{scope}_k$ & Below trust boundary (legal expert / ML) & Def.~\ref{def:festo_rebutted}; Section~\ref{sec:discussion} Limitation~6 \\
Annotation $\textit{ann}$ & Below trust boundary (LLM) & Def.~\ref{def:annotation} \\
Weight scheme $w(\tau(v))$ & Below trust boundary (doctrinal, unvalidated) & Section~\ref{sec:discussion} Limitation~3 \\
\bottomrule
\end{tabular}%
}
\end{center}

\subsection{Instantiation of Formal Inputs}
\label{sec:instantiation}

Several objects in the formal framework (Section~\ref{sec:formulation}) are defined mathematically without specifying their construction:
\begin{itemize}[leftmargin=2em, topsep=2pt, itemsep=1pt]
\item $\Phi_v$: the claim scope space for each element $v$
\item $\Phi_v^{\text{orig},k}$, $\Phi_v^{\text{amend},k}$: scope regions for prosecution history analysis
\item $\textit{scope}_k$: the scope disclaimed by prosecution argument $\textit{arg}_k$
\item $\textit{func}(v)$, $\textit{way}(v)$, $\textit{res}(v)$: functional decomposition of claim elements
\end{itemize}
In the prototype, each of these is instantiated by the ML/NLP layer using LLM-based generation and classification. This is an implementation choice that lies below the trust boundary. Practitioners deploying the framework in other contexts may instantiate these objects differently (through human expert review, structured patent knowledge bases, or domain-specific classifiers) without affecting the validity of the formal guarantees. The formal layer certifies correctness of computation given these inputs; the quality of the inputs determines the legal meaningfulness of the results.

\section{Formal Framework}
\label{sec:framework}

\subsection{Claim DAG Encoding in Dependent Type Theory}
\label{sec:dagencoding}

\begin{figure}[H]
\centering
\includegraphics[width=\textwidth]{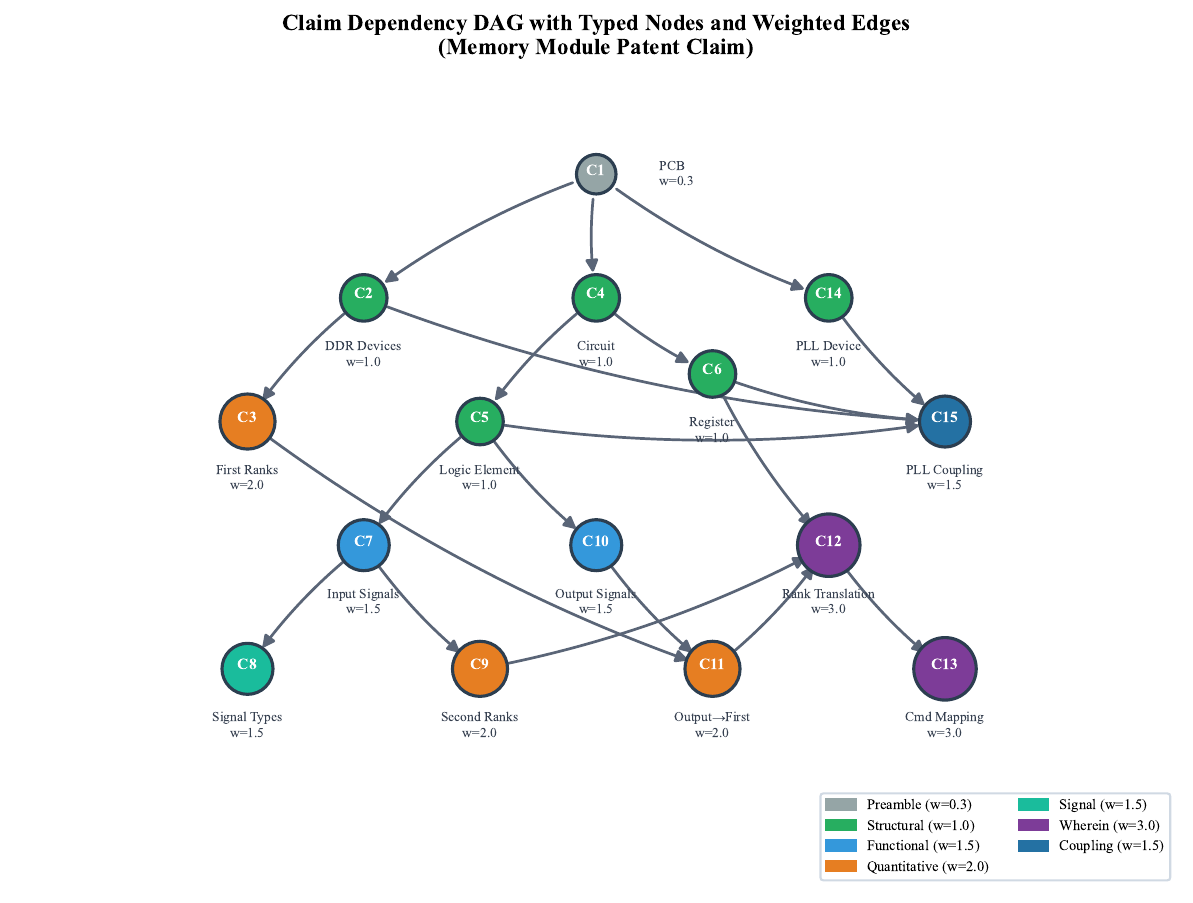}
\caption{Claim dependency DAG for a memory module patent. Node colors represent component types; sizes are proportional to weights. The \textit{wherein} nodes (C12, C13) receive the highest weight ($w{=}3.0$). Arrows point from dependency to dependent: an arrow $u \to v$ denotes that $u \in \textit{deps}(v)$, i.e., $v$'s effective score is zeroed if $\textit{eff}(u) < \theta$.}
\label{fig:claimdag}
\end{figure}

We encode the claim DAG in Lean~4 using inductive types with decidable equality. The encoding below defines three components: (1)~\texttt{ClaimNode}, an inductive type enumerating all claim components; (2)~\texttt{dependencies}, a function mapping each node to the list of nodes it depends on (nodes with no explicit dependencies default to an empty list); and (3)~\texttt{dag\_acyclic}, a machine-verified theorem proving the dependency graph contains no cycles.

\smallskip\noindent\textit{Lean~4 listing convention (used throughout the paper).} We mark each Lean~4 code block with one of two labels: \colorbox{green!15}{\small\textsc{Lean 4: compiled excerpt}} denotes code that is part of the prototype's actual compiled Lean~4 development; \colorbox{orange!20}{\small\textsc{Lean 4: illustrative, not compiled}} denotes specification-level pseudocode, mathematical-intent rendering, or tactic sketches included to convey structure, not to be a substitute for compiled proof. Verification-status classifications (Table~\ref{tab:proofstatus}) are assigned accordingly.

\smallskip
\noindent\colorbox{green!15}{\small\textsc{Lean 4: compiled excerpt}}\nopagebreak

\begin{lstlisting}
-- Inductive type: one constructor per claim component
-- Schematic: the concrete encoding in Appendix A.2 enumerates
-- one constructor per atomic claim limitation; `...` stands in
-- for the remaining constructors, not for a metavariable Cn.
inductive ClaimNode where
  | C1 | C2 | /* ... C_{n-1} ... */ | Cn
  deriving DecidableEq, Repr, Fintype

-- Dependency map: each node lists its parent dependencies
-- Nodes not listed explicitly have no dependencies (empty list)
def dependencies : ClaimNode -> List ClaimNode
  | .C2 => [.C1]       -- C2 depends on C1
  | .C3 => [.C2]       -- C3 depends on C2
  | _   => []           -- all other nodes: no dependencies

-- Edge relation: b is a dependency of a
def depEdge (a b : ClaimNode) : Prop := b ∈ dependencies a

-- Machine-verified proof that the dependency graph is acyclic
theorem dag_acyclic : Acyclic (Relation.TransGen depEdge) := by
  decide
\end{lstlisting}

The acyclicity theorem is mechanically verified by Lean's kernel. Any claim structure with circular dependencies would fail to compile.

\subsection{Match Strength as a Verified Complete Lattice}

We model match strengths as elements of a formally verified complete lattice $\mathbb{L} = ([0,1], \leq, \wedge, \vee, \bot, \top)$ where $\bot = 0.0$, $\top = 1.0$, $\wedge = \min$, and $\vee = \max$. (We use $\mathbb{L}$ for the match strength lattice throughout, reserving $\mathcal{L}$ for the claim language text space from Definition~\ref{def:textspaces}.) Using Mathlib, we instantiate the \texttt{CompleteLattice} typeclass and machine-verify all required axioms. The formal guarantees depend on the alignment between Mathlib's \texttt{CompleteLattice} specification and the intended mathematical structure; for $[0,1]$ with $\min/\max$ this alignment is straightforward, but it represents a general methodological assumption in formalized mathematics that the formal specification correctly captures the intended concept.

\noindent\textbf{Implementation: finite discretization.} Although the mathematical presentation treats $\mathbb{L}$ as the continuous interval $[0,1]$, the Lean~4 implementation uses a finite $10{,}001$-point discretization of $[0,1]$ represented with exact rational arithmetic (basis points in $\{0, 1, \ldots, 10{,}000\}$ corresponding to $[0.0000, 1.0000]$). This choice is driven by two constraints: (i)~Lean's proof tactics (\texttt{omega}, \texttt{norm\_num}, \texttt{decide}) operate on exact arithmetic types (\texttt{Nat}, \texttt{Int}, \texttt{Rat}), not \texttt{Float}; and (ii)~finiteness is required to make Kleene iteration on $\mathbb{L}$ terminate in finitely many steps (Theorem~\ref{thm:convergence}). Each raw ML score is discretized to the nearest basis point, introducing at most $10^{-4}$ rounding error per score; because $W_{\text{cov}}$ is a convex combination, the cumulative error in $W_{\text{cov}}$ is likewise bounded by $10^{-4}$ (at most $0.01$~pp on the percentage scale), which is negligible relative to the typical $0.05$-point thresholds used in practice. The \texttt{CompleteLattice} instance on the finite discretization is derived automatically via \texttt{Fintype}-based \texttt{Finset.sup}/\texttt{Finset.inf}; Appendix~\ref{app:lean} (Section~A.1) provides the full Lean~4 encoding and the distinction between compiled code and illustrative pseudocode. Readers should understand ``$[0,1]$'' throughout Sections~\ref{sec:framework}--\ref{sec:algorithms} as shorthand for this finite discretization; no downstream guarantee depends on the continuum structure.

The lattice structure provides four key properties (exact verification status per Table~\ref{tab:proofstatus}: the \texttt{CompleteLattice} instance itself is machine-verified; the monotonicity and bounded-propagation consequences for $W_{\text{cov}}$ are currently informal proof sketches):
\begin{itemize}[leftmargin=2em]
\item \textbf{Monotonicity:} Adding evidence cannot decrease match strength: $s_1(v) \leq s_2(v)\;\forall v \Rightarrow W_{\text{cov}}(s_1) \leq W_{\text{cov}}(s_2)$ (argued for $W_{\text{cov}}$ in Theorem~\ref{thm:algcorrect}(iv), Section~\ref{sec:algorithms}; presented as an informal proof sketch per Table~\ref{tab:proofstatus}, not a compiled Lean~4 proof).
\item \textbf{Idempotency:} Duplicate evidence does not inflate scores: $a \vee a = a$.
\item \textbf{Bounded propagation:} $\bot \leq \textit{eff}(v) \leq \top$ for all $v$.
\item \textbf{Compositionality:} Lattice operations satisfy absorption, distributivity, and associativity.
\end{itemize}

\begin{figure}[H]
\centering
\includegraphics[width=\textwidth]{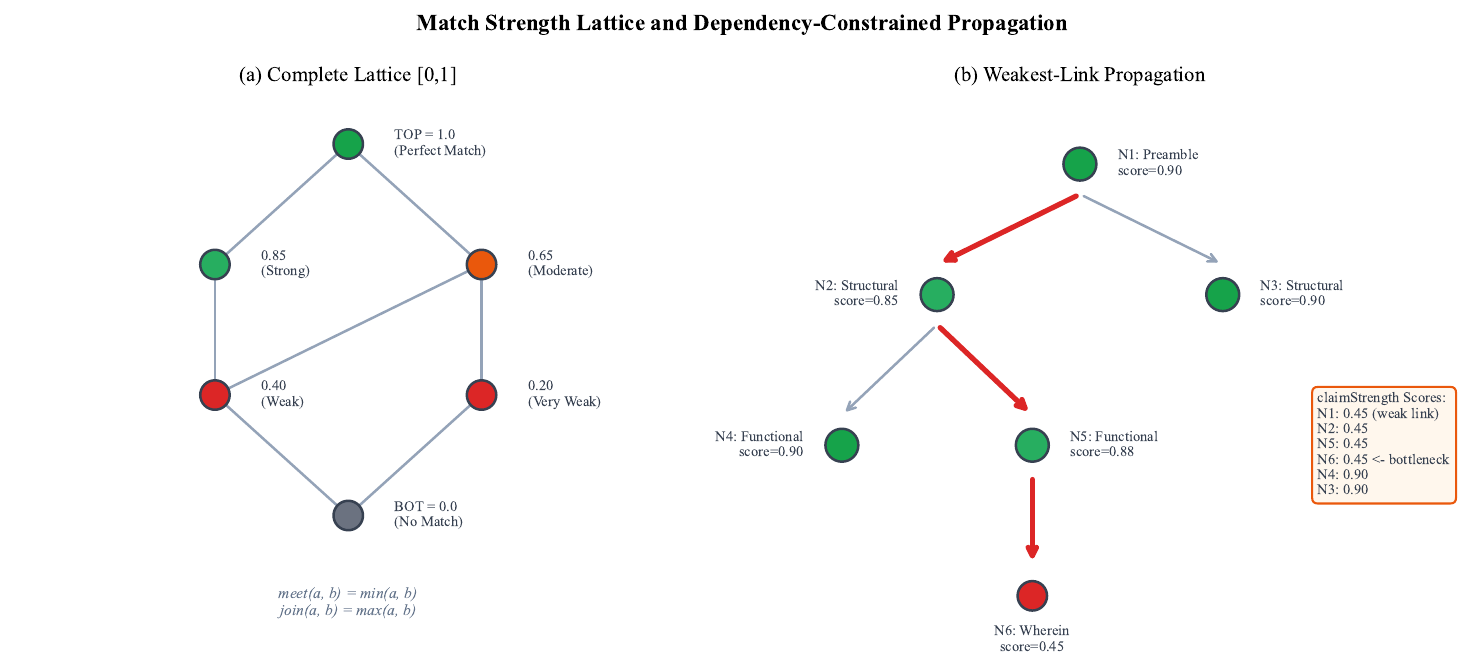}
\caption{(a) The complete lattice $[0,1]$ with meet and join. (b)~Weakest-link sensitivity analysis using the meet-based $\textit{claimStrength}$ model (Definition~\ref{def:claimstrength}): a schematic DAG in which a node's low score of 0.45 propagates through the dependency chain, bounding nodes that transitively depend on it. \textbf{The ``C12'' label in panel (b) is a schematic node name that happens to reuse the same letter/number pattern as the memory-module case study; the score $0.45$ and the dependency layout shown here are not drawn from the case study ($C_{12}$ there has $\beta = 0.80$ under $I_1$).} This figure illustrates the meet-based propagation model, not the threshold-based model (Definition~\ref{def:effcov}) used for coverage computation.}
\label{fig:lattice}
\end{figure}

\subsection{Weakest-Link Propagation}

The overall strength of a claim mapping is determined by the weakest path through the dependency graph. Because the Lean~4 edge relation is oriented from a node to its dependencies (Section~\ref{sec:dagencoding}), the natural recursion order is the reverse reachability relation:
\[
u \prec v \;\Longleftrightarrow\; \texttt{Relation.TransGen depEdge}\; v\; u,
\]
i.e., $u$ is a direct or transitive dependency of $v$. Since \textit{ClaimNode} is finite (Section~\ref{sec:dagencoding}, \texttt{deriving Fintype}) and \texttt{dag\_acyclic} rules out $v \prec v$, the relation $\prec$ is well-founded. The Lean~4 definition of \texttt{claimStrength} is therefore justified by well-founded recursion on $\prec$.

\noindent\textbf{Edge orientation convention.} In Figure~\ref{fig:claimdag}, arrows point from dependency to dependent: an arrow $u \to v$ denotes that $u \in \textit{deps}(v)$, i.e., $u$ is a dependency of $v$, and $v$'s effective score is zeroed if $\textit{eff}(u) < \theta$. In the Lean~4 encoding, the edge relation $\textit{depEdge}(a, b)$ is defined as $b \in \textit{deps}(a)$, i.e., in the reverse orientation (from dependent to dependency). The acyclicity theorem \texttt{dag\_acyclic} applies to this relation: it proves that no node can reach itself via $\textit{depEdge}$, establishing that the dependency graph contains no cycles.

\begin{lemma}[Well-founded dependency order]
\label{lem:deporderwf}
Define a relation $\prec$ on \textit{ClaimNode} by
\[
u \prec v \;\Longleftrightarrow\; \texttt{Relation.TransGen depEdge}\; v\; u .
\]
Then $\prec$ is well-founded.
\end{lemma}
\begin{proof}
Section~\ref{sec:dagencoding} defines \textit{ClaimNode} as a finite type (\texttt{deriving Fintype}). The relation \texttt{Relation.TransGen depEdge} is transitive by construction, and reversal of a transitive relation is transitive, so $\prec$ is transitive. It is irreflexive because \texttt{dag\_acyclic} rules out $\texttt{Relation.TransGen depEdge}\; v\; v$. A transitive, irreflexive relation on a finite type admits no infinite descending chain, so $\prec$ is well-founded.
\end{proof}

\begin{definition}[Claim Strength]
\label{def:claimstrength}
The claim strength function $\textit{claimStrength} : (\textit{ClaimNode} \to \textit{MatchStrength}) \to \textit{ClaimNode} \to \textit{MatchStrength}$ is defined recursively as:
\[
\textit{claimStrength}(\textit{score}, v) = \textit{score}(v) \;\wedge\; \bigwedge_{u \in \textit{deps}(v)} \textit{claimStrength}(\textit{score}, u)
\]
where $\wedge$ is the meet operation and $\bigwedge$ is the infimum over a finite set on the complete lattice $\mathbb{L}$. Well-foundedness follows from Lemma~\ref{lem:deporderwf}. In the full Lean~4 formalization, \texttt{claimStrength} is defined by well-founded recursion on the reverse dependency-reachability relation $\prec$.
\end{definition}

\noindent\textit{Note on proof presentation.} The proofs of Theorems~\ref{thm:weakest} and~\ref{thm:monotone} below are informal proof sketches describing the structure of the corresponding formal Lean~4 proofs. Where Mathlib theorem names appear in \texttt{monospace}, they refer to lemmas in the Mathlib4 library~\cite{mathlib2025} invoked in the intended formal development; these citations indicate the mathematical content the sketches rely on, not compilable code excerpts. These proofs are \emph{not} part of the compiled artifact accompanying this paper: no proof terms for Theorems~\ref{thm:weakest} and~\ref{thm:monotone} are currently submitted to the Lean~4 kernel. If and when these sketches are compiled, the kernel would type-check the resulting proof terms and the $\Omega$-audit would verify their axiom dependencies; until then, Table~\ref{tab:proofstatus} classifies both as informal sketches.

\begin{theorem}[Weakest Link]
\label{thm:weakest}
Assume $G$ is acyclic (i.e., \texttt{dag\_acyclic} holds; see Definition~\ref{def:claimdag} and Section~\ref{sec:dagencoding}). For all $\textit{score} : \textit{ClaimNode} \to \textit{MatchStrength}$ and all $v \in V$:
\begin{enumerate}[label=(\roman*), leftmargin=2.5em, topsep=2pt, itemsep=1pt]
\item $\textit{claimStrength}(\textit{score}, v) \leq \textit{score}(v)$ \hfill (upper bound)
\item $\forall u \in \textit{deps}(v),\; \textit{claimStrength}(\textit{score}, v) \leq \textit{claimStrength}(\textit{score}, u)$ \hfill (dependency bound)
\item $\textit{claimStrength}(\textit{score}, v)$ is the greatest value satisfying (i) and (ii) \hfill (optimality)
\item If $\textit{deps}(v) = \emptyset$ then $\textit{claimStrength}(\textit{score}, v) = \textit{score}(v)$ \hfill (base case)
\end{enumerate}
\end{theorem}
\begin{proof}
By well-founded induction on the relation $\prec$ (Lemma~\ref{lem:deporderwf}).

\textit{Base case} ($\textit{deps}(v) = \emptyset$): $\textit{claimStrength}(\textit{score}, v) = \textit{score}(v) \wedge \bigwedge_\emptyset = \textit{score}(v) \wedge \top = \textit{score}(v)$, so (i) holds with equality, (ii) is vacuously true, (iii) is immediate, and (iv) holds by definition.

\textit{Inductive case} ($\textit{deps}(v) \neq \emptyset$; the $\textit{deps}(v) = \emptyset$ subcase of the induction step reduces to the base case above, under which (ii) is vacuously true): Assume (i)--(iii) hold for all $u \in \textit{deps}(v)$. By Definition~\ref{def:claimstrength},
\[
\textit{claimStrength}(\textit{score}, v) = \textit{score}(v) \;\wedge\; \bigwedge_{u \in \textit{deps}(v)} \textit{claimStrength}(\textit{score}, u).
\]
Property~(i) follows from \texttt{inf\_le\_left}: $a \wedge b \leq a$.
Property~(ii) follows in two steps (valid because $\textit{deps}(v) \neq \emptyset$, so the folded meet is non-trivial): \texttt{inf\_le\_right} ($a \wedge b \leq b$) extracts the $\bigwedge$ component, and then the finite infimum property ($\bigwedge_S \leq x$ for each $x \in S$) extracts the specific dependency $u$ from the folded meet over $\textit{deps}(v)$.
Property~(iii) follows from the universal property of infimum on the complete lattice~$\mathbb{L}$ (i.e., if $x \leq a$ and $x \leq b_i$ for all $i$, then $x \leq a \wedge \bigwedge_i b_i$): $\textit{claimStrength}$ computes the greatest lower bound of $\{\textit{score}(v)\} \cup \{\textit{claimStrength}(\textit{score}, u) \mid u \in \textit{deps}(v)\}$, which is the largest element below all of them.
\end{proof}

\begin{theorem}[Propagation Monotonicity]
\label{thm:monotone}
Assume $G$ is acyclic (\texttt{dag\_acyclic}). If $s_1(v) \leq s_2(v)$ for all $v \in V$, then $\textit{claimStrength}(s_1, v) \leq \textit{claimStrength}(s_2, v)$ for all $v \in V$.
\end{theorem}
\begin{proof}
By well-founded induction on the relation $\prec$ (Lemma~\ref{lem:deporderwf}).

\textit{Base case} ($\textit{deps}(v) = \emptyset$): $\textit{claimStrength}(s_i, v) = s_i(v)$, so $\textit{claimStrength}(s_1, v) = s_1(v) \leq s_2(v) = \textit{claimStrength}(s_2, v)$.

\textit{Inductive case}: Assume the result holds for all $u \in \textit{deps}(v)$. Then $s_1(v) \leq s_2(v)$ by hypothesis, and $\textit{claimStrength}(s_1, u) \leq \textit{claimStrength}(s_2, u)$ for all $u \in \textit{deps}(v)$ by the inductive hypothesis. By monotonicity of $\wedge$ on $\mathbb{L}$ (i.e., $a_1 \leq a_2 \wedge b_1 \leq b_2 \implies a_1 \wedge b_1 \leq a_2 \wedge b_2$) and of $\bigwedge$ (i.e., $f(u) \leq g(u)$ for all $u$ implies $\bigwedge_u f(u) \leq \bigwedge_u g(u)$), $\textit{claimStrength}(s_1, v) \leq \textit{claimStrength}(s_2, v)$.
\end{proof}

\noindent\textbf{Relationship between propagation models.} Definition~\ref{def:effcov} defines a threshold-based propagation: $\textit{eff}(v) = \beta(v)$ if all dependencies meet threshold $\theta$, else $0$. Definition~\ref{def:claimstrength} defines a meet-based propagation: $\textit{claimStrength}$ computes the continuous minimum along dependency paths. These are distinct models producing different values in general. The threshold-based model (implemented in Algorithms~1, 2, 5, and 6) is used for coverage computation and proof certificates, providing binary pass/fail semantics appropriate for legal determinations. The meet-based model provides a complementary continuous sensitivity analysis, showing how bottleneck scores propagate through the DAG. The proof certificates generated by this system certify properties of the threshold-based model.

\subsection{Proof Certificate Structure}

\begin{figure}[H]
\centering
\includegraphics[width=\textwidth]{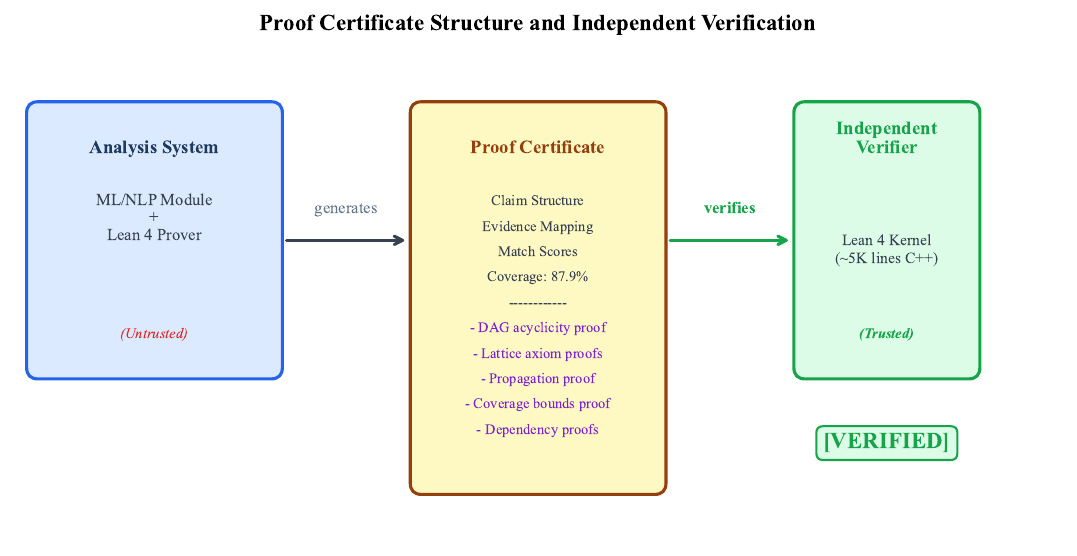}
\caption{Proof certificate generation and independent verification. The analysis system (untrusted) generates results and proofs. An independent verifier checks proofs using only the Lean~4 kernel (trusted, ${\sim}5$K lines of C++).}
\label{fig:certificate}
\end{figure}

Each analysis in the prototype generates a proof certificate intended to be validated by independent kernel type-checking and a \texttt{sorry}-free axiom audit (Definition~\ref{def:validcert}). For Algorithm~\ref{alg:dagcov-given}, the generator is a compiled, machine-verified Lean~4 function whose soundness is established in Proposition~\ref{prop:gensound} (Section~\ref{sec:closedpath}). For the remaining algorithms (2--6), the generators are prototype components whose certificates are ``machine-checkable'' in the sense that the kernel deterministically accepts or rejects them, not in the sense that the generators themselves are formally specified. The \texttt{ProofCertificate} record below bundles the input data (the claim DAG, evidence set, match scores, and coverage score) together with proof fields whose intended propositions are described after the record declaration:

\smallskip
\noindent\colorbox{orange!20}{\small\textsc{Lean 4: illustrative, not compiled}}\nopagebreak

\begin{lstlisting}
structure ProofCertificate where
  claim      : ClaimStructure
  evidence   : EvidenceSet
  scores     : ClaimNode -> MatchStrength
  coverage   : Rat
  -- Formal proofs
  p_acyclic  : proof_dag_acyclic claim
  p_lattice  : proof_lattice_axioms scores
  p_propag   : proof_propagation_correct claim scores
  p_bounded  : 0 <= coverage /\ coverage <= 100
  p_deps     : proof_dependencies_enforced claim scores
  -- Additional fields for specific use cases:
  -- construction, match_types, construction_consistent
\end{lstlisting}

\noindent\textbf{Remark (Proof generation and field types).} The \texttt{ProofCertificate} structure above specifies \textit{what} a certificate must contain. The functions that \textit{construct} certificates are described schematically in the algorithms:
\begin{itemize}[leftmargin=2em, topsep=2pt, itemsep=1pt]
\item \texttt{generateCertificate} (Algorithms~\ref{alg:dagcov} and~\ref{alg:dagcov-given}; compiled, Section~\ref{sec:closedpath});
\item \texttt{ProveNoMatch} and \texttt{ProveDependencyBlock} (Algorithm~\ref{alg:fto});
\item \texttt{ProveConstructionResults} (Algorithm~\ref{alg:construction});
\item \texttt{ProveConsistency} (Algorithm~\ref{alg:consistency});
\item \texttt{ProveDOE} (Algorithm~\ref{alg:doe});
\item \texttt{ProveFixedPoint} (Algorithm~\ref{alg:kleene}).
\end{itemize}
With the exception of \texttt{generateCertificate} for Algorithms~\ref{alg:dagcov} and~\ref{alg:dagcov-given} (Section~\ref{sec:closedpath}), their internal construction is designed for the prototype described in Section~\ref{sec:implementation} but is not formally specified in this paper.

The four abstract proof field types encode the following intended propositions:
\begin{itemize}[leftmargin=2em, topsep=2pt, itemsep=1pt]
\item \texttt{proof\_dag\_acyclic}: the claim DAG contains no cycles (cf.\ \texttt{dag\_acyclic}, Section~\ref{sec:framework}).
\item \texttt{proof\_lattice\_axioms}: the \texttt{CompleteLattice} typeclass instance for \texttt{MatchStrength} is valid (cf.\ Section~\ref{sec:framework} and Appendix~\ref{app:lean}, Section~A.1).
\item \texttt{proof\_propagation\_correct}: effective scores respect dependency constraints, a node with an unsatisfied transitive dependency has $\textit{eff}(v) = 0$ (cf.\ Theorem~\ref{thm:algcorrect}(ii)--(iii)).
\item \texttt{proof\_dependencies\_enforced}: all dependency relationships in the DAG are transitively enforced in the propagated scores (cf.\ Theorem~\ref{thm:algcorrect}(ii)).
\end{itemize}
The correctness theorems (Theorems~\ref{thm:algcorrect}--\ref{thm:convergence}) establish that the computed values satisfy the mathematical properties these fields certify. However, with the exception of \texttt{generateCertificate} for Algorithms~\ref{alg:dagcov} and~\ref{alg:dagcov-given} (Section~\ref{sec:closedpath}, Proposition~\ref{prop:gensound}), the formal proof that the remaining generation functions produce \texttt{sorry}-free certificates satisfying Definition~\ref{def:validcert} for all valid inputs is not given in this paper.

\medskip
\noindent\textbf{Mapping from the abstract schema to the compiled specialization.} The abstract \texttt{ProofCertificate} record above uses four specification-level proof-field names whose intended propositions are described informally; the compiled specialization for Algorithm~\ref{alg:dagcov-given} (Section~\ref{sec:closedpath}, Appendix~\ref{app:lean}) instantiates them with five concrete Lean~4 propositions, as follows:
\begin{center}
{\small
\begin{tabular}{@{}p{4.9cm}p{8cm}@{}}
\toprule
\textbf{Abstract field (this section)} & \textbf{Compiled field (\S\ref{sec:closedpath}/App.~\ref{app:lean})} \\
\midrule
\texttt{proof\_dag\_acyclic} & \texttt{p\_acyclic}: $\forall v,\; \neg\, \texttt{Relation.TransGen}\ \texttt{depEdge}\ v\ v$ \\
\texttt{proof\_lattice\_axioms} & \texttt{p\_lattice}: $\forall a\,b\,c,\; a \leq b \to b \leq c \to a \leq c$ (transitivity witness on \texttt{DMatchStrength}; the full \texttt{CompleteLattice} instance is supplied by Mathlib and used implicitly when the lattice appears in downstream proofs) \\
\texttt{proof\_propagation\_correct} & \texttt{p\_propag}: dependency-propagation correctness of \texttt{computeEff} (any node with a below-threshold dependency is zeroed) \\
(implicit in the abstract schema) & \texttt{p\_bounded}: $\texttt{weightedCoverage}(\beta,\theta) \in [0, 100]$ (compiled bound, \texttt{coverage\_in\_range}) \\
\texttt{proof\_dependencies\_enforced} & \texttt{p\_deps}: $\texttt{coverage} = \texttt{weightedCoverage}(\beta,\theta)$ (definitional equality; via the definition of \texttt{weightedCoverage} this is equivalent to the Definition~\ref{def:depscov} formula over propagated $\textit{eff}$, machine-verified by unfolding) \\
\bottomrule
\end{tabular}%
}
\end{center}
\noindent The compiled specialization thus \emph{narrows} the abstract \texttt{proof\_lattice\_axioms} to the specific transitivity witness that downstream proofs actually consume, and \emph{refines} \texttt{proof\_dependencies\_enforced} from ``dependencies are enforced in the coverage computation'' (informal) to the definitional coverage-equality that, combined with the definition of \texttt{weightedCoverage}, records the exact numerical equality $\texttt{coverage} = W_{\text{cov}}$. An explicit \texttt{p\_bounded} field is added; the abstract schema folded this into the intended content of \texttt{proof\_dependencies\_enforced}. This gap is mitigated by the trust model: the generation functions are part of the untrusted Analysis System (Section~\ref{sec:framework} below), and the Lean~4 kernel independently type-checks every certificate before acceptance. A faulty or incomplete generation function therefore produces certificates that are \textit{rejected} by the Independent Verifier, not certificates that are silently accepted. Providing explicit Lean~4 proposition definitions for the remaining abstract proof fields and proving that the remaining generation functions produce valid certificates is identified as future work (Section~\ref{sec:discussion}).

\subsubsection{Why the Analysis System, Including the Lean~4 Prover, Is Untrusted}

Figure~\ref{fig:certificate} labels the entire Analysis System, comprising both the ML/NLP Module and the Lean~4 Prover, as \textit{Untrusted}. This labeling may appear paradoxical: if the Lean~4 Prover is a formal verification tool, why is it not trusted? The answer lies in the de Bruijn criterion~\cite{barendregt2005challenge} and in the architectural distinction between proof \textit{generation} and proof \textit{checking} that the Lean~4 ecosystem explicitly enforces.

\textbf{What the Lean~4 Prover does.} The Lean~4 Prover in the Analysis System is the proof \textit{generation} component, distinct from the Lean~4 Kernel (the proof \textit{checking} component). Given the ML-computed match scores and the encoded claim DAG, the Lean~4 Prover: (1)~applies proof tactics (\texttt{decide}, \texttt{omega}, \texttt{norm\_num}, etc.) to search for and construct formal proof terms certifying the required properties; (2)~calls into the Mathlib library (over 210,000 formalized theorems) for reusable results; (3)~assembles the \texttt{ProofCertificate} record by populating each proof field with the constructed terms; and (4)~bridges the ML layer and formal layer by encoding ML-computed match scores as discretized \texttt{MatchStrength} values.

\textbf{Why the Lean~4 Prover is untrusted.} Despite being a formal tool, the Lean~4 Prover is treated as untrusted for three independent reasons. First, the proof generation system is large and complex: the Lean~4 tactic framework, the elaborator, and the Mathlib library collectively comprise millions of lines of code, any of which could contain bugs that cause the prover to generate an incorrect proof term. Second, the Lean~4 Prover receives ML-computed match scores as inputs that lie below the trust boundary; any computation depending on these inputs inherits that untrustworthiness. Third, the proof generation process depends on proof scripts authored by the system developers, with no structural guarantee of correctness or completeness. Treating the entire Analysis System as untrusted ensures that all such errors, whether accidental or deliberate, are caught by the independent kernel verification step.

\textbf{Why the Lean~4 Kernel is trusted.} In contrast, the Lean~4 Kernel is trusted for four reasons: (1)~it is small (approximately 5,000 lines of C++), feasible to audit; (2)~it performs exactly one operation: given a proof term $t$ and a type $T$, it determines whether $\vdash t : T$ in the Calculus of Inductive Constructions, a well-defined deterministic procedure with no heuristics; (3)~it is stateless and deterministic; and (4)~it has been independently self-verified by the Lean4Lean project~\cite{carneiro2024lean4lean}. The entire system's trustworthiness reduces to trusting this kernel, a reduction in trusted computing base of several orders of magnitude.

\subsubsection{The \texttt{sorry}-Free Verification Requirement}

A critical soundness requirement for the proof certificate mechanism is that every proof certificate must be verified to be free of \texttt{sorry} and any axiom outside the trusted axiom set before it is accepted as valid.

\textbf{What \texttt{sorry} is.} Lean~4 provides a built-in mechanism called \texttt{sorry} that allows any proposition to be asserted as proven without supplying a proof term. At the axiom level, \texttt{sorry} is realized by a primitive constant \texttt{sorryAx} that inhabits any type unconditionally. Critically, proof terms containing \texttt{sorry} type-check in the Lean~4 Kernel without error, because \texttt{sorry} is declared as an axiom at the foundational level and the kernel does not distinguish between axioms that are mathematically justified and axioms that are placeholders for missing proofs.

\textbf{The soundness hole.} Because the Lean~4 Kernel accepts proof terms containing \texttt{sorry}, a malicious or buggy proof generator could trivially produce a \texttt{ProofCertificate} for any coverage score by discharging every proof obligation with \texttt{sorry}:

\smallskip
\noindent\colorbox{orange!20}{\small\textsc{Lean 4: illustrative, not compiled (counter example)}}\nopagebreak

\begin{lstlisting}
-- A fraudulent certificate: type-checks in the kernel
-- but contains no mathematical content
def fraudulentCertificate : ProofCertificate := {
  claim    := encodedClaim,
  evidence := encodedEvidence,
  scores   := someScores,
  coverage := 100,       -- false: asserts 100% with no basis
  p_acyclic := sorry,    -- no acyclicity proof provided
  p_lattice := sorry,    -- no lattice axiom proof provided
  p_propag  := sorry,    -- no propagation proof provided
  p_bounded := sorry,    -- no coverage bounds proof provided
  p_deps    := sorry     -- no dependency enforcement provided
}
\end{lstlisting}

This certificate would type-check and produce a \textsc{Verified} output, despite containing no mathematical content. Without additionally auditing its axiom dependencies, the entire proof certificate mechanism would be rendered meaningless.

\begin{definition}[Valid Proof Certificate]
\label{def:validcert}
A proof certificate $\Pi$ is accepted as valid by the Independent Verifier if and only if:
\begin{enumerate}[label=(\roman*), leftmargin=2.5em, topsep=2pt, itemsep=1pt]
\item \textit{Type-checking:} every proof field in $\Pi$ type-checks in the Lean~4 Kernel; and
\item \textit{\texttt{sorry}-free:} the complete set of axioms on which $\Pi$ depends is a subset of the trusted axiom set $\Omega = \{\texttt{propext}, \texttt{Classical.choice}, \texttt{Quot.sound}\}$. In implementation, this is checked via Lean~4's \texttt{\#print axioms} command, which reports the foundational axiom \texttt{sorryAx} (not the tactic name \texttt{sorry}) when placeholder proofs are present; the audit rejects any certificate whose axiom list contains \texttt{sorryAx} or any other axiom outside $\Omega$.
\end{enumerate}
A certificate satisfying (i) but not (ii) is rejected.
\end{definition}

\textbf{The trusted axiom set.} The three axioms in $\Omega$ are the complete foundational axioms of Lean~4's logical framework: \texttt{propext} (propositional extensionality: logically equivalent propositions are equal), \texttt{Classical.choice} (given a proof of nonemptiness, a canonical element can be chosen, enabling classical reasoning), and \texttt{Quot.sound} (quotient soundness: related elements have equal images under the quotient map). These are universally accepted in the formal verification community and introduce no soundness issues. Any axiom outside $\Omega$, including \texttt{sorry}, \texttt{sorryAx}, or the axiom introduced by \texttt{native\_decide}, must cause condition~(ii) to fail and the certificate to be rejected.

\textbf{A note on \texttt{native\_decide}.} The Lean~4 tactic \texttt{native\_decide} compiles a decision procedure to native machine code and evaluates it outside the kernel, trusting the compiler and runtime. When used, \texttt{\#print axioms} reports an additional axiom, \texttt{Lean.ofReduceBool}, not in $\Omega$. Proof scripts in the Analysis System must therefore use \texttt{decide} rather than \texttt{native\_decide} for all decidable propositions, and the Independent Verifier must reject any certificate whose axiom set includes \texttt{Lean.ofReduceBool} (or any other axiom outside $\Omega$).

\textbf{Enforcement.} The \texttt{sorry}-free requirement is enforced using Lean~4's \texttt{\#print axioms} command, which enumerates all axioms on which a given definition transitively depends:

\smallskip
\noindent\colorbox{orange!20}{\small\textsc{Lean 4: illustrative, not compiled}}\nopagebreak

\begin{lstlisting}
-- Specification-level pseudocode; the actual verification
-- uses Lean 4's #print axioms command and kernel API.
-- See text below for implementation details.
def trustedAxiomSet : Finset Name := {
  `propext,
  `Classical.choice,
  `Quot.sound
}

-- Full two-step verification per Definition 5.5 (Valid Proof Certificate)
def verifyProofCertificate
    (cert : ProofCertificate) : VerificationResult :=
  let axDeps := collectTransitiveAxioms cert
  if typeChecks cert && axDeps.all
    (fun ax => trustedAxiomSet.contains ax)
  then .verified
  else .rejected
\end{lstlisting}

The \texttt{collectTransitiveAxioms} operation traverses the full dependency graph of the certificate's proof terms, collecting every axiom at any depth, ensuring that \texttt{sorry} cannot be hidden inside a helper lemma or intermediate definition.

\textbf{Relationship to the trust model.} The \texttt{sorry}-free requirement is a direct and necessary consequence of the trust boundary. Because the Lean~4 Prover is untrusted, the Independent Verifier cannot assume that proof scripts are written correctly or completely. The \texttt{sorry}-free requirement ensures that the kernel's \textsc{Verified} output reflects a genuine mathematical theorem, with every proof obligation discharged by a real proof term grounded in $\Omega$, rather than an artifact of the proof generation system bypassing its own obligations.

\medskip
\noindent The verification status of every formal result in this paper is summarized in Table~\ref{tab:proofstatus} (Section~\ref{sec:contributions}).

\subsubsection{A Fully Specified Coverage Certificate for Algorithm~\ref{alg:dagcov-given}}
\label{sec:closedpath}

The abstract \texttt{ProofCertificate} record above is retained as a common schema across all use cases. To close one end-to-end formal path without changing the verification status of the remaining algorithms, we now present a compiled specialization for Algorithm~\ref{alg:dagcov-given}, the coverage core that lies entirely above the trust boundary once the bounded score function $\beta$ is fixed.

The compiled Lean~4 development (Appendix~\ref{app:closedpath}) defines:
\begin{enumerate}[leftmargin=2em, topsep=2pt, itemsep=1pt]
\item \texttt{computeEff}: a propagation function that computes effective scores by unrolling the DAG in topological order, with termination justified by \texttt{topoDepth} (a depth function proven strictly decreasing along dependency edges);
\item \texttt{propag\_proof}: a machine-verified theorem that \texttt{computeEff} enforces dependency constraints (if any dependency of $v$ falls below $\theta$, then $\texttt{computeEff}(v) = 0$);
\item \texttt{coverage\_in\_range}: a machine-verified theorem that the weighted coverage lies in $[0, 100]$ for any valid score function;
\item \texttt{ProofCertificate}: a compiled certificate structure whose five proof fields (acyclicity, lattice transitivity, propagation correctness, coverage bounds, and coverage-definition equality) are concrete Lean~4 propositions, not abstract specification-level names;
\item \texttt{generateCertificate}: a compiled generator function that, given any bounded score function $\beta : V \to [0,1]$ and valid threshold $\theta$, constructs a \texttt{ProofCertificate} with all five fields discharged by real proof terms grounded in $\Omega$, no \texttt{sorry}.
\end{enumerate}

\begin{proposition}[Coverage-core generator soundness for Algorithm~\ref{alg:dagcov-given}]
\label{prop:gensound}
Let $\beta : V \to [0,1]$ be a score function, let $h_{\text{valid}}$ be an $\Omega$-grounded proof of $\texttt{ScoreValid}\ \beta$ (that is, a proof term depending only on axioms in $\Omega = \{\texttt{propext}, \texttt{Classical.choice}, \texttt{Quot.sound}\}$), and let $\theta$ be a valid threshold in the range required by the generator. Then $\texttt{generateCertificate}(\beta, h_{\text{valid}}, \theta)$ returns a \texttt{ProofCertificate} accepted under Definition~\ref{def:validcert}:
\begin{enumerate}[label=(\roman*), leftmargin=2.5em, topsep=2pt, itemsep=1pt]
\item \textit{Type-checking:} every proof field in the returned certificate type-checks in the Lean~4 kernel (by construction: the generator populates each field with a concrete proof term);
\item \textit{\texttt{sorry}-free:} the generator function \texttt{generateCertificate} itself depends only on $\Omega$ (verified by \texttt{\#print axioms generateCertificate}; see Appendix~\ref{app:axiomaudit}), and because $h_{\text{valid}}$ is by hypothesis $\Omega$-grounded, the returned certificate's transitive axiom set remains a subset of $\Omega$.
\end{enumerate}
The $\Omega$-grounded hypothesis on $h_{\text{valid}}$ is load-bearing: a caller who supplies a \texttt{sorry}-based validity proof will see \texttt{sorryAx} in the axiom audit of the resulting certificate, and the certificate will be correctly rejected under Definition~\ref{def:validcert}(ii).
\end{proposition}
\begin{proof}
Clause~(i) is discharged by construction: the generator function (Appendix~\ref{app:closedpath}) populates each proof field with a concrete proof term: \texttt{dag\_acyclic} for acyclicity, \texttt{le\_trans} for lattice transitivity, \texttt{propag\_proof} for propagation correctness, \texttt{coverage\_in\_range} for bounds, and \texttt{rfl} for the coverage-definition equality. The Lean~4 kernel verifies that each term has the declared type. Clause~(ii) is discharged in two parts: first, \texttt{\#print axioms generateCertificate} (Appendix~\ref{app:axiomaudit}) reports only $\Omega$, establishing that the generator's own code introduces no axioms outside $\Omega$; second, the two case-study certificates (\texttt{cert\_I1\_broad} and \texttt{cert\_I2\_narrow}), which supply concrete $\Omega$-grounded \texttt{ScoreValid} proofs, are separately audited and also report only $\Omega$. Any future caller who supplies $\Omega$-grounded validity proofs will obtain the same result; a caller who supplies a \texttt{sorry}-based validity proof will see \texttt{sorryAx} in the axiom audit of the resulting certificate, and the certificate will be correctly rejected under Definition~\ref{def:validcert}(ii).
\end{proof}

\noindent\textit{Scope of the closed path.} Proposition~\ref{prop:gensound} closes the formal gap for Algorithm~\ref{alg:dagcov-given}: once the bounded score function $\beta$ is fixed, the entire downstream computation (propagation, coverage, and certificate generation) is machine-verified and accepted under Definition~\ref{def:validcert}. The ML computation of $\beta$ itself remains below the trust boundary. Algorithm~\ref{alg:dagcov} (the monolithic form) inherits this generator soundness because it is semantically equivalent to computing $\beta(v) = \max_{s \in S} M(v,s)$ and then invoking Algorithm~\ref{alg:dagcov-given}. The remaining generators (Algorithms~2--6) retain their architecturally-mitigated status per Table~\ref{tab:proofstatus}.

\noindent\textit{On the \texttt{noncomputable} marker.} \texttt{generateCertificate}, \texttt{weightedCoverage}, and the case-study instantiations \texttt{cert\_I1\_broad}, \texttt{cert\_I2\_narrow} are declared \texttt{noncomputable}. This reflects their transitive use of \texttt{Classical.choice} (via rational-number order-theoretic lemmas from Mathlib that invoke classical reasoning for \texttt{Linear\-Order} / \texttt{Complete\-Linear\-Order} instances); it does not indicate that the functions lack definitional content. \texttt{noncomputable} prevents Lean from generating executable bytecode for the definitions but leaves their proof-relevant content fully available to the kernel, which is all that Definition~\ref{def:validcert}'s type-checking criterion requires. Because \texttt{Classical.choice} is in $\Omega$, the $\Omega$-grounded axiom-audit criterion of Definition~\ref{def:validcert}(ii) is preserved; this is confirmed by the \texttt{\#print axioms} output in Appendix~\ref{app:axiomaudit}. The prototype (Section~\ref{sec:implementation}) constructs concrete certificates by invoking \texttt{generateCertificate} from a separate driver; the resulting certificate values are serialized and audited by the Independent Verifier as usual.

\medskip
\noindent\textbf{Remark (exactly what the certificate attests to).} Proposition~\ref{prop:gensound} establishes Definition~\ref{def:validcert}'s two validity criteria (type-checking and $\Omega$-grounded axioms) for the returned certificate. The five compiled proof fields separately attest to: (a)~\texttt{p\_acyclic}, acyclicity of the claim DAG; (b)~\texttt{p\_lattice}, transitivity of $\leq$ on \texttt{DMatchStrength}; (c)~\texttt{p\_propag}, dependency-propagation correctness of \texttt{computeEff}, namely that $\texttt{computeEff}(\beta,\theta)(v) = 0$ whenever some $u \in \textit{deps}(v)$ has $\texttt{computeEff}(\beta,\theta)(u) < \theta$; (d)~\texttt{p\_bounded}, the arithmetic bound $\texttt{weightedCoverage}(\beta,\theta) \in [0,100]$ for any bounded $\beta$ (Theorem \texttt{coverage\_in\_range}, Appendix~\ref{app:lean}); and (e)~\texttt{p\_deps}, the definitional equality $\texttt{coverage} = \texttt{weightedCoverage}(\beta,\theta)$. Critically, the compiled \texttt{weightedCoverage} function (Appendix~\ref{app:lean}) is defined directly on the propagated effective scores: $\texttt{weightedCoverage}(\beta,\theta) := \bigl(\sum_v \texttt{claimWeight}(v) \cdot \texttt{computeEff}(\beta,\theta)(v)\bigr) \big/ \bigl(\sum_v \texttt{claimWeight}(v)\bigr) \times 100$. Consequently $\texttt{weightedCoverage}(\beta,\theta)$ is equal to Definition~\ref{def:depscov}'s $W_{\text{cov}}$ by definitional unfolding, and the certificate's \texttt{coverage} field carries the numerical quantity reported by Algorithm~\ref{alg:dagcov-given} with no caller-side propagation precondition. The coverage-field equality $\texttt{coverage} = W_{\text{cov}}$ is therefore machine-verified, not merely argued in prose.

\section{Algorithms}
\label{sec:algorithms}

We present six formal algorithms covering the five IP use cases. Algorithm~1b (WeightedDAGCoverageGivenScores) is a shared helper variant of Algorithm~1, not a distinct use case: it separates DAG propagation from score computation so that Algorithm~3 (construction sensitivity) can reuse the propagation logic on pre-computed scores.

\noindent\textbf{Two propagation models.} This paper defines two distinct propagation models. The \textit{threshold-based} model (Definition~\ref{def:bestmatch}) computes $\textit{eff}(v) = \beta(v)$ if all dependencies meet threshold~$\theta$, else~$0$; it provides binary pass/fail semantics and is used in Algorithms~1, 2, 5, and~6 for coverage computation and proof certificates. The \textit{meet-based} model ($\textit{claimStrength}$, Definition~\ref{def:claimstrength}) computes the continuous minimum along dependency paths and is used for sensitivity analysis (Section~\ref{sec:framework}). These models produce different values in general; their relationship is discussed in detail following Theorem~\ref{thm:monotone}.

\subsection{Algorithm 1: DAG-Based Weighted Coverage}

Algorithm~\ref{alg:dagcov} computes the weighted DAG coverage score $W_{\text{cov}}$ and generates a machine-checkable proof certificate. It proceeds in three phases. First (lines~1--6), it computes match scores between every claim limitation and every evidence segment, retaining the best match $\beta(v)$ for each node; this step constitutes the ML layer and lies below the trust boundary. Second (lines~7--13), it traverses the DAG in topological order, guaranteed well-defined by \texttt{dag\_acyclic}, and propagates dependency constraints: a node whose dependencies fail threshold~$\theta$ receives $\textit{eff}(v) = 0$ regardless of its own $\beta(v)$, enforcing the all-elements rule transitively through the dependency structure (Definition~\ref{def:bestmatch}). This threshold-based zeroing is the key mechanism distinguishing DAG-weighted coverage from a flat weighted average. Third (lines~14--16), it computes $W_{\text{cov}}$ as the weighted average of effective scores using the legally-motivated weights from Table~\ref{tab:weights} and passes the results to the compiled \texttt{generateCertificate} (Section~\ref{sec:closedpath}) to produce the formal certificate~$\Pi$. Theorem~\ref{thm:algcorrect} proves that the algorithm correctly enforces dependency constraints, respects coverage bounds, and is monotone in the input scores.

\begin{algorithm}[H]
\caption{WeightedDAGCoverage \hfill {\small\normalfont\itshape Correctness: Theorem~\ref{thm:algcorrect} (informal sketch; see Table~\ref{tab:proofstatus})}}
\label{alg:dagcov}
\begin{algorithmic}[1]
\Require Claim DAG $G=(V,E,\tau,w)$, Evidence $S=\{s_1\ldots s_m\}$, threshold $\theta$
\Ensure Coverage score $W_{\text{cov}}$, proof certificate $\Pi$
\For{each $v \in V$}
  \For{each $s \in S$}
    \State $M(v,s) \gets \text{MatchScore}(v, s)$ \Comment{Section~\ref{sec:architecture}}
  \EndFor
  \State $\beta(v) \gets \max_{s \in S} M(v,s)$
\EndFor
\State $\textit{levels} \gets \text{TopologicalSort}(G)$
\For{level $\ell = 0$ to $\max\_\text{level}$}
  \For{each $v$ at level $\ell$}
    \State $\textit{deps\_met}(v) \gets \forall u \in \textit{deps}(v): \textit{eff}(u) \geq \theta$ \Comment{transitive enforcement via propagated eff}
    \State $\textit{eff}(v) \gets \begin{cases} \beta(v) & \text{if } \textit{deps\_met}(v) \\ 0 & \text{otherwise}\end{cases}$
  \EndFor
\EndFor
\State $W_{\text{cov}} \gets \frac{\sum_{v} w(\tau(v)) \cdot \textit{eff}(v)}{\sum_{v} w(\tau(v))} \times 100$
\State $\Pi \gets \text{generateCertificate}(\beta, \theta)$ \Comment{compiled generator; \S\ref{sec:closedpath}. Internally re-computes $\textit{eff}$ via \texttt{computeEff} and reports the same $W_{\text{cov}}$.}
\State \Return $(W_{\text{cov}}, \Pi)$
\end{algorithmic}
\end{algorithm}

\begin{theorem}[Correctness of Algorithm~\ref{alg:dagcov}]
\label{thm:algcorrect}
For all valid inputs $(G, S, \theta)$ where $G = (V, E, \tau, w)$ satisfies \texttt{dag\_acyclic}, $S \neq \emptyset$, and $w(\tau(v)) > 0$ for all $v \in V$, Algorithm~\ref{alg:dagcov} computes $W_{\text{cov}}$ satisfying:
\begin{enumerate}[label=(\roman*), leftmargin=2.5em, topsep=2pt, itemsep=1pt]
\item $0 \leq W_{\text{cov}} \leq 100$ \hfill (bounds)
\item $\textit{eff}(v) = 0$ for all $v$ where $\neg\textit{deps\_met}(v)$, i.e., where $\exists u \in \textit{deps}(v)$ such that $\textit{eff}(u) < \theta$ (Definition~\ref{def:bestmatch}) \hfill (dependency enforcement)
\item $\textit{eff}(v) = \beta(v)$ for all $v$ where $\textit{deps\_met}(v)$ \hfill (completeness)
\item $W_{\text{cov}}$ is pointwise monotone: if $\beta_1(v) \leq \beta_2(v)$ for all $v \in V$, then $W_{\text{cov}}(\beta_1) \leq W_{\text{cov}}(\beta_2)$ \hfill (monotonicity)
\end{enumerate}
\end{theorem}
\begin{proof}
We prove (ii), (iii), (i), (iv) in that order; (i) uses (ii) and (iii), so this ordering avoids forward references.

\textit{(ii) Dependency enforcement:} The topological sort (line~7) ensures all predecessors of $v$ are processed before $v$, so $\textit{eff}(u)$ values are final when $\textit{deps\_met}(v) \leftarrow \forall u \in \textit{deps}(v): \textit{eff}(u) \geq \theta$ is evaluated (line~10). By line~11, $\textit{eff}(v) \leftarrow 0$ whenever $\neg\textit{deps\_met}(v)$.

\textit{(iii) Completeness:} By line~11, $\textit{eff}(v) \leftarrow \beta(v)$ when $\textit{deps\_met}(v)$. Since $\beta(v) = \max_{s \in S} M(v,s) \in [0,1]$ by Definition~\ref{def:matchscore}, $\textit{eff}(v) = \beta(v)$ unconditionally whenever all dependencies are met, including the case $\beta(v) = 0$.

\textit{(i) Bounds:} We first show $\textit{eff}(v) \in [0,1]$ by case analysis on $\textit{deps\_met}(v)$: when $\textit{deps\_met}(v)$ holds, (iii) gives $\textit{eff}(v) = \beta(v)$, and $\beta(v) \in [0,1]$ by Definition~\ref{def:matchscore} (boundedness of $M$); when $\textit{deps\_met}(v)$ fails, (ii) gives $\textit{eff}(v) = 0 \in [0,1]$. Hence $0 \leq \textit{eff}(v) \leq 1$ for all $v$. \textit{Lower bound:} Since $w(\tau(v)) > 0$ by assumption and $\textit{eff}(v) \geq 0$, $\sum_v w(\tau(v)) \cdot \textit{eff}(v) \geq 0$, so $W_{\text{cov}} \geq 0$. \textit{Upper bound:} Since $\textit{eff}(v) \leq 1$, $\sum_v w(\tau(v)) \cdot \textit{eff}(v) \leq \sum_v w(\tau(v)) \cdot 1 = \sum_v w(\tau(v))$. Dividing: $W_{\text{cov}} = \frac{\sum_v w(\tau(v)) \cdot \textit{eff}(v)}{\sum_v w(\tau(v))} \times 100 \leq 100$.

\textit{(iv) Monotonicity:} Suppose $\beta_1(v) \leq \beta_2(v)$ for all $v \in V$. We prove $\textit{eff}_2(v) \geq \textit{eff}_1(v)$ for all $v$ by induction on the topological order. Consider three cases for each $v$: (a)~$\textit{deps\_met}$ fails under both $\beta_1$ and $\beta_2$: $\textit{eff}_1(v) = \textit{eff}_2(v) = 0$. (b)~$\textit{deps\_met}$ holds under both: $\textit{eff}_2(v) = \beta_2(v) \geq \beta_1(v) = \textit{eff}_1(v)$. (c)~$\textit{deps\_met}$ fails under $\beta_1$ but holds under $\beta_2$: $\textit{eff}_2(v) = \beta_2(v) \geq 0 = \textit{eff}_1(v)$. The reverse case ($\textit{deps\_met}$ holds under $\beta_1$ but fails under $\beta_2$) cannot occur: by the inductive hypothesis, $\textit{eff}_2(u) \geq \textit{eff}_1(u)$ for all predecessors $u \in \textit{deps}(v)$, so if $\textit{eff}_1(u) \geq \theta$ (as required by $\textit{deps\_met}$ under $\beta_1$), then $\textit{eff}_2(u) \geq \textit{eff}_1(u) \geq \theta$, and $\textit{deps\_met}$ holds under $\beta_2$ as well. In all cases $\textit{eff}_2(v) \geq \textit{eff}_1(v)$. Since $w(\tau(v)) > 0$, the weighted sum is non-decreasing, so $W_{\text{cov}}(\beta_2) \geq W_{\text{cov}}(\beta_1)$.
\end{proof}

\noindent\textit{Note on verification status.} Theorem~\ref{thm:algcorrect} remains an informal proof sketch per Table~\ref{tab:proofstatus}. The concrete certificate-generation path for Algorithm~\ref{alg:dagcov-given} is closed separately in Section~\ref{sec:closedpath} (Proposition~\ref{prop:gensound}), which provides a narrower but machine-verified guarantee covering propagation correctness, coverage bounds, and generator soundness.

\noindent\textbf{Complexity.} $O(n \cdot m \cdot d)$ for match computation (where $d$ is the embedding dimension) plus $O(n + |E|)$ for DAG propagation.

\begin{algorithm}[H]
\renewcommand{\thealgorithm}{1b}
\caption{WeightedDAGCoverageGivenScores (Variant of Algorithm~1) \hfill {\small\normalfont\itshape Correctness: Theorem~\ref{thm:algcorrect} applied with $\beta$ as input}}
\label{alg:dagcov-given}
\begin{algorithmic}[1]
\Require Claim DAG $G=(V,E,\tau,w)$, pre-computed scores $\beta : V \to [0,1]$, threshold $\theta$
\Ensure Coverage score $W_{\text{cov}}$, proof certificate $\Pi$
\State $\textit{levels} \gets \text{TopologicalSort}(G)$
\For{level $\ell = 0$ to $\max\_\text{level}$}
  \For{each $v$ at level $\ell$}
    \State $\textit{deps\_met}(v) \gets \forall u \in \textit{deps}(v): \textit{eff}(u) \geq \theta$
    \State $\textit{eff}(v) \gets \begin{cases} \beta(v) & \text{if } \textit{deps\_met}(v) \\ 0 & \text{otherwise}\end{cases}$
  \EndFor
\EndFor
\State $W_{\text{cov}} \gets \frac{\sum_{v} w(\tau(v)) \cdot \textit{eff}(v)}{\sum_{v} w(\tau(v))} \times 100$
\State $\Pi \gets \text{generateCertificate}(\beta, \theta)$ \Comment{compiled generator; \S\ref{sec:closedpath}. Internally re-computes $\textit{eff}$ via \texttt{computeEff} and reports the same $W_{\text{cov}}$.}
\State \Return $(W_{\text{cov}}, \Pi)$
\end{algorithmic}
\end{algorithm}
\renewcommand{\thealgorithm}{\arabic{algorithm}}
\addtocounter{algorithm}{-1}

\noindent\textbf{Closed formal core.} Algorithm~\ref{alg:dagcov-given} is accompanied by a fully specified, compiled Lean~4 certificate generator (\texttt{generateCertificate}, Section~\ref{sec:closedpath}). Because Algorithm~\ref{alg:dagcov} is semantically equivalent to computing $\beta(v) = \max_{s \in S} M(v,s)$ for each $v \in V$ and then invoking Algorithm~\ref{alg:dagcov-given}, the formal phase of Algorithm~\ref{alg:dagcov} inherits this generator soundness once $\beta$ is fixed. The ML computation of $\beta$ remains below the trust boundary.

\noindent Algorithm~\ref{alg:dagcov-given} separates DAG propagation and coverage computation (above the trust boundary) from score computation (below it). Algorithm~\ref{alg:dagcov} is semantically equivalent to computing $\beta(v) = \max_{s \in S} M(v,s)$ for each $v \in V$ and then invoking Algorithm~\ref{alg:dagcov-given} on $(G, \beta, \theta)$; we retain the monolithic form of Algorithm~\ref{alg:dagcov} for self-contained presentation, while sensitivity analyses that reuse construction-specific pre-computed scores (Algorithm~\ref{alg:construction}) invoke Algorithm~\ref{alg:dagcov-given} directly. Theorem~\ref{thm:algcorrect} applies to Algorithm~\ref{alg:dagcov-given} with $\beta$ as an external input: properties (i)--(iv) depend only on boundedness $\beta \in [0,1]$, which is the caller's precondition.

\subsection{Algorithm 2: Freedom-to-Operate Proof Construction}
\label{sec:fto}

\noindent\textbf{Result-asymmetry warning (read first).} The FTO result returned by Algorithm~\ref{alg:fto} is asymmetric and must not be read as a two-sided infringement determination: a $\textsc{Clear}$ result produces a formal proof of non-infringement (conditional on the ML layer's scoring accuracy: a missed or misscored relevant evidence segment can still yield a formally valid but legally misleading certificate), whereas a $\textsc{Risk}$ result does \textit{not} assert that infringement exists, only that per-element impossibility could not be established and human legal judgment is required. This ``no evidence of impossibility'' guarantee is weaker than a ``no infringement'' guarantee; see Section~\ref{sec:discussion} (Legal Implications) for the practitioner-facing discussion.

\begin{figure}[H]
\centering
\includegraphics[width=\textwidth]{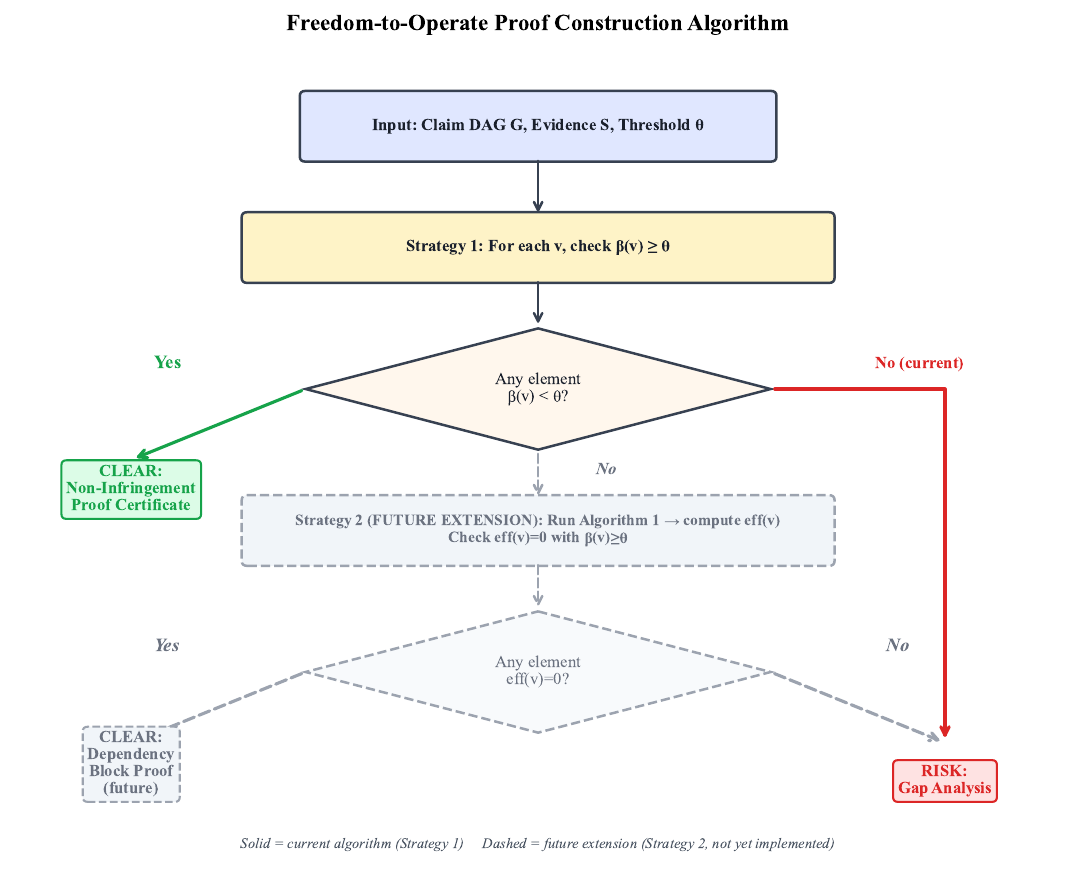}
\caption{FTO proof construction flow. \textbf{Strategy~1 (current implementation):} if any element's best match score falls below threshold, the algorithm returns $\textsc{Clear}$ with a non-infringement certificate; otherwise it returns $\textsc{Risk}$ with a gap analysis. \textbf{Strategy~2} (the dependency-chain branch shown in the figure) \textbf{is a planned future extension, not current behavior}; it would handle the case where per-element scores all clear the threshold but simultaneous satisfaction is blocked by the dependency structure under injective assignment constraints. See the Remark following Algorithm~\ref{alg:fto}.}
\label{fig:fto}
\end{figure}

Algorithm~\ref{alg:fto} (see also Figure~\ref{fig:fto}) addresses the freedom-to-operate use case: determining whether a product can be formally shown to avoid infringement. Unlike Algorithm~\ref{alg:dagcov}, which quantifies coverage, this algorithm seeks a formal proof of non-infringement, a certificate that no assignment of evidence to claim limitations achieves full coverage under Definition~\ref{def:fullcoverage}. The algorithm checks each limitation individually (lines~3--8): if no evidence segment scores above~$\theta$, infringement on that element is impossible and a \texttt{ProveNoMatch} certificate is returned immediately. The algorithm returns one of two labeled outputs: $\textsc{Clear}(\Pi)$, indicating formal non-infringement with proof certificate $\Pi : \texttt{ProofCertificate}$ encoding the impossibility of full coverage; or $\textsc{Risk}(\textit{gap})$, indicating the analysis was inconclusive, with $\textit{gap}$ an informal diagnostic report identifying the weakest elements and constraining dependency paths (see the Note on \texttt{AnalyzeWeakestElements} below). Theorem~\ref{thm:fto} proves soundness: whenever the algorithm certifies $\textsc{Clear}$, no full-coverage assignment exists.

\begin{algorithm}[H]
\caption{FTOProofConstruction \hfill {\small\normalfont\itshape Soundness: Theorem~\ref{thm:fto} (informal sketch)}}
\label{alg:fto}
\begin{algorithmic}[1]
\Require Claim DAG $G=(V,E,\tau,w)$, Evidence $S$, threshold $\theta$
\Ensure $(\textsc{Clear}, \Pi)$ or $(\textsc{Risk}, \textit{gap})$
\State $\beta(v) \gets \text{BestMatchScore}(v, S)$ for all $v \in V$
\State \Comment{\textit{Per-element impossibility}}
\For{each $v \in V$}
  \If{$\beta(v) < \theta$}
    \State $\Pi \gets \text{ProveNoMatch}(v, S, \theta)$ \Comment{no evidence meets threshold for $v$}
    \State \Return $(\textsc{Clear}, \Pi)$
  \EndIf
\EndFor
\State $\textit{gap} \gets \text{AnalyzeWeakestElements}(G, \beta, \theta)$
\State \Return $(\textsc{Risk}, \textit{gap})$
\end{algorithmic}
\end{algorithm}

\noindent\textit{Note on \texttt{AnalyzeWeakestElements}.} The function \texttt{AnalyzeWeakestElements} is an informal diagnostic procedure that ranks elements by proximity of $\beta(v)$ to $\theta$ and traces the dependency paths constraining them. Its output is not formally verified and is not included in the proof certificate~$\Pi$; it serves as practitioner-facing guidance for the $\textsc{Risk}$ case.

\noindent\textbf{Remark (Strategy~2 as Future Extension).} The current algorithm addresses FTO through per-element threshold checking. A dependency-chain impossibility strategy (Strategy~2) would handle the structurally distinct case where individual elements clear the threshold but simultaneous satisfaction is blocked by the dependency structure under injective assignment constraints. Under the current non-injective formulation, per-element checking is sufficient: if all $\beta(v) \geq \theta$, then by induction on the topological order all $\textit{eff}(v) = \beta(v) \geq \theta$, so no dependency-chain impossibility can arise. Strategy~2 is deferred to future work. It would produce a qualitatively different proof object, a dependency-chain impossibility certificate, providing additional diagnostic information about which structural dependencies block infringement.

\begin{definition}[Full Coverage]
\label{def:fullcoverage}
An assignment $A : V \to S$ achieves full coverage under threshold $\theta$ iff $\forall v \in V,\; M(v, A(v)) \geq \theta \;\wedge\; \textit{deps\_met}_A(v)$, where $\textit{deps\_met}_A(v) \Leftrightarrow \forall u \in \textit{deps}(v),\; M(u, A(u)) \geq \theta$. Note: $\textit{deps\_met}_A(v)$ is logically implied by the outer universal quantifier (since $\textit{deps}(v) \subseteq V$, the condition $\forall v \in V,\; M(v, A(v)) \geq \theta$ already ensures $M(u, A(u)) \geq \theta$ for all dependencies $u$). The conjunct is retained for clarity and to highlight the structural parallel with $\textit{deps\_met}$ in Definition~\ref{def:bestmatch}, which checks propagated effective scores rather than raw match scores. The former ($\textit{deps\_met}_A$) reflects the legal standard for infringement (all-elements rule applied per-assignment); the latter is a conservative approximation used in coverage computation that additionally enforces transitive dependency satisfaction.
\end{definition}

\begin{theorem}[Soundness of Algorithm~\ref{alg:fto}]
\label{thm:fto}
For all valid inputs $(G, S, \theta)$ where $G = (V, E, \tau, w)$ satisfies \texttt{dag\_acyclic} and $S \neq \emptyset$:
\begin{enumerate}[label=(\roman*), leftmargin=2.5em, topsep=2pt, itemsep=1pt]
\item If Algorithm~\ref{alg:fto} returns $(\textsc{Clear}, \Pi)$, then there exists $v \in V$ such that $\beta(v) < \theta$, meaning no assignment $A : V \to S$ can achieve $M(v, A(v)) \geq \theta$, and hence no full-coverage assignment exists (Definition~\ref{def:fullcoverage}). $\Pi$ encodes this per-element impossibility.
\item If Algorithm~\ref{alg:fto} returns $(\textsc{Risk}, \textit{gap})$, then $\beta(v) \geq \theta$ for all $v \in V$, meaning per-element impossibility cannot be established, and $\textit{gap}$ provides diagnostic information about proximity to the threshold.
\end{enumerate}
Note: This version does not claim completeness, a $\textsc{Risk}$ return does not mean infringement exists, only that this algorithm cannot rule it out.
\end{theorem}
\begin{proof}
\textit{(i):} Algorithm~\ref{alg:fto} returns $(\textsc{Clear}, \Pi)$ only at line~6, which is reached when some element $v$ has $\beta(v) < \theta$, meaning no $s \in S$ satisfies $M(v, s) \geq \theta$. Therefore no assignment $A$ can satisfy $M(v, A(v)) \geq \theta$, so full coverage is impossible. $\Pi$ is constructed by the \texttt{ProveNoMatch} procedure encoding the per-element impossibility argument. The internal construction of this procedure is designed for the prototype described in Section~\ref{sec:implementation} but not formally specified in this paper; see the Remark following the \texttt{ProofCertificate} structure in Section~\ref{sec:framework} for a discussion of this gap and its mitigation by the trust model.

\textit{(ii):} If the loop completes without returning $\textsc{Clear}$, all $\beta(v) \geq \theta$, so per-element impossibility cannot be established. The algorithm returns $\textsc{Risk}$ with a gap analysis identifying the weakest elements ($\beta(v)$ closest to $\theta$).
\end{proof}

\noindent\textbf{Complexity.} $O(n \cdot m \cdot d)$ for score computation (where $d$ is the embedding dimension). No DAG propagation cost is incurred: the current implementation uses only Strategy~1 and returns immediately upon finding any element with $\beta(v) < \theta$. Strategy~2 (future work) would add $O(n + |E|)$ for DAG propagation.

\noindent\textbf{Miniature example: per-element impossibility.} Consider the two-node claim reused in the UC5 example below (same claim structure; the thresholds differ because UC5 enables DOE and follows the recommended $\theta = \theta_{\text{lit}}$ configuration): $E_1$ (functional, $w = 1.5$) ``a filter element configured to attenuate high-frequency noise'' and $E_2$ (wherein, $w = 3.0$) ``wherein the filter element operates by subtracting a delayed copy of the input signal from the original signal,'' with $E_1 \to E_2$ and $\theta = 0.65$. Suppose the accused product is a purely analog amplifier with no filtering stage, so the ML layer returns $\beta(E_1) = 0.58$ and $\beta(E_2) = 0.44$. Algorithm~\ref{alg:fto} evaluates $\beta(E_1) < \theta$ in the first iteration of the per-element loop: $\beta(E_1) = 0.58 < 0.65 = \theta$, so no evidence segment can satisfy $E_1$ at threshold, and by the all-elements rule no assignment $A$ can achieve full coverage. The algorithm returns $(\textsc{Clear}, \Pi)$ at line~6, where $\Pi$ is a \texttt{ProveNoMatch} certificate encoding $\beta(E_1) < \theta$. The remaining element $E_2$ is never examined: per-element impossibility at any single node suffices. Under the current formulation, Strategy~2 (dependency-chain impossibility; see Remark above) is not required for this result because per-element checking already discharges the obligation.

\subsection{Algorithm 3: Claim Construction Sensitivity}

Algorithm~\ref{alg:construction} addresses claim construction sensitivity: the same patent claim may yield different coverage outcomes depending on how contested terms are legally interpreted. Given $k$ alternative constructions $\{I_1, \ldots, I_k\}$ ordered from broadest ($I_1$, the baseline) to narrowest, the algorithm determines which constructions yield sufficient coverage, identifies the claim terms whose interpretation is outcome-determinative, and computes the threshold construction $I_{\text{threshold}}$, the satisfied construction with the fewest satisfied elements (smallest scope). For each construction, it recomputes match scores under the ML layer and invokes Algorithm~\ref{alg:dagcov-given} to obtain $W_{\text{cov},j}$ (lines~1--5). It then identifies determinative terms (lines~6--15): for each term~$t \in \Sigma$, the algorithm constructs a single-term perturbation $I_t$ that differs from $I_1$ on $t$ alone and tests whether this isolated change alters the coverage-satisfaction outcome; $t$ is determinative if and only if it does. For unsatisfied constructions, it identifies which limitations are zeroed by dependency cascade versus direct score reduction (lines~16--18). The threshold construction is selected as the satisfied construction with the smallest scope (line~22), where $|\textit{scope}(I_j)| = |\{v : \textit{eff}_j(v) \geq \theta\}|$. Note: $|\textit{scope}|$ is a DAG-level metric (count of elements above threshold) that serves as a proxy for construction restrictiveness; it does not directly measure legal scope (the set of products that would infringe under a given construction), which requires semantic reasoning beyond the formal framework. The sensitivity report~$R$ returned by the algorithm consists of: coverage scores $\{W_{\text{cov},j}\}_{j=1}^k$, satisfaction outcomes $\{\textit{satisfied}(j)\}_{j=1}^k$, the set of determinative terms, coverage-breaking elements $\textit{breakers}(j)$ per unsatisfied construction, and the threshold construction $I_{\text{threshold}}$. The case study in Section~\ref{sec:implementation} concretizes this analysis for the memory module patent. Theorem~\ref{thm:construction} proves correctness of the coverage computation, determinative term identification, and threshold construction selection.

\begin{algorithm}[H]
\caption{ConstructionSensitivity \hfill {\small\normalfont\itshape Correctness: Theorem~\ref{thm:construction} (informal sketch)}}
\label{alg:construction}
\begin{algorithmic}[1]
\Require Claim DAG $G$, Evidence $S$, Constructions $\{I_1, \ldots, I_k\}$ where $I_1$ is the broadest (baseline) construction
\Ensure Sensitivity report $R$, formal proofs
\For{$j = 1$ to $k$}
  \State $\beta_j \gets \text{ComputeScoresUnder}(G, S, I_j)$ \Comment{$\beta_j : V \to [0,1]$; below trust boundary}
  \State $(W_{\text{cov},j}, \_) \gets \text{WeightedDAGCoverageGivenScores}(G, \beta_j, \theta)$ \Comment{Alg.~\ref{alg:dagcov-given}}
  \State $\textit{satisfied}(j) \gets (W_{\text{cov},j} \geq \textit{threshold\_cov})$
\EndFor
\State $\textit{determinative} \gets \emptyset$
\For{each term $t \in \Sigma$}
  \State $I_t \gets$ construction identical to $I_1$ except with term $t$ interpreted under $I_k$ (narrowest) \Comment{single-term perturbation}
  \State $\beta_t \gets \text{ComputeScoresUnder}(G, S, I_t)$
  \State $(W_t, \_) \gets \text{WeightedDAGCoverageGivenScores}(G, \beta_t, \theta)$ \Comment{Alg.~\ref{alg:dagcov-given}}
  \State $\textit{satisfied}_t \gets (W_t \geq \textit{threshold\_cov})$
  \If{$\textit{satisfied}_t \neq \textit{satisfied}(1)$}
    \State $\textit{determinative} \gets \textit{determinative} \cup \{t\}$
  \EndIf
\EndFor
\For{$j$ where $\neg\textit{satisfied}(j)$}
  \State $\textit{breakers}(j) \gets \{v \in V \mid \textit{eff}_j(v) < \theta \wedge \textit{eff}_1(v) \geq \theta\}$ \Comment{$I_1$ = broadest construction (baseline)}
\EndFor
\If{$\nexists j: \textit{satisfied}(j)$}
  \State $I_{\text{threshold}} \gets \bot$ \Comment{no construction satisfies coverage}
\Else
  \State $I_{\text{threshold}} \gets \arg\min_{j:\textit{satisfied}(j)} |\textit{scope}(I_j)|$ \Comment{$|\textit{scope}(I_j)| = |\{v : \textit{eff}_j(v) \geq \theta\}|$; ties broken by smallest $j$}
\EndIf
\State $\Pi \gets \text{ProveConstructionResults}(G, S, \{W_{\text{cov},j}\}, \textit{determinative}, I_{\text{threshold}})$
\State \Return $(R, \Pi)$
\end{algorithmic}
\end{algorithm}

\begin{theorem}[Correctness of Algorithm~\ref{alg:construction}]
\label{thm:construction}
For all valid inputs $(G, S, \{I_1, \ldots, I_k\})$ where $G$ satisfies \texttt{dag\_acyclic}, $S \neq \emptyset$, and $I_1$ is the broadest construction:
\begin{enumerate}[label=(\roman*), leftmargin=2.5em, topsep=2pt, itemsep=1pt]
\item \textit{Coverage correctness:} For each $j$, $W_{\text{cov},j}$ is correctly computed by Algorithm~\ref{alg:dagcov-given} under construction $I_j$ with pre-computed scores $\beta_j$, satisfying Theorem~\ref{thm:algcorrect}.
\item \textit{Determinative soundness:} If term $t \in \textit{determinative}$, then the tested perturbation $I_1 \to I_t$, changing $t$'s interpretation from $I_1$ to $I_k$ while holding all other terms at $I_1$, changes the satisfaction outcome: $\textit{satisfied}_t \neq \textit{satisfied}(1)$.
\item \textit{Determinative completeness (tested perturbation):} If term $t \notin \textit{determinative}$, then the tested perturbation $I_1 \to I_t$ does not change the satisfaction outcome: $\textit{satisfied}_t = \textit{satisfied}(1)$.
\item \textit{Determinative completeness under monotone constructions:} Assume the family $\{I_1, \ldots, I_k\}$ is term-wise monotonically ordered: narrowing any single term's interpretation (moving that term alone from $I_1$'s reading toward $I_k$'s reading) never increases $\beta(v)$ for any $v \in V$. Under this assumption, if $t \notin \textit{determinative}$, then no single-term perturbation of $t$ to any intermediate $I_j$ changes the satisfaction outcome either. Without the monotonicity assumption, completeness is guaranteed only for the tested extreme perturbation $I_1 \to I_t$; intermediate constructions must be tested separately if monotonicity cannot be established. Term-wise monotonicity is an \emph{empirical} property of the concrete scoring function $M$ that practitioners must validate for their implementation: it tends to hold for scoring functions whose lexical component dominates and whose semantic component is calibrated on matched-vocabulary corpora, but it can fail for TF-IDF plus embedding-cosine combinations when narrowing a term's interpretation removes ambiguous vocabulary and concentrates the similarity mass on a smaller, more precisely matched subset. Implementations should verify monotonicity on a representative construction pair (e.g., $I_1$ and $I_k$) before relying on the completeness extension of (iv). A concrete recipe: for each contested term $t$, compute $\beta(v)$ under $I_1$ and under the single-term perturbation $I_t$ (i.e., $I_1$ with only $t$'s interpretation narrowed to $I_k$'s reading) for every $v \in V$ where $t$ appears in $v$'s claim-language text; check that $\beta(v; I_1) \geq \beta(v; I_t)$ holds for every such $v$. If any counterexample is observed (an increase when narrowing), fall back to testing every intermediate construction $I_j$ explicitly rather than relying on the monotone-ordering extension.
\item \textit{Threshold construction:} If any construction is satisfied, $I_{\text{threshold}}$ is the satisfied construction with the smallest scope $|\{v : \textit{eff}_j(v) \geq \theta\}|$. If no construction is satisfied, $I_{\text{threshold}} = \bot$.
\end{enumerate}
\end{theorem}
\begin{proof}
\textit{(i)} follows directly from Theorem~\ref{thm:algcorrect} applied to Algorithm~\ref{alg:dagcov-given} with inputs $(G, \beta_j, \theta)$, invoked at line~3 for each construction $I_j$.

\textit{(ii)} follows from the algorithm at lines~7--14: $t$ is added to $\textit{determinative}$ iff $\textit{satisfied}_t \neq \textit{satisfied}(1)$ by the line~12 decision, so every $t \in \textit{determinative}$ witnesses the claimed outcome change on the tested perturbation.

\textit{(iii)} For the tested perturbation, $t \notin \textit{determinative}$ means the algorithm found $\textit{satisfied}_t = \textit{satisfied}(1)$ by construction.

\textit{(iv)} For the extension to intermediate perturbations under term-wise monotonicity: any intermediate single-term perturbation of $t$ to $I_j$ (changing $t$'s interpretation to $I_j$'s while holding all other terms at $I_1$) yields scores $\beta'$ satisfying $\beta_1(v) \geq \beta'(v) \geq \beta_t(v)$ for all $v$, by the monotonicity assumption applied to the single term $t$. By Theorem~\ref{thm:algcorrect}(iv), $W_{\text{cov}}(\beta_1) \geq W_{\text{cov}}(\beta') \geq W_{\text{cov}}(\beta_t)$. If $\textit{satisfied}(1) = \textit{satisfied}_t = \textit{true}$, both endpoints exceed $\textit{threshold\_cov}$, so $W_{\text{cov}}(\beta') \geq W_{\text{cov}}(\beta_t) \geq \textit{threshold\_cov}$ and the perturbed outcome is $\textit{true}$. If both are $\textit{false}$, $W_{\text{cov}}(\beta') \leq W_{\text{cov}}(\beta_1) < \textit{threshold\_cov}$ and the perturbed outcome is $\textit{false}$. Either way the intermediate outcome matches.

\textit{(v)} When the set $\{j : \textit{satisfied}(j)\}$ is non-empty, the result follows from the $\arg\min$ at line~22, which selects over all $j$ with $\textit{satisfied}(j) = \text{true}$, minimizing $|\textit{scope}(I_j)|$. When the set is empty (line~19), $I_{\text{threshold}} = \bot$ is returned, correctly indicating that no construction achieves sufficient coverage.
\end{proof}

\noindent\textit{Remark on determinative term identification.} The revised Algorithm~\ref{alg:construction} tests each term's independent causal effect by constructing single-term perturbations $I_t$ that differ from $I_1$ on exactly one term. This provides both soundness and completeness for independent determinativeness: a term is identified as determinative if and only if changing its interpretation alone changes the coverage-satisfaction outcome. Completeness holds under the assumption that constructions are monotonically ordered (a narrower construction never increases a term's match score); under this assumption, a term that is non-determinative under the extreme perturbation ($I_1 \to I_k$) is also non-determinative under any intermediate perturbation. If non-monotone score changes can occur, intermediate constructions should be tested separately. The cost of this approach is $O(|\Sigma|)$ construction evaluations, each requiring $O(n \cdot m \cdot d)$ for score computation, giving $O(|\Sigma| \cdot n \cdot m \cdot d)$ total, compared to $O(k \cdot n \cdot m \cdot d)$ in the original formulation. For typical patent claims where $|\Sigma|$ is small (tens of terms), this cost is negligible relative to the score computation.

\begin{figure}[H]
\centering
\includegraphics[width=\textwidth]{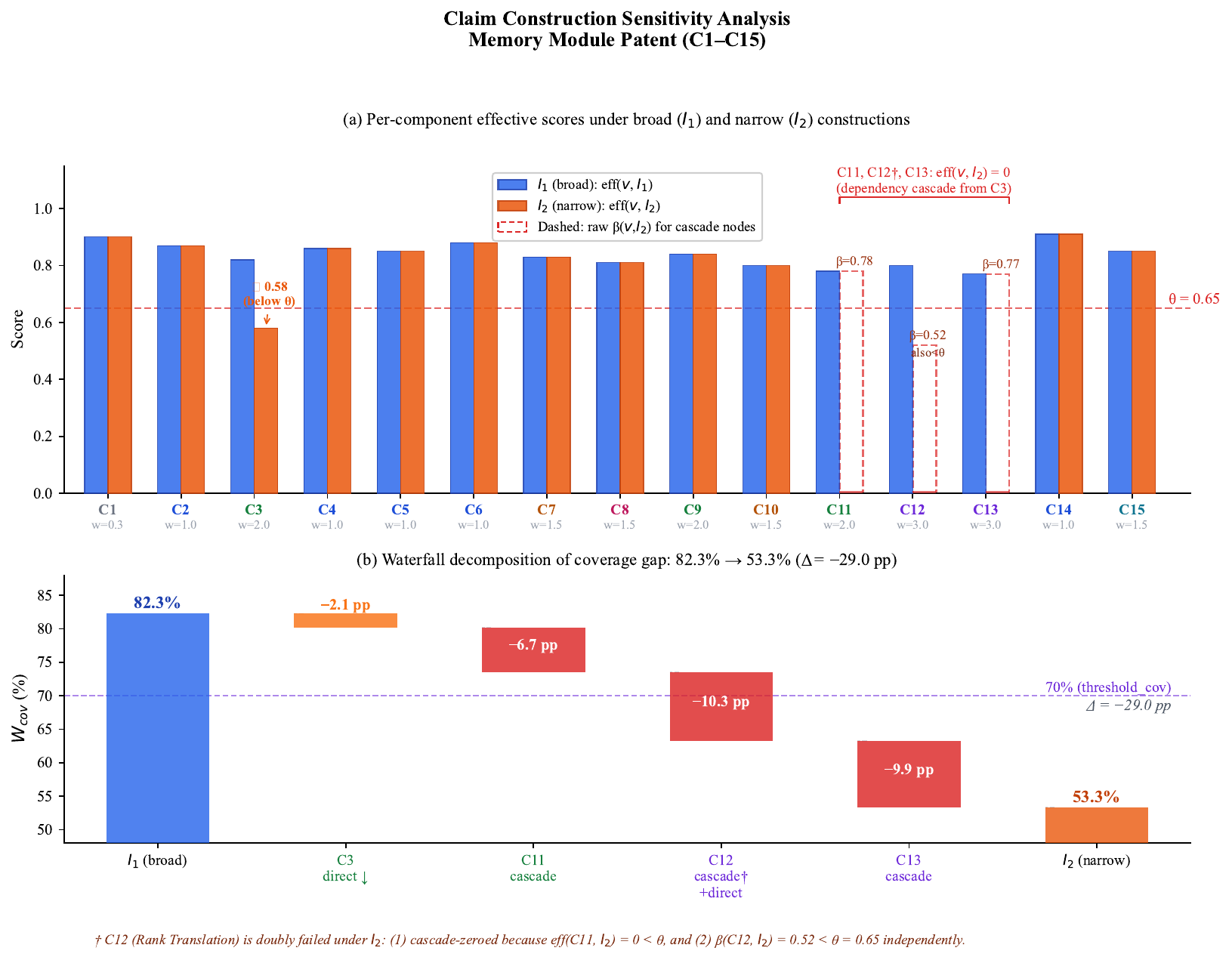}
\caption{Claim construction sensitivity analysis for the memory module patent (C1--C15). (a)~Per-component effective scores under broad $I_1$ (blue) and narrow $I_2$ (orange/red), computed via Definition~\ref{def:bestmatch}. Dashed outlines on C11, C12\textsuperscript{\dag}, C13 show raw $\beta(v, I_2)$; their eff is zeroed by dependency propagation. C3's score drops below $\theta = 0.65$, triggering the cascade. (b)~Waterfall decomposition of the 29.0~pp coverage gap. Direct raw-score drops (C3, C12\textsuperscript{\dag}~direct) account for 5.7~pp (${\sim}$20\%); cascade-driven zeroing (C11, C12\textsuperscript{\dag}~cascade, C13) accounts for 23.3~pp (${\sim}$80\%). $W_{\text{cov}}(I_1) = 82.3\%$ ($\geq 70\%$, satisfied); $W_{\text{cov}}(I_2) = 53.3\%$ ($< 70\%$, not satisfied). Determinative terms: C3 and C12.}
\label{fig:construction}
\end{figure}

\noindent\textbf{Complexity.} $O(k \cdot n \cdot m \cdot d)$ for $k$ constructions. Determinative term identification is $O(|\Sigma| \cdot n \cdot m \cdot d)$ (one construction evaluation per term).

\noindent\textbf{Case study: claim construction sensitivity for the memory module patent.}
We apply Algorithm~\ref{alg:construction} to the memory module case study (Section~\ref{sec:implementation}) using the claim DAG $G$ from Figure~\ref{fig:claimdag}, which contains fifteen atomic limitations $C_1$--$C_{15}$ and has total weight $\sum_{v \in V} w(\tau(v)) = 23.3$. Two claim constructions are defined for the key contested terms. \textit{Construction $I_1$ (broad):} the term ``rank translation'' in $C_{12}$ is interpreted to encompass any mechanism by which the memory module presents a different number of ranks to the controller than it physically contains. The term ``first number of ranks'' in $C_3$ refers to any quantity of logical ranks. \textit{Construction $I_2$ (narrow):} the term ``rank translation'' in $C_{12}$ requires a dedicated rank translation circuit with an explicit address mapping stage. The term ``first number of ranks'' in $C_3$ requires a specific 2:1 ratio to the second number in $C_9$. The system parameter is set to $\theta = 0.65$ throughout (Definition~\ref{def:bestmatch}). Only $C_3$ and $C_{12}$ receive different raw scores across constructions; all other raw scores are construction-independent. Complete scores for all 15 nodes are given in Table~\ref{tab:scores}.

Under $I_1$, all fifteen nodes satisfy $\theta$ and all dependencies are met, giving $W_{\text{cov}}(G, I_1) = 19.165 / 23.3 \times 100 \approx 82.3\%$. Under $I_2$, the cascade unfolds as follows: $C_3$'s raw score drops to $\beta(C_3, I_2) = 0.58 < \theta$, but its dependency ($C_2$) is met, so $\textit{eff}(C_3, I_2) = 0.58$. Since $0.58 < \theta$, $C_{11}$ (which depends on $C_3$) fails $\textit{deps\_met}$ and is zeroed despite $\beta(C_{11}, I_2) = 0.78 \geq \theta$. $C_{12}$ then fails because $\textit{eff}(C_{11}, I_2) = 0 < \theta$, and additionally $\beta(C_{12}, I_2) = 0.52 < \theta$ (a double failure). $C_{13}$ is zeroed because $\textit{eff}(C_{12}, I_2) = 0 < \theta$. This gives $W_{\text{cov}}(G, I_2) = 12.415 / 23.3 \times 100 \approx 53.3\%$.

The 29.0~pp gap decomposes into two components: \textit{direct raw-score reductions} account for 5.7~pp (${\sim}$20\%), comprising $C_3$ (2.1~pp) and $C_{12}$'s direct drop (3.6~pp under the counterfactual that $\textit{deps\_met}(C_{12})$ holds, computed as $(0.80 - 0.52) \times 3.0 / 23.3 \times 100 \approx 3.6$~pp (exactly $3.605$~pp per Table~\ref{tab:waterfall})); \textit{dependency-cascade zeroing} accounts for 23.3~pp (${\sim}$80\%), comprising $C_{11}$ (6.7~pp), $C_{12}$'s cascade component (6.7~pp), and $C_{13}$ (9.9~pp) (the cascade contribution $23.3$~pp coincidentally matches the total DAG weight $23.3$; this is a numerical coincidence for this case study, not a definitional relationship). The DAG dependency mechanism is therefore outcome-determinative: a 0.24-point reduction in $C_3$'s raw score propagates through the dependency structure to eliminate three of the highest-weighted nodes from the coverage sum.

We set the coverage-satisfaction threshold $\textit{threshold\_cov} = 70\%$ for Algorithm~\ref{alg:construction} line~4; this is distinct from the dependency-enforcement threshold $\theta = 0.65$ used within Algorithm~\ref{alg:dagcov}. Algorithm~\ref{alg:construction} identifies $C_3$ (``first number of ranks'', quantitative) and $C_{12}$ (``rank translation'', wherein) as the determinative terms. Under $\textit{threshold\_cov} = 70\%$: $W_{\text{cov}}(G, I_1) = 82.3\% \geq 70\%$ (satisfied) and $W_{\text{cov}}(G, I_2) = 53.3\% < 70\%$ (not satisfied). The threshold construction is $I_{\text{threshold}} = I_1$. A claim construction ruling adopting the narrow interpretation of either $C_{12}$ or $C_3$ would therefore be outcome-determinative for infringement. Figure~\ref{fig:construction} visualises this analysis.

\subsection{Algorithm 4: Cross-Claim Consistency Verification}

Algorithm~\ref{alg:consistency} addresses cross-claim consistency: if the same technical term appears across multiple patents in a portfolio, legal coherence requires a consistent interpretation in all of them. Inconsistent interpretations can create prosecution history estoppel vulnerabilities or invite claim construction challenges. The algorithm takes as input a portfolio of $p$ claim DAGs, a shared vocabulary $\Sigma = \bigcup_{i=1}^p \textit{terms}(G_i)$ (the union of all term vocabularies across claims), and an interpretation function $f(i,t)$ computed via LLM-assisted extraction (below the trust boundary). It exhaustively checks all pairs $(i,j)$ of claims sharing each term~$t$, returning a formal counterexample if any inconsistency is found. Because the interpretation function is ML-computed, the formal layer certifies only that the extracted interpretations are mutually consistent, it cannot certify that any individual interpretation correctly captures the legal meaning of the term. The algorithm returns one of two labeled outputs: $\textsc{Inconsistent}(i, j, t, f(i,t), f(j,t))$, identifying a specific counterexample (the pair of claims, the shared term, and the conflicting interpretations); or $(\textsc{Consistent}, \Pi)$, where $\Pi : \texttt{ProofCertificate}$ certifies universal consistency across the portfolio. The restriction to pairs with $i < j$ is without loss of generality by symmetry. Theorem~\ref{thm:consistency} establishes both soundness (a reported inconsistency reflects a genuine mismatch) and completeness (a $\textsc{Consistent}$ result means no inconsistency was overlooked).

\begin{algorithm}[H]
\caption{CrossClaimConsistency \hfill {\small\normalfont\itshape Correctness: Theorem~\ref{thm:consistency} (informal sketch)}}
\label{alg:consistency}
\begin{algorithmic}[1]
\Require Patent portfolio $\{G_1, \ldots, G_p\}$, shared vocabulary $\Sigma$, interpretation function $f : \{1,\ldots,p\} \times \Sigma \to \mathsf{Int}$
\Ensure \textsc{Consistent} or counterexample $(i, j, t, f(i,t), f(j,t))$
\For{each term $t \in \Sigma$}
  \State $\textit{users}(t) \gets \{i \mid \text{uses}(G_i, t)\}$
\EndFor
\For{$i = 1$ to $p$}
  \For{each $t \in \textit{terms}(G_i)$} \Comment{$\textit{terms}(G_i) = \bigcup_{v \in V_i} \textit{ann}_i(v)$, Def.~\ref{def:annotation}}
    \State $f(i, t) \gets \text{ExtractInterpretation}(G_i, t)$ \Comment{LLM-assisted}
  \EndFor
\EndFor
\For{each $t \in \Sigma$}
  \For{each $(i, j) \in \textit{users}(t) \times \textit{users}(t)$ where $i < j$}
    \If{$f(i, t) \neq f(j, t)$}
      \State \Return $\textsc{Inconsistent}(i, j, t, f(i,t), f(j,t))$
    \EndIf
  \EndFor
\EndFor
\State $\Pi \gets \text{ProveConsistency}(\Sigma, f, \textit{users})$
\State \Return $(\textsc{Consistent}, \Pi)$
\end{algorithmic}
\end{algorithm}

\begin{theorem}[Correctness of Algorithm~\ref{alg:consistency}]
\label{thm:consistency}
For all valid inputs $(\{G_1, \ldots, G_p\}, \Sigma, f)$ where $f : \{1,\ldots,p\} \times \Sigma \to \mathsf{Int}$ is the interpretation function (computed by \texttt{ExtractInterpretation}, which is LLM-assisted and below the trust boundary):
\begin{enumerate}[label=(\roman*), leftmargin=2.5em, topsep=2pt, itemsep=1pt]
\item \textit{Soundness:} If Algorithm~\ref{alg:consistency} returns $\textsc{Inconsistent}(i, j, t, f(i,t), f(j,t))$, then $\text{uses}(G_i, t) \wedge \text{uses}(G_j, t) \wedge f(i,t) \neq f(j,t)$.
\item \textit{Completeness:} If Algorithm~\ref{alg:consistency} returns $\textsc{Consistent}$, then $\forall i,j.\; \forall t \in \Sigma.\; \text{uses}(G_i,t) \wedge \text{uses}(G_j,t) \Rightarrow f(i,t) = f(j,t)$.
\end{enumerate}
The restriction to pairs $(i, j)$ with $i < j$ is without loss of generality since equality is symmetric. Note: the formal verification certifies consistency of the ML-computed interpretations $f(i,t)$; it cannot certify that $f(i,t)$ correctly captures the legal meaning of term $t$ in $G_i$.
\end{theorem}
\begin{proof}
\textit{(i):} The algorithm returns $\textsc{Inconsistent}$ only at line~12, which is reached only when $f(i,t) \neq f(j,t)$ for $(i,j)$ with $i < j$ and $t \in \Sigma$ such that both $i, j \in \textit{users}(t)$, i.e., $\text{uses}(G_i, t) \wedge \text{uses}(G_j, t)$.

\textit{(ii):} Algorithm~\ref{alg:consistency} enumerates all terms $t \in \Sigma$ (line~9) and all pairs $(i,j) \in \textit{users}(t) \times \textit{users}(t)$ with $i < j$ (line~10). This enumeration is exhaustive over all pairs using each shared term. If any pair has $f(i,t) \neq f(j,t)$, line~11 detects it and line~12 returns $\textsc{Inconsistent}$. Therefore if the algorithm reaches line~17 and returns $\textsc{Consistent}$, no such inconsistent pair exists.
\end{proof}

\noindent\textbf{Complexity.} $O(p^2 \cdot |\Sigma|)$ for the pairwise consistency checks (lines~9--15). In practice, the extraction phase (lines~4--8) requires $p \times |\Sigma|$ calls to \texttt{ExtractInterpretation}, which lies below the trust boundary. For large portfolios (e.g., $p = 5{,}000$, $|\Sigma| = 200$), the $10^6$ LLM invocations may dominate wall-clock time despite the formal complexity being dominated by the $O(p^2 \cdot |\Sigma|)$ comparison step.

\noindent\textbf{Case study: cross-claim consistency for a two-patent portfolio.} Consider a portfolio of two patents in the memory module technology area. Patent $P_1$ is the memory buffer module claim of Section~\ref{sec:implementation} and uses the terms ``rank translation'' (in $C_{12}$), ``first number of ranks'' (in $C_3$), and ``DDR memory devices'' (in $C_2$). Patent $P_2$ is a related patent on memory controllers and uses ``rank translation'' (in $D_4$), ``DDR memory devices'' (in $D_2$), and ``registered memory module'' (in $D_7$). The shared vocabulary is therefore $\Sigma = \{\text{``rank translation''}, \text{``DDR memory devices''}\}$; the terms ``first number of ranks'' and ``registered memory module'' appear in only one patent each and require no consistency check.

The ML layer (LLM-assisted \texttt{ExtractInterpretation}) produces interpretation strings $f(i, t)$ for each shared term in each patent. For $P_1$: $f(P_1, \text{``rank translation''}) = $ ``any mechanism mapping a second rank count to a first rank count, implemented in the rank interposition circuit''; $f(P_1, \text{``DDR memory devices''}) = $ ``double-data-rate SDRAM devices conforming to JEDEC standards.'' For $P_2$: $f(P_2, \text{``rank translation''}) = $ ``a firmware-controlled remapping of logical rank addresses to physical rank addresses''; $f(P_2, \text{``DDR memory devices''}) = $ ``double-data-rate SDRAM devices conforming to JEDEC standards.''

Running Algorithm~\ref{alg:consistency}: $\textit{users}(\text{``rank translation''}) = \{P_1, P_2\}$ and $\textit{users}(\text{``DDR memory devices''}) = \{P_1, P_2\}$. Only the pair $(P_1, P_2)$ is checked for each term. When the loop reaches ``rank translation,'' $f(P_1, \text{``rank translation''}) \neq f(P_2, \text{``rank translation''})$, so line~12 returns $\textsc{Inconsistent}(P_1, P_2, \text{``rank translation''}, f(P_1, \cdot), f(P_2, \cdot))$ and the algorithm halts. The ``DDR memory devices'' pair, which happens to be consistent in this portfolio, is either examined and passes (if iterated first) or never examined (if iterated after the inconsistent term); the final result is order-independent. The returned counterexample is itself the certificate: a human attorney can independently confirm that the two interpretation strings are unequal. Legally, the inconsistency matters because a narrow construction of ``rank translation'' in $P_2$ litigation (requiring firmware control) could be imported by prosecution history into $P_1$, narrowing $C_{12}$.

A contrast case: suppose $P_2$'s interpretation is revised to ``any mechanism mapping a second rank count to a first rank count, implemented in dedicated circuitry,'' producing string equality with $f(P_1, \text{``rank translation''})$ after normalization. Both pairs now compare equal, and the algorithm returns $(\textsc{Consistent}, \Pi)$ with $\Pi$ witnessing that every shared-term pair was checked. Note that the equality test on line~11 is syntactic string equality after normalization; semantic equivalence of interpretations remains the responsibility of the ML layer (below the trust boundary).

\subsection{Algorithm 5: Doctrine of Equivalents}

Algorithm~\ref{alg:doe} extends coverage analysis to the doctrine of equivalents, which permits a finding of infringement even when the accused product does not literally satisfy every claim limitation, provided each non-literal element is equivalent under the function-way-result test of \textit{Graver Tank} subject to the prosecution history estoppel constraints of \textit{Festo}. The algorithm has two phases. Phase~1 (lines~1--18) classifies each limitation into one of three categories from the $\textsc{MatchType}$ inductive type (Definition~\ref{def:matchtype}): $\textsc{Literal}$ if $\beta(v) \geq \theta_{\text{lit}}$; $\textsc{Equivalent}$ if $\beta(v) < \theta_{\text{lit}}$ but the best non-estopped evidence $s^*(v)$ (Definition~\ref{def:bestevidence}) passes all three prongs of the function-way-result test at threshold~$\theta_{\text{eq}}$ and is not vitiated; or $\textsc{NoMatch}$ otherwise. The estoppel filter (line~6) removes evidence within $\textit{EstoppedRegion}(v, \textit{PH})$, enforcing the \textit{Festo} presumption. Phase~2 (lines~20--28) propagates dependency constraints on the DOE-adjusted scores in topological order, mirroring Algorithm~\ref{alg:dagcov}. The resulting $W_{\text{DOE}}$ is not directly comparable to $W_{\text{cov}}$, as equivalence can both increase and decrease coverage relative to the literal analysis (Definition~\ref{def:doeeff}). The formal guarantees are conditional in a stronger sense than for Algorithms~\ref{alg:dagcov}--\ref{alg:consistency}: the certificate certifies mathematical correctness of classifications given ML-computed scores, but not the legal correctness of those scores as measures of function, way, and result. Theorem~\ref{thm:doe} establishes classification correctness, dependency enforcement, and coverage bounds.

\begin{algorithm}[H]
\caption{DoctrineOfEquivalents \hfill {\small\normalfont\itshape Correctness: Theorem~\ref{thm:doe} (informal sketch; stronger conditionality, see \S\ref{sec:algorithms})}}
\label{alg:doe}
\begin{algorithmic}[1]
\Require Claim DAG $G = (V, E, \tau, w)$, Evidence $S$, thresholds $\theta_{\text{lit}}, \theta_{\text{eq}}$ with $\theta_{\text{eq}} < \theta_{\text{lit}}$, propagation threshold $\theta_{\text{prop}} \in \{\theta_{\text{lit}}, \theta_{\text{eq}}\}$ (default $\theta_{\text{prop}} := \theta_{\text{lit}}$), prosecution history $\textit{PH}$, parsing functions $\textit{func}, \textit{way}, \textit{res}$, discount factor $\delta$
\Ensure Match classification $\textit{match} : V \to \textsc{MatchType}$, DOE-adjusted coverage $W_{\text{DOE}}$, proof certificate $\Pi$
\For{each $v \in V$}
  \State $\beta(v) \gets \max_{s \in S} M(v, s)$
  \If{$\beta(v) \geq \theta_{\text{lit}}$}
    \State $\textit{match}(v) \gets \textsc{Literal}(\beta(v))$; $\textit{eff}_{\text{DOE}}(v) \gets \beta(v)$
  \Else
    \State $S_v \gets S \setminus \textit{EstoppedRegion}(v, \textit{PH})$ \Comment{Def.~\ref{def:estoppel}}
    \If{$S_v = \emptyset$} $\textit{match}(v) \gets \textsc{NoMatch}$; $\textit{eff}_{\text{DOE}}(v) \gets \bot\; (= 0)$; \textbf{continue}
    \EndIf
    \State $s^* \gets \arg\max_{s \in S_v} M(v, s)$ \Comment{canonical tie-breaking}
    \State $f_{\text{sim}} \gets M(\textit{func}(v), s^*)$; $w_{\text{sim}} \gets M(\textit{way}(v), s^*)$; $r_{\text{sim}} \gets M(\textit{res}(v), s^*)$ \Comment{Def.~\ref{def:fwr}}
    \If{$f_{\text{sim}} \geq \theta_{\text{eq}}$ \textbf{and} $w_{\text{sim}} \geq \theta_{\text{eq}}$ \textbf{and} $r_{\text{sim}} \geq \theta_{\text{eq}}$ \textbf{and} $\neg\textit{vitiated}(v, s^*)$} \Comment{$s^*$ and $f_{\text{sim}}, w_{\text{sim}}, r_{\text{sim}}$ are well-defined on this branch: the empty-$S_v$ case continues at line~7.}
      \State $\textit{match}(v) \gets \textsc{Equivalent}(f_{\text{sim}}, w_{\text{sim}}, r_{\text{sim}})$
      \State $\textit{eff}_{\text{DOE}}(v) \gets \delta \cdot \min(f_{\text{sim}}, w_{\text{sim}}, r_{\text{sim}})$ \Comment{Def.~\ref{def:doeeff}}
    \Else
      \State $\textit{match}(v) \gets \textsc{NoMatch}$; $\textit{eff}_{\text{DOE}}(v) \gets \bot$
    \EndIf
  \EndIf
\EndFor
\State \Comment{\textit{Phase 2: Dependency enforcement in topological order}}
\State $\textit{levels} \gets \text{TopologicalSort}(G)$
\For{level $\ell = 0$ to $\max\_\text{level}$}
  \For{each $v$ at level $\ell$}
    \State $\textit{deps\_met\_DOE}(v) \gets \forall u \in \textit{deps}(v): \textit{eff}_{\text{DOE}}(u) \geq \theta_{\text{prop}}$ \Comment{Def.~\ref{def:doeeff}}
    \If{$\neg\textit{deps\_met\_DOE}(v)$}
      \State $\textit{eff}_{\text{DOE}}(v) \gets 0$
    \EndIf
  \EndFor
\EndFor
\State $W_{\text{DOE}} \gets \frac{\sum_{v} w(\tau(v)) \cdot \textit{eff}_{\text{DOE}}(v)}{\sum_{v} w(\tau(v))} \times 100$ \Comment{Def.~\ref{def:doeeff}}
\State $\Pi \gets \text{ProveDOE}(\textit{match}, \textit{eff}_{\text{DOE}}, W_{\text{DOE}}, G, S, \textit{PH})$
\State \Return $(\textit{match}, W_{\text{DOE}}, \Pi)$
\end{algorithmic}
\end{algorithm}

\begin{theorem}[Correctness of Algorithm~\ref{alg:doe}]
\label{thm:doe}
For all valid inputs $(G, S, \theta_{\text{lit}}, \theta_{\text{eq}}, \theta_{\text{prop}}, \textit{PH}, \delta)$ where $G$ satisfies \texttt{dag\_acyclic}, $S \neq \emptyset$, $\theta_{\text{eq}} < \theta_{\text{lit}}$, and $\theta_{\text{prop}} \in \{\theta_{\text{lit}}, \theta_{\text{eq}}\}$, Algorithm~\ref{alg:doe} computes $(\textit{match}, W_{\text{DOE}}, \Pi)$ satisfying:
\begin{enumerate}[label=(\roman*), leftmargin=2.5em, topsep=2pt, itemsep=1pt]
\item \textit{Classification correctness:} For each $v \in V$, $\textit{match}(v) \in \textsc{MatchType}$ satisfies the proof obligations of Definition~\ref{def:matchtype}.
\item \textit{Dependency enforcement:} $\textit{eff}_{\text{DOE}}(v) = 0$ for all $v$ where $\neg\textit{deps\_met\_DOE}(v)$.
\item \textit{Bounds:} $0 \leq W_{\text{DOE}} \leq 100$.
\item \textit{Score relationship (summary):} $\textit{eff}_{\text{DOE}}$ and $\textit{eff}$ can differ, and consequently $W_{\text{DOE}}$ and $W_{\text{cov}}$ are not directly comparable in general. The detailed per-element analysis is stated separately as Proposition~\ref{prop:doe-cov} below to avoid embedding a multi-paragraph analysis inside a theorem statement.
\end{enumerate}
\end{theorem}

\begin{proposition}[Relationship between $W_{\text{DOE}}$ and $W_{\text{cov}}$]
\label{prop:doe-cov}
Under the hypotheses of Theorem~\ref{thm:doe}, for each $v \in V$ the following element-level relationships hold between Algorithm~\ref{alg:doe}'s $\textit{eff}_{\text{DOE}}(v)$ and Algorithm~\ref{alg:dagcov}'s $\textit{eff}(v)$:

\smallskip\noindent\textit{Literal elements} ($\beta(v) \geq \theta_{\text{lit}}$): Phase~1 assigns $\textit{eff}_{\text{DOE}}(v) = \beta(v)$; Phase~2 retains this score when $\textit{deps\_met\_DOE}(v)$ holds under $\theta_{\text{prop}}$ and zeros it otherwise. Hence $\textit{eff}_{\text{DOE}}(v) \in \{\beta(v), 0\}$, and $\textit{match}(v) = \textsc{Literal}$ can coexist with $\textit{eff}_{\text{DOE}}(v) = 0$ when a transitive dependency fails the Phase~2 check. The equality $\textit{eff}_{\text{DOE}}(v) = \textit{eff}(v)$ is not guaranteed even for Literal elements: Algorithm~\ref{alg:dagcov} evaluates $\textit{deps\_met}(v)$ using $\textit{eff}(u) \in \{\beta(u), 0\}$ at each dependency, while Algorithm~\ref{alg:doe} evaluates $\textit{deps\_met\_DOE}(v)$ using $\textit{eff}_{\text{DOE}}(u)$, which for a non-Literal $u$ with a passing FWR test equals $\delta \cdot \min(f_{\text{sim}}, w_{\text{sim}}, r_{\text{sim}}) \neq \beta(u)$ in general. The two analyses may therefore disagree on dependency satisfaction, and $\textit{eff}_{\text{DOE}}(v)$ and $\textit{eff}(v)$ can differ accordingly; additional divergence arises when $\theta_{\text{prop}} \neq \theta$.

\smallskip\noindent\textit{Non-Literal elements} ($\beta(v) < \theta_{\text{lit}}$): Phase~1 sets $\textit{eff}_{\text{DOE}}(v) = \delta \cdot \min(f_{\text{sim}}, w_{\text{sim}}, r_{\text{sim}})$ when FWR passes and $\neg\textit{vitiated}$, else $0$; this can be greater than, equal to, or less than $\textit{eff}(v) = \beta(v)$ (under $\textit{deps\_met}(v)$) or divergent from $\textit{eff}(v) = 0$ otherwise. Estoppel filtering can force $\textit{eff}_{\text{DOE}}(v) = 0$ even when $\textit{eff}(v) > 0$.

\smallskip\noindent Consequently $W_{\text{DOE}}$ and $W_{\text{cov}}$ are not directly comparable in general (see Definition~\ref{def:doeeff}).
\end{proposition}
\begin{proof}
\textit{(i)} Phase~1 (lines~1--18) classifies each element. For $\textsc{Literal}$: $\beta(v) \geq \theta_{\text{lit}}$ is checked at line~3. For $\textsc{Equivalent}$: all three prongs $\geq \theta_{\text{eq}}$ and $\neg\textit{vitiated}$ are checked at line~11; $s^*$ is selected from $S_v = S \setminus \textit{EstoppedRegion}(v, \textit{PH})$, satisfying the non-estoppel condition by construction. For $\textsc{NoMatch}$: the else branch requires no proof obligations.

\textit{(ii)} Phase~2 (lines~20--28) processes nodes in topological order, zeroing $\textit{eff}_{\text{DOE}}(v)$ whenever any dependency $u$ has $\textit{eff}_{\text{DOE}}(u) < \theta_{\text{prop}}$, analogous to Theorem~\ref{thm:algcorrect}(ii). The argument is independent of which of $\{\theta_{\text{lit}}, \theta_{\text{eq}}\}$ is supplied as $\theta_{\text{prop}}$: only the topological-order invariant and the line~25 zeroing rule are used.

\textit{(iii)} Since $\textit{eff}_{\text{DOE}}(v) \in [0,1]$ for all $v$ (bounded by $\beta(v) \leq 1$ for literals and $\delta \cdot \min(\ldots) \leq 1$ for equivalents) and weights are positive, the argument follows as in Theorem~\ref{thm:algcorrect}(i).

\textit{(iv)} The summary statement follows from the more detailed Proposition~\ref{prop:doe-cov} below, whose proof is given inline with that proposition.
\end{proof}

\begin{proof}[Proof of Proposition~\ref{prop:doe-cov}]
\textit{Literal case.} Phase~1 line~4 sets $\textit{eff}_{\text{DOE}}(v) \leftarrow \beta(v)$ when $\beta(v) \geq \theta_{\text{lit}}$. Phase~2 line~25 sets $\textit{eff}_{\text{DOE}}(v) \leftarrow 0$ iff $\neg\textit{deps\_met\_DOE}(v)$; otherwise the Phase~1 value is retained. So $\textit{eff}_{\text{DOE}}(v) \in \{\beta(v), 0\}$. The standard analysis similarly yields $\textit{eff}(v) \in \{\beta(v), 0\}$ depending on $\textit{deps\_met}(v)$ (Theorem~\ref{thm:algcorrect}(ii)--(iii)). The predicates $\textit{deps\_met}$ and $\textit{deps\_met\_DOE}$ share the same structural form $\forall u \in \textit{deps}(v): (\cdot)(u) \geq (\cdot)$ but use different effective-score functions and thresholds, so they can diverge as stated.

\textit{Non-Literal case.} When $\beta(v) < \theta_{\text{lit}}$, Phase~1 assigns $\textit{eff}_{\text{DOE}}(v) = \delta \cdot \min(f_{\text{sim}}, w_{\text{sim}}, r_{\text{sim}})$ at line~13 if the FWR test passes and $\neg\textit{vitiated}(v, s^*)$; otherwise $\textit{eff}_{\text{DOE}}(v) = 0$ via line~7 (when $S_v = \emptyset$) or line~15 (when FWR fails or $\textit{vitiated}$). The arithmetic $\delta \cdot \min(\ldots)$ is unconstrained relative to $\beta(v)$, so $\textit{eff}_{\text{DOE}}(v)$ can lie above, below, or equal to $\textit{eff}(v) = \beta(v)$ (under $\textit{deps\_met}(v)$).
\end{proof}

\noindent\textbf{Note on Phase~2 threshold interaction.} When $\theta_{\text{prop}} = \theta_{\text{lit}}$ in Phase~2 (line~23), an $\textsc{Equivalent}$ match at node $u$ with $\textit{eff}_{\text{DOE}}(u) = \delta \cdot \min(f_{\text{sim}}, w_{\text{sim}}, r_{\text{sim}})$ may fail the dependency check $\textit{eff}_{\text{DOE}}(u) \geq \theta_{\text{lit}}$ even though $u$ passed the Phase~1 FWR test at threshold $\theta_{\text{eq}} < \theta_{\text{lit}}$. For example, with $\theta_{\text{eq}} = 0.50$, $\theta_{\text{lit}} = 0.65$, and $\delta = 1$, a node with $\min(f_{\text{sim}}, w_{\text{sim}}, r_{\text{sim}}) = 0.55$ passes classification as $\textsc{Equivalent}$ but yields $\textit{eff}_{\text{DOE}}(u) = 0.55 < 0.65 = \theta_{\text{lit}}$, zeroing its dependents. Consequently, $\textsc{Equivalent}$ matches provide propagation benefit only when $\delta \cdot \min(f_{\text{sim}}, w_{\text{sim}}, r_{\text{sim}}) \geq \theta_{\text{lit}}$. The alternative choice $\theta_{\text{prop}} = \theta_{\text{eq}}$ allows all classified $\textsc{Equivalent}$ matches to propagate. This is a defensible configuration that maximizes DOE coverage while remaining mathematically sound, at the cost of weaker dependency enforcement for equivalent-only paths. Practitioners should select $\theta_{\text{prop}}$ based on whether conservative dependency enforcement ($\theta_{\text{prop}} = \theta_{\text{lit}}$) or maximized DOE coverage ($\theta_{\text{prop}} = \theta_{\text{eq}}$) better serves their analysis goals.

\noindent\textbf{Note on \textit{match}$(v)$ vs.\ $\textit{eff}_{\text{DOE}}(v)$.} As established in Proposition~\ref{prop:doe-cov} (Literal case), $\textit{match}(v) = \textsc{Literal}$ and $\textit{eff}_{\text{DOE}}(v) = 0$ can coexist when $\neg\textit{deps\_met\_DOE}(v)$. The classification records the Phase~1 element-level result; the effective score records the Phase~2 dependency-propagated result. Practitioners should use $\textit{eff}_{\text{DOE}}(v)$ for coverage accounting and $\textit{match}(v)$ for element-level classification reporting; the two should not be conflated.

\noindent\textbf{Complexity.} $O(n \cdot m \cdot d)$ for initial score computation ($\beta(v)$ for all $v \in V$), plus $O(|V_{\text{sub}}| \cdot m \cdot d)$ for FWR similarity on the DOE-eligible subset $V_{\text{sub}} = \{v \in V \mid \beta(v) < \theta_{\text{lit}}\}$, plus $O(n \cdot |\textit{PH}|)$ for estoppel filtering, plus $O(n + |E|)$ for DAG propagation of $\textit{eff}_{\text{DOE}}$. The dominant term is $O(n \cdot m \cdot d)$.

The formal encoding in Lean~4 uses an inductive type ordered $\textsc{NoMatch} < \textsc{Equivalent} < \textsc{Literal}$:

\smallskip
\noindent\colorbox{orange!20}{\small\textsc{Lean 4: illustrative, not compiled}}\nopagebreak

\begin{lstlisting}
-- Pseudocode illustration of the MatchType encoding.
-- In the full formalization, score, theta_lit, theta_eq,
-- f_sim, w_sim, r_sim, Sv, v, and s_star are parameters
-- of the inductive type; shown here as free variables
-- for readability.
inductive MatchType where
  | literal   : (h  : score >= theta_lit) -> MatchType
  | equivalent : (hl : score < theta_lit) ->
                  (hf : f_sim >= theta_eq) ->
                  (hw : w_sim >= theta_eq) ->
                  (hr : r_sim >= theta_eq) ->
                  (hne : Sv.Nonempty) ->
                  (hnv : Not (vitiated v s_star)) -> MatchType
  | no_match  : MatchType
\end{lstlisting}
\noindent The \texttt{equivalent} constructor carries \texttt{hl~:~score~<~theta\_lit} to enforce the precedence rule (Definition~\ref{def:matchtype}): literal classification takes priority when $\beta(v) \geq \theta_{\text{lit}}$. The field \texttt{hne~:~Sv.Nonempty} captures the proof obligation that non-estopped evidence exists for $v$ (i.e., $S_v = S \setminus \textit{EstoppedRegion}(v, \textit{PH}) \neq \emptyset$), which is the precondition for $s^*(v)$ to be well-defined (Definition~\ref{def:bestevidence}). The non-estoppel condition on $s^*(v)$ itself is automatically satisfied by the selection from $S_v$.

\noindent\textbf{Scope of formal verification for DOE.} The proof certificate $\Pi$ for Problem~5 provides guarantees that are conditional in a stronger sense than for Problems~1--4. Three layers of conditionality apply:
\begin{itemize}[leftmargin=2em, topsep=2pt, itemsep=1pt]
\item \textit{Score conditionality:} As with Problems~1--4, $f_{\text{sim}}, w_{\text{sim}}, r_{\text{sim}}$ are ML-computed and lie below the trust boundary.
\item \textit{Estoppel conditionality:} $\textit{ER}_{\text{arg}}$ and the Festo rebuttal predicates require natural language interpretation of prosecution history text. These are treated as inputs whose correctness is assumed, not verified.
\item \textit{Legal concept conditionality:} The formal layer verifies that $f_{\text{sim}} \geq \theta_{\text{eq}}$, but it cannot verify that $f_{\text{sim}}$ correctly measures ``substantially the same function'' in the legal sense. This is a semantic legal judgment that resists machine verification.
\end{itemize}
The DOE proof certificate therefore provides a weaker guarantee than those for Problems~1--4: mathematical correctness of classification given inputs, not legal correctness of the classification itself.

\noindent\textbf{Case study: doctrine of equivalents for a signal filter claim.} We reuse the two-element filter claim from the FTO miniature example of Section~\ref{sec:fto}, now with DOE-specific thresholds ($\theta_{\text{lit}} = 0.70$, $\theta_{\text{eq}} = 0.45$) per the recommended $\theta = \theta_{\text{lit}}$ configuration; the claim structure is identical, only the threshold choice differs because DOE is enabled here. The two limitations are $E_1$ (functional, $w = 1.5$) ``a filter element configured to attenuate high-frequency noise'' and $E_2$ (wherein, $w = 3.0$) ``wherein the filter element operates by subtracting a delayed copy of the input signal from the original signal.'' The wherein clause $E_2$ depends on $E_1$, so the DAG is $E_1 \to E_2$ with total weight $1.5 + 3.0 = 4.5$. We also set $\theta_{\text{vit}} = 0.10$, $\delta = 1.0$, and $\theta_{\text{prop}} = \theta_{\text{lit}} = 0.70$ per the default configuration of Definition~\ref{def:doeeff}. For the standard-analysis comparison below, we use the matching $\theta = 0.70$ in Algorithm~\ref{alg:dagcov}. The accused product is a moving-average filter that attenuates high frequencies by averaging recent samples rather than by the subtraction mechanism of $E_2$.

\textit{Phase~1 (classification).} The ML layer computes $\beta(E_1) = 0.85$ and $\beta(E_2) = 0.41$. Since $\beta(E_1) \geq \theta_{\text{lit}}$, $\textit{match}(E_1) = \textsc{Literal}$ and $\textit{eff}_{\text{DOE}}(E_1) = 0.85$. Since $\beta(E_2) < \theta_{\text{lit}}$, $E_2$ enters DOE analysis. Suppose the prosecution history records one amendment to $E_2$: original language ``reduces unwanted signal components'' was narrowed to the current ``subtracts a delayed copy $\ldots$'' language for patentability reasons, placing moving-average and integration-based implementations in the surrendered territory $\Phi_{E_2}^{\text{orig}} \setminus \Phi_{E_2}^{\text{amend}}$. The three Festo rebuttal predicates (unforeseeability, tangentiality, other reason) all evaluate to false, so $\textit{Festo\_rebutted} = \textit{false}$ and the product's moving-average segment is estopped. Assume one non-estopped segment remains: documentation of an unrelated notch-filter stage. The best non-estopped evidence is therefore $s^*(E_2) = $ the notch-filter segment, with $M(E_2, s^*) \approx 0.37 \geq \theta_{\text{vit}} = 0.10$, so $\neg\textit{vitiated}(E_2, s^*)$. The function-way-result test against $s^*$ produces $f_{\text{sim}} = 0.71$, $w_{\text{sim}} = 0.38$, $r_{\text{sim}} = 0.68$. Since $w_{\text{sim}} < \theta_{\text{eq}}$, the FWR test fails, $\textit{match}(E_2) = \textsc{NoMatch}$, and $\textit{eff}_{\text{DOE}}(E_2) = 0$ by Definition~\ref{def:doeeff}.

\textit{Phase~2 (propagation).} In topological order: $E_1$ has no dependencies, so $\textit{deps\_met\_DOE}(E_1) = \textit{true}$ and $\textit{eff}_{\text{DOE}}(E_1) = 0.85$ is retained. For $E_2$, $\textit{deps\_met\_DOE}(E_2)$ requires $\textit{eff}_{\text{DOE}}(E_1) \geq \theta_{\text{prop}} = 0.70$; since $0.85 \geq 0.70$, this holds, but the Phase~1 classification already set $\textit{eff}_{\text{DOE}}(E_2) = 0$. Therefore $W_{\text{DOE}} = (1.5 \cdot 0.85 + 3.0 \cdot 0) / 4.5 \times 100 \approx 28.3\%$.

\textit{Comparison with literal-only coverage.} Under the standard analysis of Definition~\ref{def:bestmatch}, $\textit{eff}(v) = \beta(v)$ whenever $\textit{deps\_met}(v)$ holds; there is no element-level zeroing by the literal threshold. Here $\textit{eff}(E_1) = 0.85$ and $\textit{deps\_met}(E_2)$ holds because $\textit{eff}(E_1) = 0.85 \geq \theta = 0.70$, so $\textit{eff}(E_2) = \beta(E_2) = 0.41$ and $W_{\text{cov}} = (1.5 \cdot 0.85 + 3.0 \cdot 0.41) / 4.5 \times 100 \approx 55.7\%$. This example therefore concretizes Theorem~\ref{thm:doe}(iv): DOE analysis \textit{decreased} coverage from $55.7\%$ to $28.3\%$ because a non-literal element that the standard analysis credits at $\beta(E_2) = 0.41$ is reclassified as $\textsc{NoMatch}$ under DOE once estoppel removes the moving-average evidence and the remaining non-estopped evidence fails the way prong.

\textit{Contrast: when DOE raises coverage, and the Phase~2 threshold subtlety.} Suppose instead the way prong passed with $w_{\text{sim}} = 0.60$ while $f_{\text{sim}}$ and $r_{\text{sim}}$ remain as above. Then $\textit{match}(E_2) = \textsc{Equivalent}$ and Phase~1 assigns $\textit{eff}_{\text{DOE}}(E_2) = \delta \cdot \min(0.71, 0.60, 0.68) = 0.60$. In Phase~2, $E_2$ has no dependents, so no cascade applies and $W_{\text{DOE}} = (1.5 \cdot 0.85 + 3.0 \cdot 0.60) / 4.5 \times 100 \approx 68.3\%$, above the $55.7\%$ literal-only figure. However, if $E_2$ had had a dependent $E_3$, the Phase~2 check at $E_3$ would have evaluated $\textit{eff}_{\text{DOE}}(E_2) = 0.60 < 0.70 = \theta_{\text{prop}}$ and zeroed $E_3$, concretizing the Note on Phase~2 threshold interaction above: an $\textsc{Equivalent}$ classification does not automatically support downstream propagation when $\theta_{\text{prop}} = \theta_{\text{lit}}$.

\textit{What the certificate attests.} The proof certificate $\Pi$ generated by \texttt{ProveDOE} certifies, in the failure case: $\beta(E_2) = 0.41 < \theta_{\text{lit}}$; the estoppel region calculation given the scope classifications of $\Phi_{E_2}^{\text{orig}}, \Phi_{E_2}^{\text{amend}}$; the FWR threshold comparison $w_{\text{sim}} = 0.38 < \theta_{\text{eq}}$; the $\textsc{NoMatch}$ classification discharging the Definition~\ref{def:matchtype} proof obligations; and the bound $0 \leq W_{\text{DOE}} = 28.3\% \leq 100$. What $\Pi$ does \textit{not} certify, per the three-layer conditionality above, is whether $f_{\text{sim}}, w_{\text{sim}}, r_{\text{sim}}$ correctly measure ``substantially the same function, way, and result,'' whether the moving-average segment in fact belongs to the surrendered territory, or whether the Festo rebuttal predicates were correctly evaluated. Those judgments remain below the trust boundary.

\subsection{Algorithm 6: Kleene Fixed-Point Propagation}

Algorithm~\ref{alg:kleene} provides an alternative to the topological-sort approach of Algorithm~\ref{alg:dagcov} by computing propagated effective scores via Kleene fixed-point iteration on the match strength lattice. Rather than processing nodes in a precomputed topological order, it initializes all effective scores to $\bot = 0$ and iteratively applies the propagation function~$F$ (Definition~\ref{def:propfunc}) until convergence.

The convergence guarantee combines two distinct mathematical facts. First, on the idealized lattice $(V \to [0,1])$ with pointwise order, the \textit{Knaster-Tarski fixed-point theorem}~\cite{cousot1979systematic, dang2022faster} guarantees that $F$ has a least fixed point $\textit{eff}^*$ satisfying $F(\textit{eff}^*) = \textit{eff}^*$, because $F$ is monotone (Lemma~\ref{lem:Fmono}). Second, the implementation uses the finite 10001-point discretization of \texttt{MatchStrength} described in Appendix~\ref{app:lean}, on which \textit{Kleene iteration} from $\bot$ converges to the least fixed point in at most $h+1$ steps, where $h$ is the DAG depth. The $h+1$ bound rests on the finiteness of the discretized lattice (so ascending chains terminate), the monotonicity of $F$, and the acyclicity of $G$ (so each node stabilizes at an iteration bounded by its depth). Scott-continuity of $F$ on the idealized $[0,1]$ lattice is neither claimed nor needed: $F$'s threshold comparison at $\theta$ is discontinuous, but the finite discretization makes Kleene iteration converge regardless. The fixed point computes the same threshold-based effective scores as Definition~\ref{def:bestmatch}.

Termination occurs within $h+1$ iterations, where $h$ is the DAG depth (height). The complexity is $O((h{+}1)(n + |E|))$, which is $O(n^2 + n|E|)$ in the worst case, strictly less efficient than Algorithm~\ref{alg:dagcov}'s unconditional $O(n + |E|)$. Algorithm~\ref{alg:kleene} is provided as a theoretically motivated alternative whose fixed-point framing (Knaster-Tarski existence plus finite-lattice Kleene convergence) offers different formal guarantees than the direct topological induction used for Algorithm~\ref{alg:dagcov}. Theorem~\ref{thm:convergence} establishes convergence and the relationship between the fixed point of~$F$ and the meet-based $\textit{claimStrength}$ model (Definition~\ref{def:claimstrength}).

\begin{algorithm}[H]
\caption{KleeneFixedPointPropagation \hfill {\small\normalfont\itshape Convergence: Theorem~\ref{thm:convergence} (informal sketch)}}
\label{alg:kleene}
\begin{algorithmic}[1]
\Require Claim DAG $G$, raw scores $\beta$, threshold $\theta$
\Ensure Effective scores $\textit{eff}$ (the least fixed point of $F$), proof certificate $\Pi$
\State $\textit{eff}_0(v) \gets \bot$ for all $v \in V$
\State $k \gets 0$
\Repeat
  \State $k \gets k + 1$
  \For{each $v \in V$}
    \If{$\forall u \in \textit{deps}(v): \textit{eff}_{k-1}(u) \geq \theta$}
      \State $\textit{eff}_k(v) \gets \beta(v)$
    \Else
      \State $\textit{eff}_k(v) \gets \bot$
    \EndIf
  \EndFor
\Until{$\textit{eff}_k = \textit{eff}_{k-1}$}
\State $\Pi \gets \text{ProveFixedPoint}(F, \textit{eff}_k)$
\State \Return $(\textit{eff}_k, \Pi)$
\end{algorithmic}
\end{algorithm}

\begin{definition}[Propagation Function]
\label{def:propfunc}
The propagation function $F : (V \to \mathbb{L}) \to (V \to \mathbb{L})$ is defined by:
\[
F(\textit{eff})(v) = \begin{cases} \beta(v) & \text{if } \forall u \in \textit{deps}(v):\; \textit{eff}(u) \geq \theta \\ \bot & \text{otherwise} \end{cases}
\]
\end{definition}

\begin{lemma}[Monotonicity of $F$]
\label{lem:Fmono}
$F$ is monotone on the complete lattice $(V \to \mathbb{L})$ with pointwise order: if $\textit{eff}_1 \leq \textit{eff}_2$ pointwise, then $F(\textit{eff}_1) \leq F(\textit{eff}_2)$ pointwise.
\end{lemma}
\begin{proof}
Suppose $\textit{eff}_1(u) \leq \textit{eff}_2(u)$ for all $u$, and consider $F(\textit{eff}_1)(v)$ vs.\ $F(\textit{eff}_2)(v)$. If $\textit{deps\_met}$ fails under $\textit{eff}_1$, then $F(\textit{eff}_1)(v) = \bot \leq F(\textit{eff}_2)(v)$. If $\textit{deps\_met}$ holds under $\textit{eff}_1$, then $\textit{eff}_1(u) \geq \theta$ for all $u \in \textit{deps}(v)$, so $\textit{eff}_2(u) \geq \textit{eff}_1(u) \geq \theta$, meaning $\textit{deps\_met}$ also holds under $\textit{eff}_2$, and $F(\textit{eff}_1)(v) = \beta(v) = F(\textit{eff}_2)(v)$.
\end{proof}

\begin{theorem}[Convergence of Algorithm~\ref{alg:kleene}]
\label{thm:convergence}
Assume $G$ is acyclic (\texttt{dag\_acyclic}). Let $F$ be the propagation function (Definition~\ref{def:propfunc}). Algorithm~\ref{alg:kleene}, starting from $\textit{eff}_0(v) = \bot$ for all $v \in V$, terminates in at most $h+1$ iterations where $h$ is the depth (height) of $G$ (length of the longest path). The result is the least fixed point of $F$, the smallest function $\textit{eff}^* : V \to \mathbb{L}$ satisfying $F(\textit{eff}^*) = \textit{eff}^*$.
\end{theorem}
\begin{proof}
\textit{Step~1: Existence and computability.} By Lemma~\ref{lem:Fmono}, $F$ is monotone. On the idealized complete lattice $(V \to [0,1])$, the Knaster-Tarski theorem yields a least fixed point $\textit{eff}^*$ satisfying $F(\textit{eff}^*) = \textit{eff}^*$; monotonicity suffices for existence. On the discretized lattice used in the implementation ($V \to \texttt{MatchStrength}$, finite by Appendix~\ref{app:lean}), Kleene iteration from $\bot$ converges to the least fixed point in finitely many steps: the Kleene chain $\bot \leq F(\bot) \leq F^2(\bot) \leq \cdots$ is non-decreasing by monotonicity and must stabilize because the lattice is finite; once $F(\textit{eff}) = \textit{eff}$, induction on the chain gives $F^n(\bot) \leq y$ for every fixed point $y$, so the stabilized value is the least fixed point on the discretized lattice. Scott-continuity of $F$ on $[0,1]$ is neither claimed nor needed: $F$'s threshold comparison at $\theta$ is discontinuous, and finiteness rather than Scott-continuity is what makes Kleene iteration computationally tractable.

\textit{Step~2: Termination:} Since $F(\textit{eff})(v) \in \{0, \beta(v)\}$ for all $v$, each node can change value from $0$ to $\beta(v)$ at most once across iterations. After iteration $k$, all nodes at DAG depth $\leq k$ whose dependencies are met have stabilized. Since the maximum depth is $h$ (finite by \texttt{dag\_acyclic}), stabilization occurs by iteration $h+1$. The repeat-until condition $\textit{eff}_k = \textit{eff}_{k-1}$ (line~12) detects this.

\textit{Step~3: Relationship to Definition~\ref{def:effcov} and claimStrength:} The fixed point of $F$ computes exactly the threshold-based effective score from Definition~\ref{def:effcov}: a node receives $\beta(v)$ iff all its transitive dependencies meet the threshold, and $0$ otherwise. This follows directly from Step~2: at convergence, leaf nodes have $F(\textit{eff})(v) = \beta(v)$ (vacuously), and each subsequent node receives $\beta(v)$ iff all predecessors have stabilized above $\theta$, matching the structural induction in Definition~\ref{def:effcov}. This is related to but distinct from $\textit{claimStrength}$ (Definition~\ref{def:claimstrength}), which computes the continuous meet along dependency paths. The two models are not equivalent in general. The threshold model zeros out any node whose transitive dependency falls below $\theta$, regardless of that dependency's actual score; the meet model instead propagates the minimum score along dependency paths, which is zero only when a dependency has score~$0$. Concretely: for a node $v$ whose dependency $u$ has $\textit{score}(u) = s_u$ with $0 < s_u < \theta$, the threshold model sets $\textit{eff}(v) = 0$, while the meet model sets $\textit{claimStrength}(\textit{score}, v) = \textit{score}(v) \wedge s_u \wedge \bigwedge_{u' \in \textit{anc}(v) \setminus \{u\}} \textit{score}(u')$, which is strictly positive whenever $\textit{score}(v) > 0$ and every ancestor score is positive. For unsatisfied-threshold dependencies with non-zero scores, the two models therefore produce different values for the same node: zero under the threshold model versus a strictly positive meet under $\textit{claimStrength}$. The two models coincide only in the special cases where $\textit{score}(u) = 0$ for some dependency $u$, or where every dependency satisfies $\textit{score}(u) \geq \theta$ and the meet on the right-hand side already equals $\textit{score}(v)$. Even in the latter case, the fixed point of $F$ returns $\beta(v)$ while $\textit{claimStrength}$ returns the possibly strictly smaller meet when any ancestor has a score below $\textit{score}(v)$.
\end{proof}

\noindent\textbf{Complexity.} $O((h+1) \cdot (n + |E|))$, where $h$ is the DAG depth. For sparse patent claim DAGs with bounded depth (the typical case, as in the 15-node memory-module case study with $h = 6$), this is effectively $O(n + |E|)$, only slightly worse than Algorithm~\ref{alg:dagcov}'s unconditional $O(n + |E|)$. In the worst case (e.g., a chain DAG with $h = n - 1$) the bound degrades to $O(n^2 + n|E|)$, which is where Algorithm~\ref{alg:dagcov}'s unconditional $O(n + |E|)$ strictly dominates. Algorithm~\ref{alg:kleene} is provided as a theoretically motivated alternative with different convergence guarantees (Theorem~\ref{thm:convergence}), and the typical-case gap between the two is small.

\section{IP Use Cases and Applications}
\label{sec:usecases}

\begin{figure}[H]
\centering
\includegraphics[width=\textwidth]{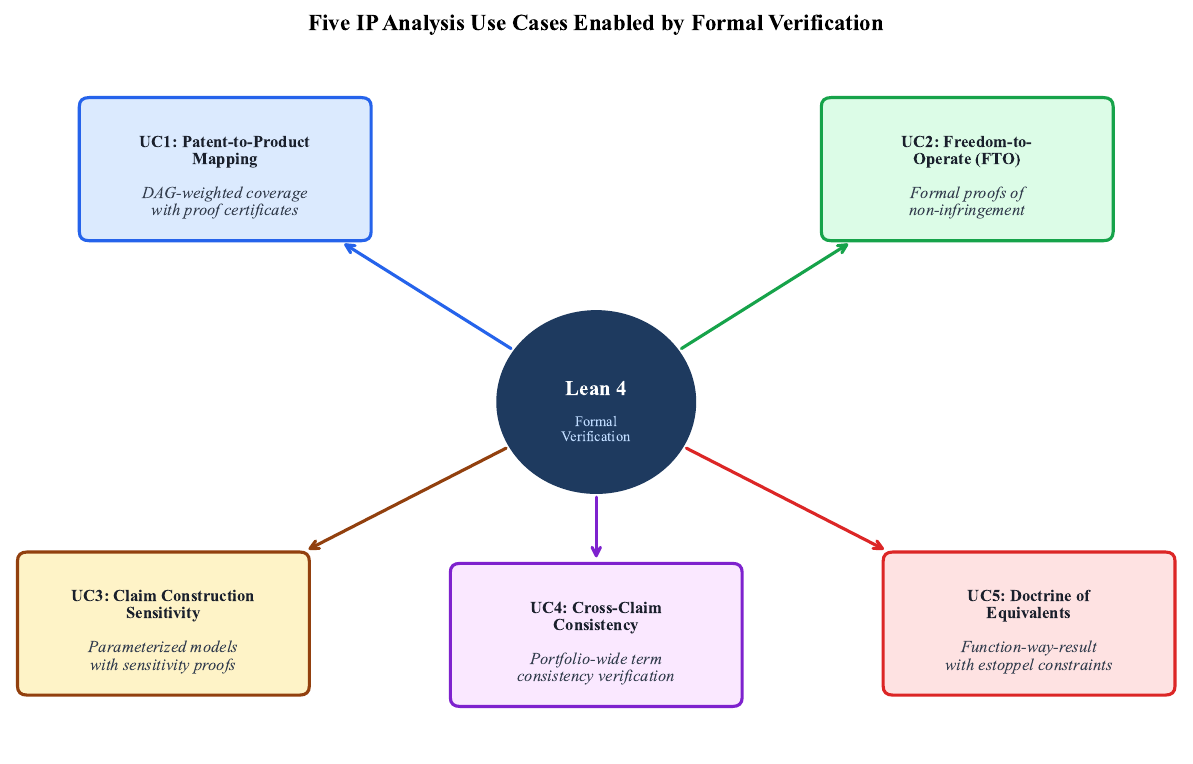}
\caption{Five IP analysis use cases enabled by formal verification. For UC1 (Algorithm~\ref{alg:dagcov-given}), a compiled Lean~4 generator produces certificates accepted under Definition~\ref{def:validcert} (Proposition~\ref{prop:gensound}). For UC2--UC5, the prototype generates candidate certificates validated by independent kernel type-checking and a \texttt{sorry}-free axiom audit.}
\label{fig:usecases}
\end{figure}

\noindent This section (see also Figure~\ref{fig:usecases}) serves as a practitioner-facing index of how the six formal algorithms in Section~\ref{sec:algorithms} address the five IP analysis problems formulated in Section~\ref{sec:formulation}; full technical development, correctness arguments, and case-study numerics are in Section~\ref{sec:algorithms} and Section~\ref{sec:implementation}. The use-case bullets below summarize the scope, output, and conditionality of each algorithm in terms meaningful for legal and technical stakeholders, without re-deriving the mechanisms. Trust-boundary placement of the per-algorithm inputs and outputs is consolidated in the Guarantee Map of Section~\ref{sec:architecture}.

\textbf{UC1: Patent-to-Product Mapping.} Algorithm~\ref{alg:dagcov} maps claim limitations to product evidence using DAG-weighted coverage with legally-motivated weights (Table~\ref{tab:weights}). Output: a coverage score $W_{\text{cov}} \in [0,100]$ plus a proof certificate attesting to bounds, dependency enforcement, and coverage-definition equality. For the formal phase (Algorithm~\ref{alg:dagcov-given}, once $\beta$ is fixed), the certificate is generated by a compiled, machine-verified generator (Proposition~\ref{prop:gensound}); the ML computation of $\beta$ remains below the trust boundary.

\begin{table}[t]
\centering
\caption{Weight assignment scheme based on patent law doctrine.}
\label{tab:weights}
\begin{tabular}{@{}llp{7cm}@{}}
\toprule
\textbf{Type} & \textbf{Weight} & \textbf{Legal Rationale} \\
\midrule
Wherein      & 3.0 & Core inventive concept; highest specificity \\
Quantitative & 2.0 & Numeric constraints; hard to design around \\
Functional   & 1.5 & Operational requirements; 35~U.S.C.~\S112(f) \\
Coupling     & 1.5 & Connectivity and interface requirements \\
Signal       & 1.5 & Interface signal specifications \\
Structural   & 1.0 & Physical components; baseline \\
Preamble     & 0.3 & Introductory clause; rarely limiting \\
\bottomrule
\end{tabular}
\end{table}

\noindent\textit{Doctrinal rationale.} The weight assignments reflect the relative legal significance of element types in claim construction and infringement analysis. \textit{Wherein} clauses receive the highest weight ($w = 3.0$) because they typically define the functional heart of independent claims and are frequently outcome-determinative in litigation, courts have consistently held that wherein clauses narrow claim scope to the specific inventive contribution. \textit{Preamble} elements receive the lowest weight ($w = 0.3$) because preamble language rarely limits claim scope under established precedent (\textit{Pitney Bowes, Inc.\ v.\ Hewlett-Packard Co.}~\cite{pitneybowes1999}), serving primarily to set the technical field. \textit{Quantitative} constraints ($w = 2.0$) carry high weight because numeric limits are precise, hard to design around, and frequently dispositive. \textit{Functional} ($w = 1.5$) elements implicate 35~U.S.C.~\S112(f), which can substantially affect scope. \textit{Structural} elements ($w = 1.0$) serve as the baseline. These ratios are not empirically validated against litigation outcomes; see Section~\ref{sec:discussion} for limitations and planned validation.

\textbf{UC2: Freedom-to-Operate.} Algorithm~\ref{alg:fto} provides the most legally consequential application. A formal proof of non-infringement provides evidence qualitatively different from ``we didn't find a match.'' \textit{Asymmetry warning:} $\textsc{Clear}$ is a formal (conditional) non-infringement certificate; $\textsc{Risk}$ is \textit{not} an infringement finding but the absence of a per-element impossibility proof, and requires human legal judgment (Section~\ref{sec:fto}, Section~\ref{sec:discussion}).

\textbf{UC3: Claim Construction Sensitivity.} Algorithm~\ref{alg:construction} identifies determinative terms and the threshold construction (the most restrictive interpretation under which the claim is still satisfied). Output: a sensitivity report $R$ with coverage per construction, determinative-term set, and $I_{\text{threshold}}$; detailed correctness statement in Theorem~\ref{thm:construction}.

\textbf{UC4: Cross-Claim Consistency.} Algorithm~\ref{alg:consistency} verifies term interpretation consistency across a portfolio, preventing legal vulnerabilities from contradictory constructions. Output: $\textsc{Consistent}$ with a certificate, or a concrete $\textsc{Inconsistent}(i, j, t, f(i,t), f(j,t))$ counterexample.

\textbf{UC5: Doctrine of Equivalents.} Algorithm~\ref{alg:doe} extends matching beyond literal infringement with formally justified function-way-result analysis and estoppel constraints. \textit{Stronger conditionality:} the UC5 proof certificate carries materially weaker guarantees than UC1--UC4. In addition to the ML-score conditionality shared with all other use cases, UC5 certificates are conditional on (a)~the correctness of ML-computed function, way, and result similarities; (b)~the natural-language interpretation of prosecution history (the estoppel region $\textit{ER}_{\text{arg}}$, the scope classifications of $\Phi_v^{\text{orig}}$ and $\Phi_v^{\text{amend}}$, and the Festo rebuttal predicates); and (c)~the semantic judgment that these ML/NLP-computed quantities correctly capture the underlying legal concepts. See \S\ref{sec:algorithms} (``Scope of formal verification for DOE'') for the full three-layer conditionality discussion.

\section{Theoretical Analysis}
\label{sec:theory}

In this section we analyze the correctness guarantees, computational complexity, and fundamental differences between our hybrid formal/ML approach and purely probabilistic methods. The framework's correctness story is \emph{two-layered}.

\smallskip\noindent\textit{Closed coverage-core layer (Algorithm~\ref{alg:dagcov-given}).} A closed Lean~4 development (Section~\ref{sec:closedpath}, Appendices~\ref{app:lean}--\ref{app:axiomaudit}) supplies a set of \textit{machine-verified} results: the structural lemmas \texttt{dag\_acyclic} and the Mathlib \texttt{CompleteLattice} instance; the coverage bound \texttt{coverage\_in\_range} (both endpoints of $[0,100]$, for any \texttt{ScoreValid} score function); the propagation-correctness theorem \texttt{propag\_proof} for \texttt{computeEff}; the coverage-definition equality $\texttt{weightedCoverage}(\beta,\theta) = W_{\text{cov}}$ (machine-verified by definitional unfolding, because \texttt{weightedCoverage} is defined as the weighted average of \texttt{computeEff}); and the generator-soundness proposition \texttt{generateCertificate} / Proposition~\ref{prop:gensound}, audited against $\Omega = \{\texttt{propext}, \texttt{Classical.choice}, \texttt{Quot.sound}\}$.

\smallskip\noindent\textit{Higher-level algorithmic layer (Algorithms~\ref{alg:dagcov}, \ref{alg:fto}--\ref{alg:kleene}; monotonicity and weakest-link).} The coverage-bound part~(i) of Theorem~\ref{thm:algcorrect} is discharged for Algorithm~\ref{alg:dagcov-given} by \texttt{coverage\_in\_range} from the closed layer; the remaining higher-level guarantees (Theorems~\ref{thm:weakest} (weakest link), \ref{thm:monotone} (propagation monotonicity), and parts~(ii)--(iv) of Theorems~\ref{thm:algcorrect}--\ref{thm:convergence}, covering dependency enforcement, completeness, monotonicity, FTO soundness, construction-sensitivity correctness, cross-claim consistency, DOE correctness, and Kleene convergence) are presented as \textit{informal proof sketches} describing the intended Lean~4 proof structure (Table~\ref{tab:proofstatus}). The proof-generation functions for Algorithms~\ref{alg:fto}--\ref{alg:kleene} are \textit{architecturally mitigated} by independent kernel type-checking rather than by a formal specification of the generator (Section~\ref{sec:contributions}).

\smallskip\noindent We show that the framework provides six categories of correctness properties, plus a characterization of conditionality for the DOE analysis (Section~\ref{sec:correctness}), that all algorithms run in polynomial time for typical patent analysis workloads (Section~\ref{sec:complexityanalysis}), and that the hybrid approach achieves qualitatively different guarantees than either ML or formal verification alone (Section~\ref{sec:probcomparison}). Throughout, the verification status of each specific result is the one given in Table~\ref{tab:proofstatus}, not a uniform ``machine-verified'' claim.

\subsection{Correctness Guarantees}
\label{sec:correctness}

Our framework identifies six categories of correctness properties, plus a characterization of DOE conditionality. These are summarized in Table~\ref{tab:guarantees}. \textbf{The verification status of each property varies per Table~\ref{tab:proofstatus}}: some are machine-verified in Lean~4 (e.g., \texttt{dag\_acyclic}, the \texttt{CompleteLattice} instance, \texttt{coverage\_in\_range}, the coverage-definition equality $\texttt{weightedCoverage}(\beta,\theta) = W_{\text{cov}}$ by definitional unfolding, and Proposition~\ref{prop:gensound}), some are architecturally enforced by type signatures, and the higher-level algorithmic guarantees covering dependency enforcement, completeness, monotonicity, weakest-link bounding, and algorithm correctness for Algorithms~\ref{alg:fto}--\ref{alg:kleene} (Theorems~\ref{thm:weakest}, \ref{thm:monotone}, and parts~(ii)--(iv) of Theorems~\ref{thm:algcorrect}--\ref{thm:convergence}) are currently informal proof sketches describing the intended Lean~4 proof structure, not compiled proofs. The first six categories below are \textit{mathematical properties of the system itself}, not statistical observations about its outputs on test data; the seventh characterizes the additional conditionality of the DOE analysis.

\textbf{Structural soundness} ensures that the claim DAG is well-formed: no circular dependencies exist, all node types are valid, and all referenced dependency IDs correspond to existing nodes. This is verified by Lean's \texttt{decide} tactic, which exhaustively checks the finite graph structure at compile time.

\textbf{Lattice axioms} guarantee that match strengths form a complete lattice with verified reflexivity, transitivity, antisymmetry, and the existence of meets and joins for arbitrary subsets. This ensures that all score composition operations are mathematically well-defined, a property that bare floating-point arithmetic does not provide.

\textbf{Propagation correctness} is argued (informal proof sketch, Theorem~\ref{thm:algcorrect}(ii)) to hold: effective scores respect dependency constraints, i.e., if any dependency $u \in \textit{deps}(v)$ has $\textit{eff}(u) < \theta$ then $\textit{eff}(v) = 0$, enforcing transitive dependency satisfaction. For Algorithm~\ref{alg:dagcov}, the intended proof proceeds directly via topological sort; for Algorithm~\ref{alg:kleene}, it invokes the Knaster-Tarski fixed-point theorem for monotone functions on complete lattices (Theorem~\ref{thm:convergence})~\cite{cousot1979systematic, dang2022faster}. Both theorems are currently informal sketches per Table~\ref{tab:proofstatus}.

\textbf{Coverage bounds} is the combined name for two complementary results, \emph{both machine-verified for the closed Algorithm~\ref{alg:dagcov-given} path} (Table~\ref{tab:proofstatus}; see Section~\ref{sec:closedpath} and Appendix~\ref{app:lean}): (a) \textit{arithmetic bounds} (\texttt{coverage\_in\_range}) is a compiled theorem that $0 \leq \texttt{weightedCoverage}(\beta,\theta) \leq 100$ for \textit{any} bounded score function (\texttt{ScoreValid}) and any $\theta \geq 0$; (b) \textit{coverage-definition equality} is the identity $\texttt{weightedCoverage}(\beta,\theta) = W_{\text{cov}}$ of Definition~\ref{def:depscov}, machine-verified by definitional unfolding because \texttt{weightedCoverage} is defined directly as the weighted average of the propagated effective scores $\texttt{computeEff}(\beta,\theta)$. Together they establish that the certificate's \texttt{coverage} field is both in range and numerically equal to Definition~\ref{def:depscov}'s $W_{\text{cov}}$, with no caller-side propagation precondition.

\textbf{Monotonicity} is stated (informal proof sketch, Theorem~\ref{thm:algcorrect}(iv)) to hold: improving any individual match score $\beta(v)$ can never decrease the overall coverage $W_{\text{cov}}$. The intended theorem says that additional evidence always helps (or at worst has no effect), a property relevant to iterative analysis workflows where new product documentation may be discovered during litigation.

\textbf{Weakest-link bounding} is argued (informal proof sketch, Theorem~\ref{thm:weakest}) to hold: the effective claim strength along any dependency path is bounded above by the weakest node on that path. This provides the theoretical foundation for the actionable insight ``your claim chart is only as strong as C12 at 45\%.''

\textbf{DOE conditionality} (distinct from the six categories above) characterizes the additional assumptions under which the DOE analysis (Problem~5, Algorithm~\ref{alg:doe}) operates: the proof certificate for Problem~5 is conditional on ML-computed FWR scores, estoppel inputs, and the legal concept mapping, in addition to the ML scores on which all other algorithms are conditional. This is not a formal property but a trust boundary characterization; its method in Table~\ref{tab:guarantees} is accordingly listed as ``Trust boundary analysis'' rather than a Lean~4 proof.

\begin{table}[H]
\centering
\caption{Formal correctness guarantees and verification methods.}
\label{tab:guarantees}
\small
\begin{tabular}{@{}p{2.5cm}p{3.6cm}p{4.2cm}p{1.5cm}@{}}
\toprule
\textbf{Property} & \textbf{Guarantee} & \textbf{Method} & \textbf{Status$^*$} \\
\midrule
Structural soundness & DAG is acyclic; all types valid & Lean \texttt{decide} & MV \\
Lattice axioms & Match strengths form complete lattice & Mathlib instance & MV \\
Propagation correctness & Effective scores respect dependencies & Topological sort (Alg.~\ref{alg:dagcov}); Knaster-Tarski (Alg.~\ref{alg:kleene}) & IS \\
Arithmetic bounds & $0 \leq \texttt{weightedCoverage}(\beta,\theta) \leq 100$ for any bounded $\beta$ and $\theta \geq 0$ & \texttt{coverage\_in\_range} (Appendix~\ref{app:lean}) & MV \\
Coverage-semantic equality & $\texttt{weightedCoverage}(\beta,\theta) = W_{\text{cov}}$ of Def.~\ref{def:depscov} & Definitional unfolding: \texttt{weightedCoverage} $=$ weighted average of \texttt{computeEff}; certificate's \texttt{p\_deps} discharged by \texttt{rfl} & MV \\
Monotonicity & Improving evidence cannot hurt & Lattice monotonicity argument & IS \\
Weakest-link & Strength bounded by weakest path & Meet/infimum properties & IS \\
Generator soundness (Alg.~\ref{alg:dagcov-given}) & Certificate matches propagation output, satisfies coverage bounds, accepted under Def.~\ref{def:validcert} & Compiled generator + $\Omega$-audit (\S\ref{sec:closedpath}) & MV \\
\midrule
\textit{DOE conditionality} & Problem~5 conditional on ML scores, estoppel inputs, legal concept mapping & Trust boundary analysis & TB \\
\bottomrule
\end{tabular}

\smallskip\noindent{\footnotesize $^*$Status key (per Table~\ref{tab:proofstatus}): MV~=~machine-verified in Lean~4; IS~=~informal proof sketch describing the intended Lean~4 proof structure, not compiled; TB~=~trust boundary characterization rather than a formal property.}
\end{table}

Table~\ref{tab:proofstatus} (Section~\ref{sec:contributions}) summarizes the verification status of each major result.

\subsection{Complexity Analysis}
\label{sec:complexityanalysis}

Table~\ref{tab:complexity} summarizes the time complexity of each algorithm. The dominant cost across all use cases is the similarity computation in the claim-evidence match scoring component, not the formal verification layer.

\begin{table}[H]
\centering
\caption{Complexity analysis. $n{=}$claim components, $m{=}$document segments, $d{=}$feature dimension, $k{=}$constructions, $p{=}$portfolio claims, $|\Sigma|{=}$vocabulary size.}
\label{tab:complexity}
\begin{tabular}{@{}lll@{}}
\toprule
\textbf{Algorithm} & \textbf{Time Complexity} & \textbf{Dominant Term} \\
\midrule
1: DAG Coverage           & $O(nmd + n + |E|)$ & Similarity computation \\
2: FTO Proof              & $O(nmd)$  & Per-element score computation$^{\dagger}$ \\
3: Construction Sens.     & $O(knmd)$$^{\ddagger}$          & $k$ construction evaluations \\
4: Cross-Claim            & $O(p^2|\Sigma|)$   & Pairwise consistency checks \\
5: DOE                    & $O(nmd + n|\textit{PH}|)$ & Initial scoring + FWR + estoppel \\
6: Fixed-Point            & $O((h{+}1)(n + |E|))$ & Kleene iteration; $h{=}$DAG depth \\
\bottomrule
\end{tabular}

\smallskip
{\footnotesize $^{\dagger}$Algorithm~2 currently implements only Strategy~1 (per-element threshold checking); no DAG propagation is performed. Strategy~2 (future work; see Remark following Algorithm~\ref{alg:fto}) would add an $O(n + |E|)$ propagation term when implemented.\\
$^{\ddagger}$The $O(knmd)$ figure covers the $k$ construction evaluations only. Determinative-term identification adds $O(|\Sigma|nmd)$ (one construction evaluation per term in the shared vocabulary); see the complexity remark following Algorithm~\ref{alg:construction} in Section~\ref{sec:algorithms}.}
\end{table}

Concrete runtime numbers are not reported in this paper: the prototype has not been systematically benchmarked, and no directly comparable Lean~4 patent-analysis benchmark exists in the literature (existing Lean~4 deployments such as AWS Cedar~\cite{aws2024cedar} use Lean for offline verification of an evaluation engine rather than for runtime kernel type-checking). For the typical case of $n = 15$ components, $m = 50$ segments, and $d = 768$ (embedding dimension), the match-scoring stage is dominated by the ML inference cost; the DAG propagation and coverage computation in Algorithm~\ref{alg:dagcov} require only a topological sort and a weighted sum over $n$ nodes, which is expected to be dominated by ML inference for claims of this size (empirical benchmarking is future work). Kernel type-checking cost for the proof certificate depends on the size of the generated proof terms, which is a function of the prototype's generator implementation rather than of the algorithm itself. Establishing actual wall-clock runtimes for each stage (ML inference, proof generation, kernel verification) on standardized hardware is a priority for future work (Section~\ref{sec:discussion}).

Algorithm~2 (FTO) checks each element's best match score against the threshold in $O(n \cdot m \cdot d)$. The overall FTO analysis is therefore polynomial in all parameters.

Algorithm~4 (Cross-Claim Consistency) has $O(p^2 |\Sigma|)$ complexity, which for a large portfolio of $p{=}5{,}000$ claims and $|\Sigma|{=}200$ terms yields $5 \times 10^9$ comparisons. While large, this is structured (not brute-force) and amenable to parallelization across independent term groups.

Figure~\ref{fig:complexity} illustrates the asymptotic scaling shapes implied by the complexity classes in Table~\ref{tab:complexity}; it conveys no empirical information beyond the table and should not be interpreted as a runtime measurement.

\begin{figure}[H]
\centering
\includegraphics[width=\textwidth]{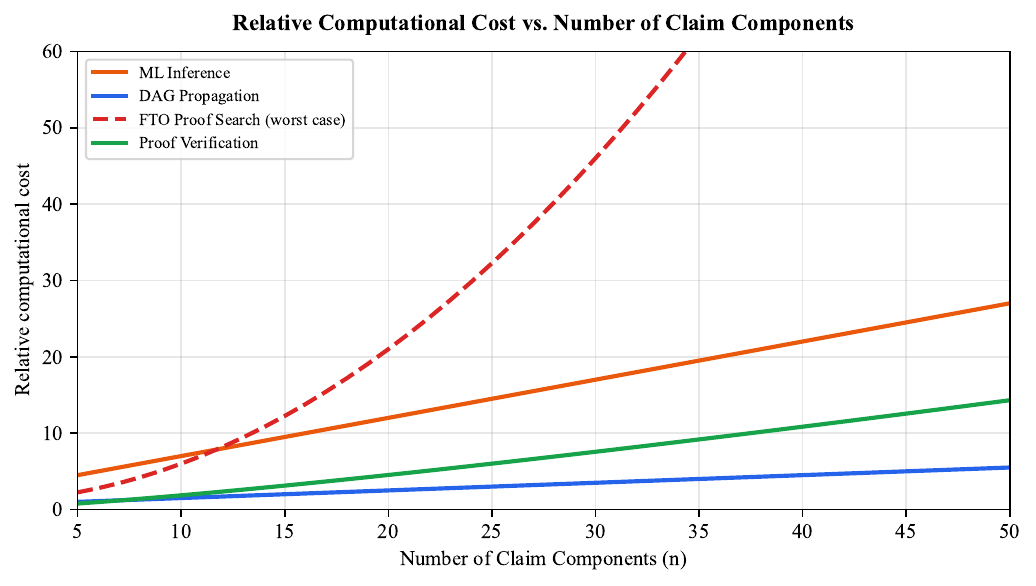}
\caption{\textbf{Schematic of asymptotic scaling under typical parameter assumptions} (not empirical runtimes). Shapes implied by the complexity classes in Table~\ref{tab:complexity}, with $m$ and $d$ held fixed. The $y$-axis is in arbitrary units; constant factors are undetermined without benchmarking. The relative ordering of slopes (DAG Propagation $\prec$ Proof Verification $\prec$ ML Inference) is qualitatively justified in Section~\ref{sec:complexityanalysis} by the structure of the primitive operations at each stage. Empirical benchmarking remains future work.}
\label{fig:complexity}
\end{figure}

\noindent\textbf{Qualitative justification of slope ordering.}
Although the specific constant factors in Table~\ref{tab:complexity} cannot be determined without empirical benchmarking, the relative ordering of computational costs across stages, DAG Propagation $\prec$ Proof Verification $\prec$ ML Inference, is qualitatively justified by the structure of the primitive operations each stage performs, independent of implementation details or hardware characteristics.

DAG Propagation (Algorithm~\ref{alg:kleene}) performs $O(n + |E|)$ operations consisting entirely of rational number threshold comparisons and score assignments: one comparison per dependency edge and one assignment per node. For the sparse patent claim DAGs considered in this work, $|E|/n$ is small (approximately 1.33 in the memory module case study, with $|E| = 20$ and $n = 15$), so the constant factor multiplying $n$ is close to~2.33. This is the computationally cheapest class of primitive operation: a single numerical comparison with no memory indirection, no branching, and no algebraic reduction.

Proof verification in the Lean~4 kernel scales with the size of the proof term, which grows linearly with $n$ as each additional claim node adds a bounded number of proof obligations to the \texttt{ProofCertificate} fields. Each proof step requires the kernel to perform operations that are structurally more expensive than a numerical comparison, including definition lookup, constructor checking, definitional equality via term reduction (beta, delta, and iota reduction in the Calculus of Inductive Constructions), and universe level checking. However, the number of such operations per additional node is bounded by the structure of the proof, with no additional multiplicative factor from external data. The slope of proof verification therefore exceeds that of DAG propagation due to a larger per-operation cost, while remaining smaller than ML inference because it carries no data-dependent multiplier.

The ML Inference component carries an explicit multiplicative factor absent from both other stages: for each additional claim node, the system must compute match scores against all $m$ evidence segments using $d$-dimensional embedding vectors, incurring $O(m \cdot d)$ floating-point multiply-accumulate operations for embedding similarity alone. With the typical parameter values of $m = 50$ document segments and $d = 768$ embedding dimensions, this amounts to approximately 38,400 floating-point operations per additional claim node. This structural multiplier dominates the slope of the ML Inference line relative to both other stages, regardless of the relative per-operation costs of floating-point arithmetic and kernel type-checking. The slope ordering DAG Propagation $\prec$ Proof Verification $\prec$ ML Inference is therefore architecturally stable: it reflects the multiplicative structure of each stage's computational pattern rather than any property of a specific implementation.

\subsection{Structural Properties and Formal Limitations of Probabilistic Approaches}
\label{sec:probcomparison}

The distinction between the hybrid approach introduced in this work and purely probabilistic ML/NLP systems is not a matter of degree but of \textit{kind}. Numeric capability comparisons would require empirical benchmarks on standardized patent analysis datasets evaluated against adjudicated case outcomes; no such datasets currently exist, and constructing them is identified as a priority for future work (Section~\ref{sec:discussion}). In the absence of empirical benchmarks, the most rigorous characterization available is a structural one: identifying properties that are provably achievable or provably unachievable by each class of system, independent of implementation quality or scale. We establish three such results before presenting a property-by-property comparison in Table~\ref{tab:comparison}.

\begin{proposition}[Proof Certificate Impossibility for Probabilistic Systems]
\label{prop:certimpossible}
A system whose output is a sample from a probability distribution cannot unconditionally guarantee the production of a valid \texttt{ProofCertificate} (Section~\ref{sec:framework}).
\end{proposition}
\begin{proof}
A \texttt{ProofCertificate} is valid (Definition~\ref{def:validcert}) iff every proof field type-checks in the Lean~4 kernel AND the certificate is \texttt{sorry}-free. Type-checking is a deterministic binary predicate: for any term $t$ and type $T$, either $\vdash t : T$ or $\not\vdash t : T$; there is no intermediate probabilistic state. While constrained generation techniques can increase the probability that a generated term is syntactically well-formed, they cannot provide an unconditional guarantee that the generated term satisfies the semantic validity condition (that it constitutes a genuine proof of the stated proposition in the Calculus of Inductive Constructions) without that constraint mechanism being itself formally verified. No probability bound $p < 1$ constitutes an unconditional guarantee: a system that produces valid certificates with probability $p$ fails with probability $1 - p$, and those failures are precisely the cases where legal correctness is most consequential. The formal verification layer of the hybrid pipeline avoids this by construction: given identical inputs, the Lean~4 kernel either accepts or rejects deterministically, with no probabilistic intermediate.
\end{proof}

The force of this result is definitional rather than empirical. Probabilistic systems can propose candidate proof terms, but they cannot themselves provide unconditional proof-validity guarantees for their own outputs without an external deterministic checker; a confidence score is a statistical estimate conditioned on a training distribution, whereas a proof certificate's validity is a deterministic binary property of a proof term relative to the Lean~4 kernel's inference rules. A probabilistic system paired with an external deterministic proof checker (such as this paper's hybrid architecture) can produce certificates that the checker accepts or rejects; the proposition says only that the \textit{unconditional} guarantee must come from the checker, not from the probabilistic component.

\begin{proposition}[Dependency Structure Necessity under the All-Elements Rule]
\label{prop:compositionality}
Let $\mathcal{A}$ be any \textit{flat-scoring} analysis, one that computes per-element scores independently of the claim DAG's dependency structure and combines them by a function of the marginal per-element scores alone (e.g., a weighted sum, product, or minimum over $V$). Then $\mathcal{A}$ cannot faithfully model claim satisfaction under the all-elements rule, in either of the following senses: (a)~if $\mathcal{A}$ is probabilistic, the probability it computes cannot equal the joint satisfaction probability without an exogenous distributional assumption that is not available from the marginals or the DAG; (b)~if $\mathcal{A}$ is deterministic, the score it computes cannot distinguish configurations that differ only in whether transitive dependencies are satisfied. Graph-structured analyses that explicitly condition on the claim DAG (for example, graph neural networks operating on the dependency graph, including FLAN-Graph~\cite{gao2024flan}) fall outside the scope of this proposition.
\end{proposition}
\begin{proof}
Under the all-elements rule (Definition~\ref{def:claimdag}), infringement requires that every limitation $v \in V$ be satisfied simultaneously:
\[
\text{Infringes} \;\Longleftrightarrow\; \bigwedge_{v \in V} \text{satisfied}(v).
\]

\textit{Part (a), probabilistic case.} The probability of infringement, $P\bigl(\bigwedge_{v \in V} \text{satisfied}(v)\bigr)$, is a property of the \emph{joint} distribution of the events $\{\text{satisfied}(v)\}_{v \in V}$ over the population of accused products, not of the marginals $\{P(\text{satisfied}(v))\}_{v \in V}$ alone. For any fixed set of marginals, the Fr\'echet--Hoeffding bounds place the conjunction probability anywhere in a non-degenerate interval that collapses to a single point only under an additional structural assumption about the joint, for example $P\bigl(\bigwedge_{v} \text{satisfied}(v)\bigr) = \prod_v P(\text{satisfied}(v))$ under statistical independence. Statistical independence of the satisfaction events, however, is a property of the joint distribution over real products, not of the DAG $G$: two logically independent limitations may be statistically correlated across real products (e.g., ``has a screen'' and ``has a battery'' in phones), and two logically dependent limitations may happen to be statistically independent in the population being evaluated. The DAG therefore cannot by itself supply the independence assumption, and the marginals cannot self-justify one. A flat-scoring probabilistic system combining marginals by any fixed function of $\{P(\text{satisfied}(v))\}_{v}$ consequently cannot compute $P(\text{Infringes})$ under the all-elements rule without an exogenous distributional assumption that is not available from either its inputs or its own outputs. Part~(b) sharpens this into a deterministic impossibility that does not depend on any distributional premise.

\textit{Part (b), deterministic case.} Consider two score configurations $\beta_1, \beta_2 : V \to [0,1]$ that agree on every node except one: there exist nodes $u, v \in V$ with $u \in \textit{deps}(v)$ such that $\beta_1(u) = 0$ and $\beta_2(u) \geq \theta$, while $\beta_1(w) = \beta_2(w)$ for every $w \neq u$. The all-elements rule yields \textit{different} infringement outcomes for $\beta_1$ and $\beta_2$ (under $\beta_1$, $u$ is unsatisfied, so the transitive dependent $v$ cannot be credited; under $\beta_2$, $u$ is satisfied). But a flat-scoring deterministic analysis by definition outputs a value $\mathcal{A}(\beta_i) = F(\beta_i(v_1), \ldots, \beta_i(v_n))$ for some function $F$ that sees only the per-element scores and is blind to the DAG. $F$ therefore assigns the same relative contribution to $u$ in both configurations, and cannot record that $u$'s contribution in $\beta_1$ should also zero $v$'s contribution (via the cascade). The two configurations are distinguishable by the all-elements rule but indistinguishable to $\mathcal{A}$ up to the single-element change at $u$. Concretely, the memory-module case study (Section~\ref{sec:algorithms}) exhibits this: a $0.24$-point reduction in $C_3$'s score cascades to zero three additional high-weight nodes under the threshold model, a $23.3$~pp coverage swing that no weighted sum over the original per-element scores can reproduce without explicit dependency conditioning.

In both cases, correctness under the all-elements rule requires conditioning the combination step on the DAG. The weighted DAG coverage $W_{\text{cov}}$ (Definition~\ref{def:effcov}) and the dependency enforcement of Algorithm~\ref{alg:dagcov} account for this structure precisely by propagating dependency constraints through the DAG topology, zeroing the effective score of any limitation whose transitive dependencies are unsatisfied. The theorem says only that any correct analysis (probabilistic or deterministic) must similarly condition on the DAG at the combination step; our framework's output $W_{\text{cov}}$ is a deterministic weighted score that does so, not itself a probability of infringement.
\end{proof}

This result has a direct practical consequence. A patent claim is a conjunction of limitations; legally, all must be present for infringement. A system that reports per-element confidence scores and sums or averages them without conditioning on the DAG is implicitly committing to a distributional or structural assumption (statistical independence in the probabilistic case; insensitivity to dependency cascades in the deterministic case) that neither the marginals nor the DAG can justify. Algorithm~\ref{alg:dagcov}'s propagation mechanism supplies the missing dependency conditioning in the deterministic case (Part~(b)); probabilistic systems lacking a comparable joint-distribution model inherit the gap identified in Part~(a).

\textit{Scope of Proposition~\ref{prop:compositionality}.} The proposition targets flat-scoring ML pipelines, systems whose per-element scores are computed without access to the dependency edges. It does \textit{not} apply to graph-structured models that explicitly condition on the DAG (graph neural networks, graph attention networks, and similar approaches operating over the claim dependency graph such as FLAN-Graph~\cite{gao2024flan}, cited in Section~\ref{sec:related}). For such models the compositionality gap identified here does not arise from the architecture; separate claims about their calibration, reproducibility, and certificate-generation properties (Propositions~\ref{prop:certimpossible} and~\ref{prop:reproducibility}) continue to apply to the extent that the model's output remains probabilistic.

\begin{proposition}[Stochastic Inference and Unconditional Reproducibility]
\label{prop:reproducibility}
Any system whose outputs are sampled from a non-degenerate probability distribution cannot guarantee that two independent runs on identical inputs produce identical outputs.
\end{proposition}
\begin{proof}
A system uses stochastic inference if its output is a function of both the input and a random variable $\omega$ drawn from a distribution with support of cardinality greater than one (as is the case for LLM decoding with temperature $T > 0$, Monte Carlo dropout, or non-deterministic floating-point operations across hardware). For any such system, there exist realisations $\omega_1 \neq \omega_2$ such that $f(\text{input}, \omega_1) \neq f(\text{input}, \omega_2)$. This is a definitional consequence of stochasticity, not an empirical observation about any particular system or scale. The formal verification layer of the hybrid pipeline is deterministic by construction: given identical input scores and DAG structure, the Lean~4 kernel always produces the same proof certificate, and this guarantee holds unconditionally for all valid inputs.
\end{proof}

Propositions~\ref{prop:certimpossible} through~\ref{prop:reproducibility} collectively establish that proof certification, compositional correctness under the all-elements rule, and unconditional deterministic reproducibility are not achievable by probabilistic systems through any degree of engineering improvement. They represent qualitative discontinuities separable from questions of scale, model quality, or prompting strategy. Table~\ref{tab:comparison} extends this structural analysis across nine properties covering both the structural guarantees argued in this section (Propositions~\ref{prop:certimpossible}--\ref{prop:reproducibility}, with inline mathematical proofs; verification status per Table~\ref{tab:proofstatus}) and the practical capabilities of the respective approaches.

\begin{table}[H]
\centering
\caption{Structural property comparison across ML/NLP-only, formal verification-only, and hybrid approaches for patent analysis. All entries are derived from definitions, theorems and proof sketches in this paper (verification status per Table~\ref{tab:proofstatus}), or cited formal results; no entry is based on empirical measurement. ``Conditional'' denotes that the stated property holds for the formal verification layer given ML-layer inputs; the correctness of those inputs remains the responsibility of the ML component. See Section~\ref{sec:discussion} for a full discussion of the conditional trust model.}
\label{tab:comparison}
\scriptsize
\resizebox{\textwidth}{!}{%
\begin{tabular}{p{2.5cm}p{4.0cm}p{3.2cm}p{3.8cm}}
\toprule
\textbf{Property} & \textbf{ML/NLP Only} & \textbf{Formal Verification Only} & \textbf{Hybrid (This Work)} \\
\midrule
Produces machine-checkable proof certificate & No (Prop.~\ref{prop:certimpossible}): probabilistic outputs cannot unconditionally type-check & Yes: \texttt{ProofCertificate} verified by Lean~4 kernel & Yes, conditional on ML-layer input scores \\[4pt]
Output deterministic given identical inputs & No (Prop.~\ref{prop:reproducibility}): stochastic decoding is definitionally non-deterministic & Yes: Lean~4 kernel is deterministic by construction & Formal layer: yes; ML layer: no \\[4pt]
Can process raw natural language & Yes & No: requires formally pre-structured input & Yes, via the ML/NLP layer components \\[4pt]
Trusted computing base & Entire model, training data, inference code, hardware & Lean~4 kernel (${\sim}$5K lines C++), self-verified~\cite{carneiro2024lean4lean}, \textit{plus} the correctness of the manual formalization (which a pure formal system requires a human to supply) & Lean~4 kernel for proof validity; manual formalization of structural inputs; full ML stack for scoring \\[4pt]
Guarantees for all valid inputs & No (PAC, probably approximately correct bounds) & Structural lemmas (\texttt{dag\_acyclic}, \texttt{CompleteLattice}) machine-verified; higher-level theorems (\ref{thm:weakest}, \ref{thm:monotone}, \ref{thm:algcorrect}, \ref{thm:fto}, \ref{thm:construction}, \ref{thm:consistency}, \ref{thm:doe}, \ref{thm:convergence}) currently informal sketches per Table~\ref{tab:proofstatus} & Same as formal-only column for the formal layer, conditional on ML-layer input scores \\[4pt]
Independently verifiable by third party & No: requires model weights and runtime & Yes: certificate is self-contained & Yes for formal outputs; ML scores require separate audit \\[4pt]
Correct composition under all-elements rule & No for flat-scoring systems (Prop.~\ref{prop:compositionality}): $P(A \wedge B)$ cannot be computed from marginals without independence$^\dagger$ & Yes: $\wedge = \min$ is associative, commutative, idempotent (Thm.~\ref{thm:weakest}) & Yes for formal computation layer \\[4pt]
Explicit trust boundary & No: no structural record of where guarantees end & Yes: all non-kernel inputs explicitly unverified & Yes; see Section~\ref{sec:architecture} and Figure~\ref{fig:architecture} \\[4pt]
Adversarial robustness of certified output & No: text perturbations alter similarity scores & Proof type-checks or not, regardless of input construction & ML scores: vulnerable; certificate structure: immune \\
\bottomrule
\end{tabular}%
}
\end{table}

\noindent$^\dagger$\textit{Proposition~\ref{prop:compositionality} applies to systems that compute per-element scores independently without conditioning on the DAG. Graph-structured models that explicitly model dependency structure (e.g., graph neural networks operating on claim DAGs) are not covered by this result.}

\noindent\textbf{Note on $\Phi_v$.} The claim scope space $\Phi_v$ is a formal input to the framework. Its construction (whether by human experts or ML tools) is outside the formal layer for all three system types. The hybrid approach makes this dependency explicit through the trust boundary; pure ML approaches embed it implicitly in the model's internal representations.

\noindent\textbf{Discussion.}
Table~\ref{tab:comparison} reveals a structure of complementary strengths that the hybrid approach is uniquely positioned to combine. Formal verification contributes the properties that are provably impossible for probabilistic systems: machine-checkable certificates (Proposition~\ref{prop:certimpossible}), compositionally correct propagation (Proposition~\ref{prop:compositionality}), unconditional deterministic reproducibility (Proposition~\ref{prop:reproducibility}), a minimal trusted computing base, and immunity to adversarial manipulation of the proof structure itself. The ML/NLP component contributes the one capability that formal verification structurally cannot supply: the ability to process raw natural language patent claim text without requiring prior manual formalization.

The hybrid approach inherits both sets of properties subject to the trust boundary established in Section~\ref{sec:architecture}. That trust boundary is itself a structural feature absent from both pure approaches. A pure ML system produces no explicit record of where its guarantees end; its entire output is probabilistic without a structurally encoded boundary between what has been verified and what has been estimated. A pure formal verification system, when applied to manually formalized claims, provides unconditional guarantees for the formal computation but offers no mechanism to reach the natural language input; the conditionality of the guarantee on the human formalizer's correctness is implicit rather than structurally encoded. The hybrid pipeline makes conditionality explicit: proof certificates certify exactly the formally verified portions of the analysis, and the trust boundary in Figure~\ref{fig:architecture} identifies precisely which components lie below it.

The conditional nature of the hybrid's guarantees warrants direct acknowledgment. A formally verified computation applied to systematically biased ML scores will certify incorrect conclusions with mathematical precision. The formal guarantee (correctness of computation given the inputs) does not extend to the correctness of the inputs themselves. This limitation is not a defect unique to the hybrid; it is the honest characterization of any system that processes natural language through a non-formal intermediary. The contribution of the formal layer is to make this conditionality auditable: one can inspect the proof certificate, identify the input scores on which it rests, and independently evaluate those scores, in a way that is not possible for an opaque probabilistic output. The framework does not eliminate the need for judgment about ML-layer quality; it isolates and makes explicit the precise point at which that judgment is required.

No claim is made that the hybrid approach achieves superior performance on any empirically measurable task. The contribution is formal and architectural: this work introduces machine-checkable proof certificates into patent analysis for the first time, establishes a formally characterized trust boundary between probabilistic and verified computation, and argues conditional mathematical correctness for six algorithms covering five IP use cases through informal proof sketches grounded in machine-verified structural lemmas (verification status per Table~\ref{tab:proofstatus}). Whether this framework improves on practitioner outcomes in deployed settings is an empirical question that requires adjudicated case data and is identified as the primary direction for future work.

\section{Implementation and Case Study}
\label{sec:implementation}

\subsection{Prototype}

The framework takes the form of a prototype realizing the full hybrid pipeline described in Section~\ref{sec:architecture}. The prototype comprises an LLM-based claim decomposition module, a document extraction and segmentation backend, a claim-evidence match scoring component using combined lexical and semantic matching, a Lean~4 formal verification core for DAG encoding and lattice-theoretic propagation, and a web-based interface with interactive DAG visualization.

\noindent\textbf{Prototype stack.}
\begin{itemize}[leftmargin=2em, topsep=2pt, itemsep=1pt]
\item \emph{ML/NLP layer:} Python~3.11 with PyTorch for the embedding pipeline; sentence-level embeddings from a general-purpose transformer encoder combined with TF-IDF lexical scores per Section~\ref{sec:architecture}. Claim decomposition is performed by an instruction-tuned LLM under structured prompting; the resulting decomposition is written out as a JSON document that the Lean code generator consumes.
\item \emph{Lean~4 layer:} Lean~4 v4.30.0-rc1 with Mathlib. Organized as a Lake project (\texttt{patent\_verification/}) with the modules \texttt{ClaimNode.lean} (15 constructors for the case-study claim, dependency function, \texttt{claimWeight}), \texttt{MatchStrength.lean} (bounded-rational and discretized lattice representations), \texttt{DAGAcyclicity.lean} (\texttt{topoDepth} and \texttt{dag\_acyclic}), \texttt{Coverage.lean} (\texttt{computeEff}, \texttt{weightedCoverage}, \texttt{coverage\_in\_range}, \texttt{propag\_proof}, \texttt{eff\_cases}, \texttt{eff\_le\_score}), \texttt{Certificate.lean} (\texttt{ProofCertificate}, \texttt{generateCertificate}), and \texttt{CaseStudy.lean} (per-construction score tables and certificate instantiations).
\item \emph{Code generation:} A Python driver (\texttt{generate\_lean.py}) consumes the JSON decomposition and emits \texttt{ClaimNode.lean}, \texttt{DAGAcyclicity.lean}, \texttt{Coverage.lean}, and \texttt{CaseStudy.lean}; the remaining modules are hand-written. The generator is untrusted; its output is independently type-checked by the Lean~4 kernel and subjected to the Definition~\ref{def:validcert}(ii) $\Omega$-audit before acceptance.
\item \emph{Approximate sizes:} ML pipeline $\approx 1.5$k SLOC of Python; Lean~4 core $\approx 0.7$k SLOC (of which roughly half is generated); case-study Lean encoding $\approx 0.3$k SLOC.
\item \emph{End-to-end run status:} The prototype has been exercised end-to-end on the case-study claim (Section~\ref{sec:implementation}): the ML layer produces the best-match score tables, the generator emits the Lean source, and \texttt{lake build} compiles the project without warnings. The three top-level \texttt{\#print axioms} invocations (on \texttt{generateCertificate}, \texttt{cert\_I1\_broad}, and \texttt{cert\_I2\_narrow}) each report exactly $\Omega = \{\texttt{propext}, \texttt{Classical.choice}, \texttt{Quot.sound}\}$, satisfying Definition~\ref{def:validcert}(ii). No wall-clock runtimes are reported in this paper.
\item \emph{Axiom audit on a real claim.} When the prototype is applied to a real (non-synthetic) claim, the audit result is expected to be identical to the synthetic case for Algorithm~\ref{alg:dagcov-given}'s coverage core, because the generator's axiom dependencies are independent of the specific \texttt{ClaimNode} constructors and score values. The remaining generators (Algorithms~2--6) are architecturally mitigated; their audit will flag any \texttt{sorryAx} or \texttt{Lean.ofReduceBool} that slips through, causing the certificate to be rejected as required.
\end{itemize}

\noindent The case study in Section~\ref{sec:implementation} illustrates the application of the framework to the memory buffer module claim and demonstrates the generation of machine-checkable proof certificates of the form described in Section~\ref{sec:framework}.

\subsection{Case Study: Memory Buffer Module Claim}

We apply our framework to a synthetic patent claim. The claim describes a memory buffer module with a rank interposition circuit, a technology domain in which rank multiplication is a well-established architectural concept. The decomposition into atomic limitations C1--C15 and the dependency structure in Figure~\ref{fig:claimdag} were produced by the claim decomposition pipeline and are used here as fixed inputs to the formal verification layer. As discussed in Section~\ref{sec:discussion} (Limitation~1), evaluating whether this is the only valid decomposition is outside the scope of this work. \textit{Limitation of the synthetic case study:} Because this claim is synthetic, there is no adjudicated outcome to compare against, no real prosecution history, and no real product documentation. The case study demonstrates the formal mechanics of the framework (DAG encoding, lattice propagation, proof certificate generation, and construction sensitivity analysis) but cannot validate that the framework would reach legally correct conclusions in adversarial settings with real patent claims. Validation against adjudicated cases is identified as future work in Section~\ref{sec:discussion}.

\begin{quote}\small
\textbf{Claim~1.} A memory buffer module connectable to a host computer system, the memory buffer module comprising: a printed circuit board; a plurality of double-data-rate (DDR) memory devices mounted to the printed circuit board, the plurality of DDR memory devices organized into a first number of ranks; a rank interposition circuit mounted to the printed circuit board, the rank interposition circuit comprising a logic element and a register, the logic element configured to receive a set of input control signals from the host computer system, the set of input control signals comprising at least one address signal, at least one bank address signal, and at least one chip-select signal, the set of input control signals corresponding to a second number of ranks, the logic element further configured to generate a set of output control signals corresponding to the first number of ranks; and a phase-locked loop device mounted to the printed circuit board, the phase-locked loop device operatively coupled to each of the plurality of DDR memory devices, the logic element, and the register; wherein the register stores rank translation information mapping the second number of ranks to the first number of ranks, the second number of ranks being less than the first number of ranks, and the logic element generates the set of output control signals by applying the rank translation information to the set of input control signals received from the host computer system; and wherein the rank interposition circuit further translates memory access commands received from the host computer system corresponding to the second number of ranks into translated commands corresponding to the first number of ranks, and transmits the translated commands together with the set of output control signals to the plurality of DDR memory devices.
\end{quote}

The two claim constructions used in the sensitivity analysis (Section~\ref{sec:algorithms}) differ on two terms. Under $I_1$ (broad), ``a first number of ranks'' denotes any quantity of ranks with no ratio constraint, and ``rank translation information'' denotes any mapping mechanism. Under $I_2$ (narrow), ``a first number of ranks'' requires a specific 2:1 ratio relative to the second number of ranks, and ``rank translation information'' requires a dedicated rank translation circuit with an explicit address mapping stage. All other claim terms are interpreted identically under both constructions. In both cases $I_2$'s reading is a strict narrowing of $I_1$'s in the implementation-scope sense of Definition~\ref{def:scopespace}: the 2:1-ratio implementations form a strict subset of the unrestricted-ratio implementations, and the dedicated-rank-translation-circuit implementations form a strict subset of the unrestricted-mechanism implementations. Consequently $\Phi_v^{I_2} \subseteq \Phi_v^{I_1}$ for the two contested elements ($v \in \{C_3, C_{12}\}$), and the term-wise monotonicity premise of Theorem~\ref{thm:construction}(iv) holds: narrowing either contested term alone from $I_1$'s reading toward $I_2$'s never increases any $\beta(v)$ in the resulting score vector, consistent with the drop from $\beta(C_3, I_1) = 0.82$ to $\beta(C_3, I_2) = 0.58$ and from $\beta(C_{12}, I_1) = 0.80$ to $\beta(C_{12}, I_2) = 0.52$ in Table~\ref{tab:scores}. The complete set of best-match scores for all 15 nodes under both constructions, together with step-by-step coverage and waterfall verification, is provided in Table~\ref{tab:scores} (Appendix~A.6) to enable independent reproduction of all computations below.

The framework yields a weighted coverage score of 82.3\% under $I_1$. Two complementary notions of ``weakest'' apply here and should not be conflated. First, the \emph{lowest-scoring node} under $I_1$ is $C_{13}$ (Command Mapping) at $\beta = \textit{eff}(C_{13}, I_1) = 0.77$ (Table~\ref{tab:scores}); this is the quantity used in Proposition~\ref{prop:weightrobust} (Section~\ref{sec:discussion}) to establish the weight-independent $77\%$ lower bound on $W_{\text{cov}}(I_1)$. Second, the \emph{highest-weight non-leaf bottleneck} is $C_{12}$ (Rank Translation) at $\beta = 0.80$; $C_{12}$ affects the assessment through two distinct mechanisms: (i)~as one of the highest-weighted elements (tied with $C_{13}$ at $w{=}3.0$, triple the structural baseline), $C_{12}$ contributes disproportionately to the weighted average $W_{\text{cov}}$; and (ii)~under the meet-based sensitivity model (Definition~\ref{def:claimstrength}), $C_{12}$'s $0.80$ score propagates as an upper bound on the claim strength of its dependents, providing the actionable insight that $C_{12}$ is the highest-weight upstream bottleneck (its dependent $C_{13}$'s claim strength is further reduced to $\min(0.80, 0.77) = 0.77$ by $C_{13}$'s own score). (Important clarification: mechanism~(ii) is a property of the $\textit{claimStrength}$ model, which is distinct from the threshold-based model (Definition~\ref{def:effcov}) used for the $W_{\text{cov}}$ computation. In the threshold model, $C_{12}$'s $0.80$ exceeds $\theta = 0.65$, so $C_{13}$'s effective score is \textit{not} zeroed under $I_1$; the meet-based bounding is a complementary sensitivity insight, not a coverage penalty.) (Figure~\ref{fig:lattice}(b) uses a schematic example with C12 at 0.45 to illustrate the meet-based propagation mechanism; the case study here uses actual ML-computed scores.)

\begin{table}[t]
\centering
\caption{Case study: flat vs.\ DAG-weighted analysis.}
\label{tab:casestudy}
\begin{tabular}{@{}llll@{}}
\toprule
\textbf{Metric} & \textbf{Flat} & \textbf{DAG-Weighted} & \textbf{Delta} \\
\midrule
Overall Coverage     & 83.8\% & 82.3\% & $-$1.5\%$^*$ \\
Weighting            & Equal  & Type-based (0.3--3.0) & More nuanced \\
Dependencies         & Ignored & Enforced with proofs & Structural correctness \\
Lowest-scoring node   & N/A   & $C_{13}$ at $0.77$ & Tight bound in Prop.~\ref{prop:weightrobust} \\
Highest-weight upstream bottleneck & N/A & $C_{12}$ at $0.80$ & Actionable insight (meet-based) \\
Proof certificate    & None   & Machine-checkable & Potential evidentiary artifact$^{**}$ \\
\bottomrule
\end{tabular}

\smallskip\noindent{\footnotesize $^*$Under $I_1$, all 15 nodes have $\beta(v) \geq \theta = 0.65$, so dependency enforcement changes no scores; the 1.5pp difference is entirely due to the type-based weighting scheme.\\
$^{**}$Admissibility as legal evidence is untested (see \S\ref{sec:discussion}, Legal Implications); the certificate is a kernel-checked analytical record of the computation downstream of the ML layer, not an adjudicated evidentiary artifact.}
\end{table}

\medskip
\noindent\textbf{Weight-independent robustness.} The case-study conclusion under $I_1$ (``satisfied at $\theta_{\text{cov}} = 70\%$'') is robust to \emph{any} positive reweighting of the seven element types: because $\min_v \textit{eff}(v, I_1) = 0.77$ (Table~\ref{tab:scores}, attained at $C_{13}$), Proposition~\ref{prop:weightrobust} (Section~\ref{sec:discussion}) gives a weight-independent lower bound $W_{\text{cov}}(G, I_1; w') \geq 77\%$ for every positive weight vector $w'$, with a guaranteed margin of $7$~pp above the $70\%$ satisfaction threshold. This is the tightest bound on the legal conclusion's sensitivity to the unvalidated weight scheme and isolates the reliance on the specific Table~\ref{tab:weights} numbers: the satisfied conclusion under $I_1$ survives arbitrary re-weighting, while the \emph{not-satisfied} conclusion under $I_2$ does not (Proposition~\ref{prop:weightrobust} and the ensuing discussion in Section~\ref{sec:discussion}).

\section{Discussion}
\label{sec:discussion}

\subsection{Legal Implications}

The proof certificates produced by the framework are best described as machine-checkable analytical records of the computation downstream of the ML layer: kernel-checked, reproducible, and independently verifiable artifacts that \emph{could} contribute to legal analysis, but whose status as admissible legal evidence is untested. Under the Daubert standard~\cite{daubert1993}, expert testimony is evaluated for scientific validity, and proof certificates exhibit properties (determinism, reproducibility, transparency, and independent verifiability) that align with several of Daubert's criteria for admissible expert methodology. However, courts have not yet considered formal proof certificates as evidentiary artifacts, whether they satisfy Daubert's ``general acceptance in the relevant scientific community'' requirement remains an open legal question, and even if admitted as expert-methodology support, the certificate's weight would remain conditional on the quality of the ML-layer inputs (which lie below the trust boundary and are not themselves certified). We present the certificates as a candidate evidentiary artifact for expert analysis rather than as legal evidence in the adjudicative sense, and regard the question of judicial acceptance as a promising direction meriting further scholarly and judicial consideration.

Practical impact for different stakeholders:
\begin{itemize}[leftmargin=2em]
\item \textbf{Patent litigators:} Transparent, reproducible claim charts whose underlying computation is kernel-checked; admissibility as evidence remains an untested legal question (see the Daubert discussion above).
\item \textbf{Patent prosecutors:} Claim construction sensitivity identifies candidate determinative terms and a scope-minimal satisfied construction under the formal model; results are conditional on the provided constructions and ML-computed scores.
\item \textbf{Corporate IP teams:} FTO proof certificates may contribute to willfulness-defense documentation by providing a kernel-checked record of the non-infringement analysis, though their weight as evidence is untested in court and depends on the ML layer's scoring accuracy. Practitioners must understand the asymmetry of the FTO result (Theorem~\ref{thm:fto}): a \textsc{Clear} result is a formal (conditional) proof of per-element non-infringement assuming the ML scores are faithful; a missed or misscored relevant evidence segment can yield a formally valid certificate that does not reflect the true legal situation. A \textsc{Risk} result does \textit{not} mean infringement exists; it means only that per-element impossibility could not be established, and a human patent attorney must make the final FTO determination. The system provides a ``no evidence of impossibility'' guarantee, not a ``no infringement'' guarantee, and the gap between these is precisely where legal judgment is required.
\item \textbf{Portfolio managers:} Higher-throughput analysis than manual claim charting is plausible, but concrete throughput has not been benchmarked in this paper; scalability claims are projective rather than empirical (see Section~\ref{sec:complexityanalysis} and Limitation~3 in Section~\ref{sec:discussion}).
\end{itemize}

\subsection{Limitations and Future Work}

\noindent Each limitation below concerns an input or component whose trust-boundary placement is recorded in the Guarantee Map of Section~\ref{sec:architecture}; readers who want the at-a-glance summary of what is trusted where should consult that table.

\begin{enumerate}[leftmargin=2em]

\item \textbf{LLM decomposition errors.} All formal guarantees of this system are conditional on the correctness of the LLM-based claim decomposition component. Three distinct failure modes exist: (i)~\textit{omission}: the LLM fails to identify one or more claim limitations, causing coverage to appear artificially high; (ii)~\textit{fabrication}: the LLM invents spurious limitations not present in the claim, causing coverage to appear artificially low; (iii)~\textit{misclassification}: the LLM correctly identifies limitations but assigns wrong types, weights, or dependencies, causing incorrect weakest-link analysis. Lean's type system validates structural properties of the decomposition (acyclicity, valid node types, defined dependencies) but cannot detect semantic errors, since it has no access to the original claim text. More effective mitigations include: expert human review of LLM decompositions before formal encoding; cross-validation across multiple independent LLM calls; structured prompting constrained by patent claim grammar; and statistical calibration of LLM decomposition accuracy against manually verified claim sets.

\item \textbf{$\Phi_v$ instantiation errors (DOE-specific).} Distinct from claim decomposition, the doctrine of equivalents analysis (Algorithm~\ref{alg:doe}) depends on the quality of the scope space $\Phi_v$ (Definition~\ref{def:scopespace}) and its partition into $\Phi_v^{\text{orig},k}$ and $\Phi_v^{\text{amend},k}$. $\Phi_v$ is a legal-scope artifact, not a claim-text artifact: errors here corrupt the estoppel analysis rather than the structural decomposition, and the formal layer cannot detect them. Three distinct failure modes exist: (i)~\textit{scope underrepresentation}: $\Phi_v$ fails to include implementations near the scope boundary, making the surrendered territory $\Phi_v^{\text{orig},k} \setminus \Phi_v^{\text{amend},k}$ empty or misleadingly small; (ii)~\textit{scope misclassification}: implementations are incorrectly assigned to $\Phi_v^{\text{orig},k}$ or $\Phi_v^{\text{amend},k}$, corrupting the estoppel analysis; (iii)~\textit{evidence mismatch}: $\Phi_v$ does not contain implementations resembling the accused product, making the projection $\pi_v(s)$ (Definition~\ref{def:scopespace}) unreliable. These failures are undetectable by the formal layer because $\Phi_v$ is a given input. Expert review of $\Phi_v$ and its scope partitions before formal encoding is the primary mitigation; structured scope elicitation protocols (e.g., templates driven by claim construction rulings) are identified as future work.

\item \textbf{Weight scheme is legally unvalidated.} The current weight scheme (preamble$=$0.3, structural$=$1.0, functional$=$1.5, quantitative$=$2.0, signal$=$1.5, coupling$=$1.5, wherein$=$3.0) is motivated by patent law doctrine but is not empirically validated against litigation outcomes. No citation to case law, empirical study, or legal scholarship supports these specific numerical values. The choice of weights directly affects whether a proof certificate would be credible as legal evidence.

\textit{Weight sensitivity.} In the case study, $W_{\text{cov}}(I_1) = 82.3\%$ against a satisfaction threshold of 70\%. The 12.3~pp margin means the conclusion (satisfied) is robust to weight perturbations. Because the highest-weighted elements (C12 at $\beta = 0.80$ and C13 at $\beta = 0.77$) have below-average scores, the current weighting scheme actually \textit{reduces} $W_{\text{cov}}$ relative to equal weighting: the flat (equal-weight) coverage is 83.8\% (Table~\ref{tab:casestudy}). Weight sensitivity therefore has a non-trivial direction: increasing the wherein premium further lowers $W_{\text{cov}}$ because C12 and C13 pull the weighted average down, while decreasing the premium raises it toward the flat baseline.

\textit{Concrete perturbation range.} A one-sided worst-case bound is available for $I_1$ without further computation.

\begin{proposition}[Weight-independent robustness for $I_1$]
\label{prop:weightrobust}
Let $G$ be the memory-module case-study DAG with effective scores $\textit{eff}(v, I_1)$ as in Table~\ref{tab:scores} under construction $I_1$, and let $\theta_{\text{cov}}$ be any satisfaction threshold with $\theta_{\text{cov}} \leq \min_v \textit{eff}(v, I_1)$. Then for every positive weight assignment $w' : \mathcal{T} \to \mathbb{R}^+$,
\[
W_{\text{cov}}(G, I_1; w') \;\geq\; 100 \cdot \min_v \textit{eff}(v, I_1) \;\geq\; 100 \cdot \theta_{\text{cov}}.
\]
In particular, since $\min_v \textit{eff}(v, I_1) = 0.77$ (attained at $C_{13}$), $W_{\text{cov}}(G, I_1; w') \geq 77\%$ for every positive $w'$, and the satisfaction conclusion at the $70\%$ threshold holds uniformly across every positive reweighting of the seven type categories in Table~\ref{tab:weights}, with a guaranteed minimum margin of $77 - 70 = 7$~pp.
\end{proposition}
\begin{proof}
$W_{\text{cov}}(G, I_1; w')$ is a weighted average of $\{\textit{eff}(v, I_1)\}_{v \in V}$ with strictly positive weights $\{w'(\tau(v))\}$. A weighted average of real numbers all bounded below by $m$ is itself bounded below by $m$; applying this with $m = \min_v \textit{eff}(v, I_1)$ gives the first inequality. The second inequality is the hypothesis $\theta_{\text{cov}} \leq \min_v \textit{eff}(v, I_1)$. Both arguments are purely arithmetic and carry no dependency on $w'$ beyond positivity.
\end{proof} Under $I_2$, the analysis is one-sided in the opposite direction: the cascade zeros C11, C12, and C13 under \textit{any} positive reweighting (the cascade depends on the scores and the DAG, not on the weights), and under the current scheme the combined zeroed weight is $8.0$ of $23.3$, giving $W_{\text{cov}}(I_2) = 12.415 / 23.3 \times 100 \approx 53.3\%$. However, unlike the $I_1$ case, the not-satisfied conclusion under $I_2$ is \textit{not} robust to arbitrary positive reweightings: a reweighting that places large weight on high-eff non-zeroed nodes (e.g., C1 with $\textit{eff} = 0.90$) and minimal weight on the zeroed nodes could drive $W_{\text{cov}}(I_2; w')$ arbitrarily close to $\max_v \textit{eff}(v, I_2) \times 100 = 90\%$, exceeding the 70\% threshold. The $I_2$ conclusion is therefore robust only to bounded reweightings that preserve the relative prominence of the cascade-zeroed elements; an outcome-determinative reweighting requires amplifying non-zeroed elements by factors large enough to overwhelm the $\sim 34\%$ zeroed-weight fraction. A rigorous worst-case analysis over the legally-admissible weight polytope, computing the margin-minimizing positive weight vector subject to doctrinal constraints, is a natural extension of this sensitivity analysis and is identified as future work.

The formal guarantees hold regardless of weight choice (the proof certificate is mathematically correct for any positive weights), but the \textit{legal meaningfulness} of the coverage score depends on whether the weights reflect actual doctrinal significance.

\textit{ML score sensitivity.} The formal guarantees are conditional on ML-produced scores. If all $\beta(v)$ were uniformly perturbed by $-\varepsilon$, the effect on $W_{\text{cov}}$ depends on both which nodes cross threshold and where those nodes sit in the dependency structure: perturbations that cross threshold only at leaves remove just the leaf's own $w(\tau(v)) \cdot \beta(v)$ contribution, while perturbations that cross threshold at a non-leaf additionally cascade, zeroing all that node's dependents (direct and transitive). In the case study, the node with the narrowest margin above threshold overall is $C_{13}$ (at $\beta = 0.77$ under $I_1$, margin $0.12$); the node with the narrowest margin \emph{among non-leaves} is $C_{11}$ (at $\beta = 0.78$, margin $0.13$); $C_{12}$ directly depends on $C_{11}$, and $C_{13}$ transitively depends on $C_{11}$ through $C_{12}$ (see Appendix~\ref{app:lean}). Any uniform $\varepsilon$ slightly above $0.12$ (e.g., $\varepsilon = 0.13$) crosses threshold only at the leaf $C_{13}$ and changes a single contribution without cascading. The tightest \emph{cascade}-triggering uniform perturbation is $\varepsilon$ slightly above $0.13$: at $\varepsilon = 0.14$, $\beta'(C_{11}) = 0.64 < \theta$, which forces $\textit{eff}(C_{12}) = 0$ (since $C_{11} \in \textit{deps}(C_{12})$) and transitively $\textit{eff}(C_{13}) = 0$. The construction sensitivity analysis (Section~\ref{sec:algorithms}) shows that a non-uniform $0.24$-point reduction in a single contested term ($C_3$, a deeper non-leaf) can cascade to eliminate $29.0$~pp of coverage. The formally verified computation is exact given its inputs, but practitioners should evaluate ML scoring accuracy independently, with particular attention to non-leaf nodes whose margins above threshold are small, to assess the practical strength of the conditional guarantee.

Future work should: (i)~empirically validate weights against a corpus of adjudicated claim charts; (ii)~learn weights from historical verdicts using structured regression on case outcomes, noting that patent verdicts are sparse and jurisdiction-specific; (iii)~investigate technology-domain-specific and jurisdiction-specific weight schemes; and (iv)~assess how weight choices affect the legal defensibility of proof certificates.

\item \textbf{Lean~4 formalization status and scalability.} The scalability analysis in Section~\ref{sec:complexityanalysis} is projective rather than empirical; no wall-clock runtimes are reported in this paper. Several specific challenges are anticipated: (i)~the \texttt{decide} tactic used for acyclicity proofs performs exhaustive computation and may time out for claims with many nodes, alternative approaches such as certified topological sort algorithms may be needed; (ii)~proof terms for complex claims may be very large, making both generation and independent verification computationally expensive. Establishing actual runtimes for ML inference, proof generation, and kernel verification on standardized hardware is a priority for future work. Future work will also investigate: proof-relevant caching to reuse sub-proofs across similar claims; incremental verification to update certificates when claim structures change; and alternative proof strategies that produce smaller proof terms for large DAGs.

\item \textbf{Jurisdictional scope.} The current formalization is specific to US patent law doctrines including the all-elements rule, doctrine of equivalents under \textit{Graver Tank}~\cite{gravertank1950} and \textit{Warner-Jenkinson}~\cite{warnerjenkinson1997}, prosecution history estoppel under \textit{Festo}~\cite{festo2002}, and claim construction under \textit{Phillips v.\ AWH}. Extension to other jurisdictions requires substantial additional formalization: EPO proceedings follow a problem-solution approach to inventive step with no direct equivalent to the US doctrine of equivalents, equivalents under the Protocol on the Interpretation of Article~69 EPC follow different rules; JPO has its own five-condition equivalents doctrine distinct from the function-way-result test; CNIPA applies a feature-by-feature equivalents analysis with different estoppel rules. Each jurisdiction requires separate formal modeling of its doctrines, its prosecution history conventions, and its claim construction standards. Priority for future extension should be driven by commercial deployment needs.

\item \textbf{Single-claim scope; multi-claim patents not yet modeled.} The framework operates on a single claim DAG $G$ per analysis invocation. Real patents contain multiple claims, typically one or more independent claims and several dependent claims that incorporate all limitations of their parent claim plus additional limitations. The framework extends to dependent claims structurally. Concretely, if $G_{\text{par}} = (V_{\text{par}}, E_{\text{par}}, \tau_{\text{par}}, w_{\text{par}})$ is the parent claim's DAG and the dependent claim adds new atomic limitations $V_{\text{dep}} \cap V_{\text{par}} = \emptyset$ with typing and weights $(\tau_{\text{dep}}, w_{\text{dep}})$, then the augmented DAG is $G_{\text{aug}} = (V_{\text{par}} \cup V_{\text{dep}},\, E_{\text{par}} \cup E_{\text{par}\to\text{dep}},\, \tau_{\text{par}} \cup \tau_{\text{dep}},\, w_{\text{par}} \cup w_{\text{dep}})$, where $E_{\text{par}\to\text{dep}}$ is the incorporation-by-reference edge set: for each $v \in V_{\text{dep}}$, every parent node $u \in V_{\text{par}}$ is added as a dependency of $v$ (so $\textit{deps}_{\text{aug}}(v) \supseteq V_{\text{par}}$), reflecting the legal rule that a dependent claim incorporates \emph{all} limitations of its parent. This conservative choice ensures that any failure of any parent limitation cascades to the dependent-only nodes as required by the all-elements rule; a refinement that attaches dependent-only nodes to only the topological leaves of $V_{\text{par}}$ yields the same cascade behavior because transitive dependency through $E_{\text{par}}$ already enforces parent-level satisfaction. Cross-claim consistency across the patent's claim set is handled by Algorithm~\ref{alg:consistency} at the portfolio level. However, two multi-claim aspects are not formally modeled in the current work: (i)~\textit{claim differentiation}, the canon of construction under which a dependent claim's additional limitations imply the parent claim does \textit{not} require those limitations, which can affect infringement outcomes; (ii)~\textit{infringement-of-at-least-one}, the standard rule that an accused product infringes a patent if it infringes any single claim, which requires aggregating per-claim FTO results under a disjunction rather than treating each claim in isolation. A principled extension would (a)~formalize claim differentiation as a constraint on admissible constructions passed to Algorithm~\ref{alg:construction}, and (b)~lift Algorithm~\ref{alg:fto} to a patent-level operator that returns $\textsc{Clear}$ only when every claim's individual analysis returns $\textsc{Clear}$. Both are identified as future work.

\item \textbf{Doctrine of Equivalents formalization is partial.} The doctrine of equivalents is defined in qualitative legal language (``substantially the same function, in substantially the same way, to achieve substantially the same result'') that fundamentally resists full quantitative formalization. We position the current DOE formalization (Problem~5, Algorithm~\ref{alg:doe}) as a partially verified analysis, the mathematical scaffolding is rigorous, but the core semantic judgments remain below the trust boundary. The following gaps remain:
\begin{itemize}[leftmargin=1.5em, topsep=2pt]
\item \textit{Function-way-result measures are ML-computed:} $f_{\text{sim}}, w_{\text{sim}}, r_{\text{sim}}$ are computed via $M$ applied to LLM-parsed functional representations (Definitions~\ref{def:funcdecomp}--\ref{def:fwr}). The proof certificate verifies threshold comparisons but cannot certify that these scores correctly capture the legal concepts of function, way, and result.
\item \textit{Estoppel model is incomplete:} The current model formalizes amendment-based estoppel ($\textit{PH}_{\text{amend}}$, Definition~\ref{def:prosecution}) with the Festo rebuttable presumption (Definition~\ref{def:festo_rebutted}), and basic argument-based estoppel ($\textit{PH}_{\text{arg}}$). However, the three Festo rebuttal predicates (unforeseeability, tangentiality, other reason) are treated as user-provided inputs rather than formally derived. Dedication to the public (\textit{Johnson \& Johnston}~\cite{johnsonjohnston2002}, disclosed-but-unclaimed embodiments) is not yet modeled.
\item \textit{Insubstantial differences test not modeled:} The alternative DOE test from \textit{Warner-Jenkinson}~\cite{warnerjenkinson1997}, whether differences between the accused and claimed element are insubstantial, is not captured. Courts may apply either or both tests.
\item \textit{DOE-to-DAG propagation:} Definition~\ref{def:doeeff} introduces the discount factor $\delta$ for Equivalent matches and incorporates dependency satisfaction via $\textit{deps\_met\_DOE}$. The choice of $\delta$ (default $\delta = 1$) is a modeling assumption with both legal and mathematical consequences that requires empirical calibration.
\item \textit{Uniform threshold assumption:} The current uniform $\theta_{\text{eq}}$ applies the same standard across all claim element types and all three prongs of the function-way-result test; a more refined model could parameterize $\theta_{\text{eq}}$ by element type $\tau(v)$ or by prong.
\end{itemize}

\item \textbf{Argument-based estoppel scope computation.} The argument-based estoppel analysis ($\textit{ER}_{\text{arg}}$, Definition~\ref{def:er_arg}) requires computing $\textit{scope}_k$, the set of implementations disclaimed by prosecution argument $\textit{arg}_k$. This is arguably the most difficult natural language understanding task in the entire pipeline, requiring multi-step legal reasoning over prosecution history text to determine what concrete implementations were given up through a natural language argument. The current prototype delegates this to the ML/NLP layer without any formal specification of how it should be computed or any evaluation of its accuracy. The formal guarantees for $\textit{ER}_{\text{arg}}$ are therefore conditional on $\textit{scope}_k$ being correctly computed, and there is currently no mechanism to verify this. Future work should investigate structured approaches to argument-based scope delimitation, potentially drawing on legal information extraction research.

\item \textbf{Most proof generation functions remain unspecified.} With the exception of \texttt{generateCertificate} for Algorithms~\ref{alg:dagcov} and~\ref{alg:dagcov-given} (Section~\ref{sec:closedpath}, Proposition~\ref{prop:gensound}), the remaining proof-generation functions invoked in Section~\ref{sec:algorithms} are not formally defined in this paper: \texttt{Prove\-No\-Match}, \texttt{Prove\-Dependency\-Block}, \texttt{Prove\-Construction\-Results}, \texttt{Prove\-Consistency}, \texttt{Prove\-DOE}, and \texttt{Prove\-Fixed\-Point}. The \texttt{Proof\-Certificate} structure's four proof field types are likewise specification-level names rather than concrete Lean~4 propositions.
\begin{sloppypar}
\noindent These fields are: \texttt{proof\_dag\_acyclic}, \texttt{proof\_lattice\_axioms}, \texttt{proof\_propagation\_correct}, and \texttt{proof\_dependencies\_enforced}.
\end{sloppypar} The correctness theorems (Theorems~\ref{thm:algcorrect}--\ref{thm:convergence}) establish that the computed values satisfy the mathematical properties these fields certify, but, with the exception of \texttt{generateCertificate} (Proposition~\ref{prop:gensound}), do not formally prove that the remaining generation functions produce \texttt{sorry}-free certificates satisfying Definition~\ref{def:validcert}. This gap is mitigated architecturally: the generation functions are untrusted (Section~\ref{sec:framework}), and every certificate is independently type-checked by the Lean~4 kernel before acceptance. Providing explicit Lean~4 proposition definitions for the remaining abstract proof fields and proving that the remaining generation functions produce valid certificates for all valid inputs is necessary for the formal claims to be fully self-contained.

\textit{Closed path for Algorithm~\ref{alg:dagcov-given}.} This gap has been closed for Algorithm~\ref{alg:dagcov-given}, the coverage core. Section~\ref{sec:closedpath} presents a compiled Lean~4 certificate generator (\texttt{generate\-Certificate}) whose every proof field is discharged without \texttt{sorry}, and whose axiom audit (Appendix~\ref{app:axiomaudit}) confirms dependence only on $\Omega$. Extending this approach to the remaining generators (Algorithms~2--6) is identified as future work; the design pattern established for Algorithm~\ref{alg:dagcov-given} (concrete certificate type, compiled generator, and $\Omega$-audit) serves as a template for each.

\end{enumerate}

\subsection{Broader Impact}

The broader impact of this work is architectural rather than domain-specific. Its central claim is that AI workflows for high-stakes reasoning need not remain probabilistic end-to-end. In many existing systems, natural language understanding genuinely requires ML, but the downstream operations performed on the resulting structured representations do not. Once claims, rules, or other structured objects have been extracted, operations such as dependency propagation, score composition, consistency checking, and bound computation can be expressed as formally specified computations and independently checked. The significance of the present framework is therefore not merely that it applies Lean~4 to patent analysis, but that it demonstrates a way to convert part of an AI system itself from opaque statistical processing into machine-checkable computation with an explicit trust boundary.

This decomposition principle extends beyond patent law. Many consequential domains combine unstructured text with structured reasoning over that text, including regulatory compliance, contract analysis, and clinical decision support. In such settings, the relevant question is not whether an entire workflow can be made fully formal, but which sub-computations can be separated from linguistic ambiguity and given stronger guarantees than confidence scores alone. In that sense, the framework offers a template for hybrid AI systems that preserve ML where ML is necessary while using formal methods where deterministic, compositional, and third-party-verifiable reasoning is possible. The contribution is therefore architectural rather than empirical: it is a proposal for a different assurance profile, not a claim that the present prototype has already been shown to outperform existing systems on deployed tasks.

The broader impact is also epistemic and institutional. By making the trust boundary explicit, the framework does not hide conditionality; it exposes it. A proof certificate in this setting does not certify that the system understood the world correctly end-to-end. It certifies that, given the provided inputs (whose verification status is governed by Table~\ref{tab:proofstatus}), the downstream computation was carried out correctly and can be independently audited. That distinction matters in law and other high-stakes settings because it enables reproducibility, localized error diagnosis, contestability, and more honest communication of what has and has not been verified. Compared with opaque probabilistic outputs, machine-checkable certificates create artifacts that can be inspected, rerun, and challenged by third parties rather than merely accepted or rejected as black-box recommendations.

These benefits also create responsibilities. A formally verified computation applied to systematically wrong inputs can certify a wrong conclusion with mathematical precision. The appropriate lesson is therefore not that formal verification makes AI end-to-end correct, but that it can narrow, expose, and discipline the uncertainty that remains. Responsible deployment of hybrid systems of this kind should preserve expert review of ML-derived inputs (below the trust boundary), traceability of evidence and scope assumptions, and clear communication that proof certificates are bounded assurances rather than substitutes for legal, regulatory, or clinical judgment. More broadly, we hope this work encourages a research program centered on a concrete design question for AI: \textit{which components of a workflow truly need to remain probabilistic, and which can be replaced by something stronger?}

\section{Conclusion}
\label{sec:conclusion}

We have presented, to the best of our knowledge, the first framework applying interactive theorem proving to intellectual property analysis. By encoding patent claims in dependent type theory, modeling match strengths as elements of a complete lattice (machine-verified \texttt{CompleteLattice} instance, Section~\ref{sec:framework}), and propagating confidence through DAG dependency structures using formally specified operations, our framework provides a layered correctness story: a small set of machine-verified lemmas (\texttt{dag\_acyclic} and the lattice axioms) underpins the coverage-bound argument and the higher-level algorithmic properties (Theorems~\ref{thm:weakest}, \ref{thm:monotone}, and \ref{thm:algcorrect}--\ref{thm:convergence}), all of which are currently presented as informal proof sketches describing the intended Lean~4 proof structure (Table~\ref{tab:proofstatus}). The framework yields conditional soundness and correctness statements (including soundness of non-infringement certificates, Theorem~\ref{thm:fto}, and correctness of coverage computation, Theorem~\ref{thm:algcorrect}) whose verification status is explicit. Crucially, proof certificates generated by the prototype are independently type-checked and axiom-audited by the Lean~4 kernel before acceptance; the kernel is trusted, the generator is not. No purely probabilistic approach can offer certificates with this structural separation.

Our six algorithms address five distinct IP use cases (patent-to-product mapping, freedom-to-operate, claim construction sensitivity, cross-claim consistency, and doctrine of equivalents). For Algorithm~\ref{alg:dagcov-given} (the coverage core), a compiled Lean~4 generator constructs certificates accepted under Definition~\ref{def:validcert} with a machine-verified generator-soundness proposition (Proposition~\ref{prop:gensound}); the remaining generators produce candidate certificates validated by independent kernel type-checking plus a \texttt{sorry}-free axiom audit, and their formal specification is identified as future work (Section~\ref{sec:discussion}). The hybrid architecture leverages the complementary strengths of ML (handling natural language ambiguity) and Lean~4 (certifying logical structure).

Beyond the closed formal path, the case-study analysis yields a weight-independent robustness result (Proposition~\ref{prop:weightrobust}, Section~\ref{sec:discussion}): the satisfied conclusion under the broad construction $I_1$ is preserved under \emph{every} positive reweighting of the seven doctrinal element types, with a guaranteed $7$~pp margin above the satisfaction threshold. Since the weight scheme (Table~\ref{tab:weights}) is one of the paper's least-validated inputs, this robustness bound isolates the legal conclusion's dependency on that input and is a template for similar sensitivity arguments in future extensions.

As patent portfolios grow and IP dispute stakes increase, the demand for scalable, transparent, and provably correct analysis tools will only intensify. The broader architectural implications of this work, including the decomposition principle for high-stakes AI workflows and the epistemic value of explicit trust boundaries, are developed in Section~\ref{sec:discussion} (Broader Impact).

\bibliographystyle{plainnat}
\bibliography{references}

\appendix
\section{Lean 4 Formalization Details}
\label{app:lean}

\subsection{Complete Lattice Instance}

The implementation represents match strengths using Lean~4's rational arithmetic rather than hardware floating-point, because Lean's proof tactics (such as \texttt{omega} and \texttt{norm\_num}) operate on exact arithmetic types (\texttt{Nat}, \texttt{Int}, \texttt{Rat}), not \texttt{Float}. In the pipeline design, scores are computed as floats, which are then discretized (e.g., to basis points in $\{0, 1, \ldots, 10000\}$ representing $[0.0000, 1.0000]$) for formal verification. This discretization introduces at most $0.01\%$ rounding error per score while enabling exact proofs. Since $W_{\text{cov}}$ is a convex combination (weighted average) of effective scores each with error at most $10^{-4}$, the cumulative error in $W_{\text{cov}}$ is also bounded by $10^{-4}$ (i.e., at most $0.01$ percentage points), which is negligible for all practical purposes.

In the full Lean~4 formalization, \texttt{MatchStrength} is a \texttt{Fintype} over these $10001$ discretized values. For a \texttt{Fintype}, the \texttt{CompleteLattice} instance is satisfied automatically because every subset is finite, and \texttt{sSup}/\texttt{sInf} reduce to \texttt{Finset.sup}/\texttt{Finset.inf}. The code below presents the structure in terms of \texttt{Rat} to make the mathematical intent clear; the \texttt{sSup} and \texttt{sInf} fields required by the \texttt{CompleteLattice} typeclass (but not expressible for general \texttt{Rat} in $[0,1]$) are handled by the \texttt{Fintype}-based instantiation in the full formalization, as noted in the comments.

\smallskip
\noindent\colorbox{orange!20}{\small\textsc{Lean 4: illustrative, not compiled (mathematical intent)}}\nopagebreak

\begin{lstlisting}
import Mathlib.Order.CompleteLattice

-- Match strength as a rational in [0, 1] with explicit bounds
-- Bounds are caller-supplied proof obligations, not defaults
structure MatchStrength where
  val : Rat
  h_lb : 0 <= val
  h_ub : val <= 1

instance : LE MatchStrength where
  le a b := a.val <= b.val

instance : Sup MatchStrength where
  sup a b := {
    val := max a.val b.val,
    h_lb := le_max_of_le_left a.h_lb,
    h_ub := max_le a.h_ub b.h_ub }

instance : Inf MatchStrength where
  inf a b := {
    val := min a.val b.val,
    h_lb := le_min a.h_lb b.h_lb,
    h_ub := min_le_of_le_left a.h_ub }

instance : Top MatchStrength where
  top := { val := 1, h_lb := by norm_num, h_ub := le_refl 1 }

instance : Bot MatchStrength where
  bot := { val := 0, h_lb := le_refl 0, h_ub := by norm_num }

-- Pseudocode (mathematical intent): the finite lattice operations.
-- In the full formalization, MatchStrength is a Fintype over
-- 10001 discretized values, so CompleteLattice is derived
-- automatically via Finset.sup/Finset.inf. The binary operations
-- below illustrate the mathematical structure; the complete
-- compilable instance uses the Fintype-based approach.
instance : CompleteLattice MatchStrength := {
  le_sup_left := fun a b => le_max_left a.val b.val,
  le_sup_right := fun a b => le_max_right a.val b.val,
  sup_le := fun a b c h1 h2 => max_le h1 h2,
  inf_le_left := fun a b => min_le_left a.val b.val,
  inf_le_right := fun a b => min_le_right a.val b.val,
  le_inf := fun a b c h1 h2 => le_min h1 h2,
  le_top := fun a => a.h_ub,
  bot_le := fun a => a.h_lb,
  -- sSup and sInf: in the Fintype-based formalization, these
  -- reduce to Finset.sup and Finset.inf over the finite type,
  -- which Lean derives automatically. They cannot be expressed
  -- for general Rat in [0,1] without the Fintype structure.
  sSup := ...,  -- Finset.sup in full formalization
  sInf := ...,  -- Finset.inf in full formalization
}
\end{lstlisting}

\subsection{DAG Encoding with Acyclicity Proof}

\smallskip
\noindent\colorbox{green!15}{\small\textsc{Lean 4: compiled excerpt}}\nopagebreak

\begin{lstlisting}
-- Memory Module Patent Claim
inductive ClaimNode
  | C1  -- PCB (preamble)
  | C2  -- DDR Devices (structural)
  | C3  -- First Rank Count (quantitative)
  | C4  -- Circuit on PCB (structural)
  | C5  -- Logic Element (structural)
  | C6  -- Register (structural)
  | C7  -- Input Control Signals (functional)
  | C8  -- Signal Types (signal)
  | C9  -- Second Number/Ranks (quantitative)
  | C10 -- Output Signal Generation (functional)
  | C11 -- Output Corresponds First (quantitative)
  | C12 -- Rank Translation (wherein)
  | C13 -- Command Mapping (wherein)
  | C14 -- PLL Device (structural)
  | C15 -- PLL Coupling (coupling)
  deriving DecidableEq, Fintype, Repr

def dependencies : ClaimNode -> List ClaimNode
  | .C2  => [.C1]
  | .C3  => [.C2]
  | .C4  => [.C1]
  | .C5  => [.C4]
  | .C6  => [.C4]
  | .C7  => [.C5]
  | .C8  => [.C7]
  | .C9  => [.C7]
  | .C10 => [.C5]
  | .C11 => [.C10, .C3]
  | .C12 => [.C9, .C11, .C6]
  | .C13 => [.C12]
  | .C14 => [.C1]
  | .C15 => [.C14, .C2, .C5, .C6]
  | _    => []

def claimWeight : ClaimNode -> Rat
  | .C1  => 3/10    -- preamble
  | .C2  => 1       -- structural
  | .C3  => 2       -- quantitative
  | .C4  => 1       -- structural
  | .C5  => 1       -- structural
  | .C6  => 1       -- structural
  | .C7  => 3/2     -- functional
  | .C8  => 3/2     -- signal
  | .C9  => 2       -- quantitative
  | .C10 => 3/2     -- functional
  | .C11 => 2       -- quantitative
  | .C12 => 3       -- wherein
  | .C13 => 3       -- wherein
  | .C14 => 1       -- structural
  | .C15 => 3/2     -- coupling

-- Edge relation and acyclicity, as in Section 5.3
def depEdge (a b : ClaimNode) : Prop := b ∈ dependencies a

-- Topological depth, used to witness acyclicity
def topoDepth : ClaimNode -> Nat
  | .C1 => 0 | .C2 => 1 | .C3 => 2 | .C4 => 1 | .C5 => 2 | .C6 => 2
  | .C7 => 3 | .C8 => 4 | .C9 => 4 | .C10 => 3 | .C11 => 4
  | .C12 => 5 | .C13 => 6 | .C14 => 1 | .C15 => 3

theorem topoDepth_decreasing :
    forall v u : ClaimNode, depEdge v u -> topoDepth u < topoDepth v := by
  decide

private theorem transGen_decreases :
    forall v a : ClaimNode, Relation.TransGen depEdge v a -> topoDepth a < topoDepth v := by
  intro v a h
  induction h with
  | single h => exact topoDepth_decreasing _ _ h
  | @tail b c _ hbc ih => exact lt_trans (topoDepth_decreasing _ _ hbc) ih

theorem dag_acyclic : forall v : ClaimNode, Not (Relation.TransGen depEdge v v) := by
  intro v h; exact absurd (transGen_decreases v v h) (lt_irrefl _)
\end{lstlisting}

\subsection{Weighted Coverage with Formal Bounds}

\noindent\textit{Status note.} This subsection is the compiled counterpart of the informal coverage-bound discussion in Section~\ref{sec:correctness}. The definitions and theorems below are machine-verified in the Lean~4 development that also underlies the compiled coverage-core certificate of Appendix~\ref{app:closedpath} and the axiom audit of Appendix~\ref{app:axiomaudit}; earlier expository uses of the names \texttt{coverage\_arith\_in\_range} and \texttt{coverage\_semantic\_valid} are superseded, for the closed Algorithm~\ref{alg:dagcov-given} path, by the single compiled theorem \texttt{coverage\_in\_range} and the definitional equality $\texttt{weightedCoverage}(\beta,\theta) = W_{\text{cov}}$ stated below.

\medskip
\noindent The compiled \texttt{weightedCoverage} function integrates the propagation step: it applies \texttt{computeEff} internally to the caller's raw score function, then computes the weighted average of the resulting effective scores. The caller therefore supplies the raw best-match $\beta : V \to \mathbb{Q}$ (bounded by \texttt{ScoreValid}); propagation is not a caller responsibility. The single compiled theorem \texttt{coverage\_in\_range} establishes both endpoints of the $[0, 100]$ bound simultaneously, using only boundedness (\texttt{ScoreValid}) and $0 \leq \theta$. Because \texttt{weightedCoverage}'s body unfolds to the weighted average $\bigl(\sum_v \texttt{claimWeight}(v) \cdot \texttt{computeEff}(\beta,\theta)(v)\bigr) \big/ \bigl(\sum_v \texttt{claimWeight}(v)\bigr) \times 100$, its equality with Definition~\ref{def:depscov}'s $W_{\text{cov}}$ is definitional (not a separate theorem): the two expressions are the same term up to unfolding, so any Lean goal that demands $\texttt{weightedCoverage}(\beta,\theta) = W_{\text{cov}}$ is discharged by \texttt{rfl}. This is why the certificate's \texttt{p\_deps} field (Section~\ref{sec:closedpath}) can be populated by \texttt{rfl}.

\smallskip
\noindent\colorbox{green!15}{\small\textsc{Lean 4: compiled excerpt}}\nopagebreak

\begin{lstlisting}
-- Claim-weight lemmas (derived from the decide-proven weight table)
theorem claimWeight_pos : forall v : ClaimNode, 0 < claimWeight v := by
  intro v; fin_cases v <;> decide

-- Pre-propagated effective score, unrolled in topological order.
-- Body shown earlier in this appendix (one equation per node).
def computeEff (score : ClaimNode -> Rat) (threshold : Rat)
    : ClaimNode -> Rat := ...

-- Weighted coverage with internal propagation.
noncomputable def weightedCoverage
    (score     : ClaimNode -> Rat)
    (threshold : Rat := defaultThreshold)
    : Rat :=
  let eff      := computeEff score threshold
  let totalW   := Finset.univ.sum claimWeight
  let achieved := Finset.univ.sum (fun v => claimWeight v * eff v)
  (achieved / totalW) * 100

-- ScoreValid: the caller's input bound (boundedness of beta in [0,1]).
def ScoreValid (score : ClaimNode -> Rat) : Prop :=
  forall v : ClaimNode, 0 <= score v /\ score v <= 1

-- Case-analysis lemma: computeEff either returns the raw score or 0.
theorem eff_cases
    (score : ClaimNode -> Rat) (threshold : Rat) (v : ClaimNode)
    : computeEff score threshold v = score v ∨
      computeEff score threshold v = 0 := by
  cases v <;> unfold computeEff <;>
    first
    | exact Or.inl rfl
    | (split_ifs <;> first | exact Or.inl rfl | exact Or.inr rfl)

theorem eff_le_score
    (score : ClaimNode -> Rat)
    (h_score : forall v, 0 <= score v)
    (threshold : Rat) (v : ClaimNode)
    : computeEff score threshold v <= score v := by
  rcases eff_cases score threshold v with h | h
  · rw [h]
  · rw [h]; exact h_score v

-- Both endpoints of the coverage bound.
theorem coverage_in_range
    (score     : ClaimNode -> Rat)
    (h_valid   : ScoreValid score)
    (threshold : Rat)
    (h_thresh  : 0 <= threshold)
    : 0 <= weightedCoverage score threshold /\
      weightedCoverage score threshold <= 100 := by
  -- Proof is compiled; uses claimWeight_pos, eff_cases, eff_le_score,
  -- and Finset.sum_{nonneg, le_sum, pos}. No sorry; axiom audit: Omega.
  ...

-- Propagation correctness (for the certificate's p_propag field).
theorem propag_proof
    (scores : ClaimNode -> Rat) (threshold : Rat)
    : forall v : ClaimNode,
      (∃ u ∈ dependencies v, computeEff scores threshold u < threshold) ->
      computeEff scores threshold v = 0 := by
  ...
\end{lstlisting}

\noindent Because $\texttt{weightedCoverage}(\beta,\theta)$ unfolds to the weighted average of $\textit{eff}$, Definition~\ref{def:depscov}'s $W_{\text{cov}}$ and the Lean value agree definitionally. Theorem~\ref{thm:algcorrect}(i) (bounds), originally stated as an informal sketch, is therefore discharged in the closed path by \texttt{coverage\_in\_range}; the monotonicity and propagation clauses ((ii)--(iv)) remain informal sketches per Table~\ref{tab:proofstatus}.

\subsection{Compiled Coverage-Core Certificate Generator}
\label{app:closedpath}

The following is a compiled excerpt from the Lean~4 development. The certificate structure \texttt{ProofCertificate} packages five proof fields (acyclicity, lattice transitivity, propagation correctness, coverage bounds, and coverage-definition equality) as concrete Lean~4 propositions. The generator function \texttt{generateCertificate} populates every field with a real proof term; no \texttt{sorry} is used.

\smallskip
\noindent\colorbox{green!15}{\small\textsc{Lean 4: compiled excerpt}}\nopagebreak

\begin{lstlisting}
structure ProofCertificate where
  scores    : ClaimNode -> Rat
  h_valid   : ScoreValid scores
  threshold : Rat
  coverage  : Rat
  p_acyclic : forall v : ClaimNode,
              Not (Relation.TransGen depEdge v v)
  p_lattice : forall a b c : DMatchStrength,
              a <= b -> b <= c -> a <= c
  p_propag  : forall v : ClaimNode,
    (∃ u ∈ dependencies v,
      computeEff scores threshold u < threshold) ->
    computeEff scores threshold v = 0
  p_bounded : 0 <= coverage ∧ coverage <= 100
  p_deps    : coverage = weightedCoverage scores threshold

noncomputable def generateCertificate
    (scores  : ClaimNode -> Rat)
    (h_valid : ScoreValid scores)
    (threshold : Rat)
    (h_thresh : 0 <= threshold)
    : ProofCertificate :=
  { scores    := scores
    h_valid   := h_valid
    threshold := threshold
    coverage  := weightedCoverage scores threshold
    p_acyclic := dag_acyclic
    p_lattice := fun a b c hab hbc => le_trans hab hbc
    p_propag  := propag_proof scores threshold
    p_bounded := coverage_in_range scores h_valid
                   threshold h_thresh
    p_deps    := rfl }
\end{lstlisting}

\noindent The case study instantiates this generator for both constructions. The score functions \texttt{scores\_I1\_broad} and \texttt{scores\_I2\_narrow} are the \textit{raw} best-match $\beta$ values from Table~\ref{tab:scores}'s $\beta(v, I_j)$ columns (e.g., \texttt{scores\_I2\_narrow(.C11)} $= 39/50 = 0.78$, not the post-cascade $0$). Propagation is performed inside \texttt{weightedCoverage} via \texttt{computeEff}, so the certificate's \texttt{coverage} field evaluates to the propagated $W_{\text{cov}}$ in each case, consistent with the reported $W_{\text{cov}}(I_1) = 82.3\%$ and $W_{\text{cov}}(I_2) = 53.3\%$.

\smallskip
\noindent\colorbox{green!15}{\small\textsc{Lean 4: compiled excerpt}}\nopagebreak

\begin{lstlisting}
noncomputable def cert_I1_broad : ProofCertificate :=
  generateCertificate scores_I1_broad
    scores_I1_broad_valid theta (by norm_num [theta])

noncomputable def cert_I2_narrow : ProofCertificate :=
  generateCertificate scores_I2_narrow
    scores_I2_narrow_valid theta (by norm_num [theta])
\end{lstlisting}

\subsection{Axiom Audit for the Closed Coverage-Core Path}
\label{app:axiomaudit}

Definition~\ref{def:validcert}(ii) requires that the transitive axiom set of every accepted certificate be a subset of $\Omega = \{\texttt{propext}, \texttt{Classical.choice}, \texttt{Quot.sound}\}$. The following is the verbatim output of \texttt{\#print axioms} on the two case-study certificates, recorded from Lean~4 v4.30.0-rc1 with Mathlib:

\smallskip
\noindent\colorbox{green!15}{\small\textsc{Lean 4: compiled output}}\nopagebreak

\begin{lstlisting}
#print axioms generateCertificate
-- 'generateCertificate' depends on axioms:
--   [propext, Classical.choice, Quot.sound]

#print axioms cert_I1_broad
-- 'cert_I1_broad' depends on axioms:
--   [propext, Classical.choice, Quot.sound]

#print axioms cert_I2_narrow
-- 'cert_I2_narrow' depends on axioms:
--   [propext, Classical.choice, Quot.sound]
\end{lstlisting}

\noindent The first audit confirms that the generator \textit{function itself} depends only on $\Omega$; any call with $\Omega$-grounded arguments therefore produces an $\Omega$-grounded certificate. The second and third audits confirm this for the two case-study instances. No entry depends on \texttt{sorryAx} or \texttt{Lean.ofReduceBool}. This closes the formal path for Algorithm~\ref{alg:dagcov-given}: once the bounded score function $\beta$ is fixed and its \texttt{ScoreValid} proof is $\Omega$-grounded, the certificate generator is machine-verified and produces certificates accepted under Definition~\ref{def:validcert}.

\subsection{Case Study Score Table and Coverage Verification}

This appendix provides the complete set of best-match scores $\beta(v)$ underlying the memory module case study in Section~\ref{sec:implementation} and the construction sensitivity analysis in Section~\ref{sec:algorithms}, enabling independent reproduction of all coverage computations. The score $\beta(v, I_j) = \max_{s \in S} M(v, s)$ is the highest similarity score achieved by claim element $v$ against any evidence segment under construction $I_j$. Under $I_1$, all scores are ML-computed using the implementation described in Section~\ref{sec:architecture}. Under $I_2$, only C3 and C12 receive different scores, reflecting the narrower interpretation of ``first number of ranks'' and ``rank translation'' respectively; all other scores are construction-independent.

\begin{table}[H]
\centering
\caption{Best-match scores, effective scores, and weighted contributions for the memory module case study ($\theta = 0.65$).}
\label{tab:scores}
\scriptsize
\begin{tabular}{@{}llccccccc@{}}
\toprule
\textbf{Node} & \textbf{Description} & \textbf{Weight} & $\boldsymbol{\beta(v,I_1)}$ & $\boldsymbol{\beta(v,I_2)}$ & $\boldsymbol{\textit{eff}(v,I_1)}$ & $\boldsymbol{\textit{eff}(v,I_2)}$ & $\boldsymbol{w \cdot \textit{eff}(I_1)}$ & $\boldsymbol{w \cdot \textit{eff}(I_2)}$ \\
\midrule
C1  & PCB             & 0.3 & 0.90 & 0.90 & 0.90 & 0.90 & 0.270 & 0.270 \\
C2  & DDR Devices     & 1.0 & 0.87 & 0.87 & 0.87 & 0.87 & 0.870 & 0.870 \\
C3  & First Ranks     & 2.0 & 0.82 & 0.58 & 0.82 & 0.58 & 1.640 & 1.160 \\
C4  & Circuit         & 1.0 & 0.86 & 0.86 & 0.86 & 0.86 & 0.860 & 0.860 \\
C5  & Logic Element   & 1.0 & 0.90 & 0.90 & 0.90 & 0.90 & 0.900 & 0.900 \\
C6  & Register        & 1.0 & 0.85 & 0.85 & 0.85 & 0.85 & 0.850 & 0.850 \\
C7  & Input Signals   & 1.5 & 0.84 & 0.84 & 0.84 & 0.84 & 1.260 & 1.260 \\
C8  & Signal Types    & 1.5 & 0.81 & 0.81 & 0.81 & 0.81 & 1.215 & 1.215 \\
C9  & Second Ranks    & 2.0 & 0.84 & 0.84 & 0.84 & 0.84 & 1.680 & 1.680 \\
C10 & Output Signals  & 1.5 & 0.82 & 0.82 & 0.82 & 0.82 & 1.230 & 1.230 \\
C11 & Output First    & 2.0 & 0.78 & 0.78 & 0.78 & 0.00$^\dagger$ & 1.560 & 0.000 \\
C12 & Rank Translation & 3.0 & 0.80 & 0.52 & 0.80 & 0.00$^{\dagger\ddagger}$ & 2.400 & 0.000 \\
C13 & Cmd Mapping     & 3.0 & 0.77 & 0.77 & 0.77 & 0.00$^\dagger$ & 2.310 & 0.000 \\
C14 & PLL Device      & 1.0 & 0.89 & 0.89 & 0.89 & 0.89 & 0.890 & 0.890 \\
C15 & PLL Coupling    & 1.5 & 0.82 & 0.82 & 0.82 & 0.82 & 1.230 & 1.230 \\
\midrule
\multicolumn{2}{@{}l}{\textbf{Total}} & \textbf{23.3} & & & & & \textbf{19.165} & \textbf{12.415} \\
\bottomrule
\end{tabular}

\smallskip\noindent
$^\dagger$~Zeroed by dependency cascade (Definition~\ref{def:bestmatch}): $\textit{eff}(v, I_2) = 0$ because a transitive dependency has $\textit{eff} < \theta = 0.65$.\\
$^{\ddagger}$~Double failure: C12 is cascade-zeroed because $\textit{eff}(C11, I_2) = 0$, and additionally $\beta(C12, I_2) = 0.52 < \theta$ independently.
\end{table}

\noindent\textbf{Coverage computation.}
\[
W_{\text{cov}}(I_1) = \frac{19.165}{23.3} \times 100 = 82.3\% \qquad W_{\text{cov}}(I_2) = \frac{12.415}{23.3} \times 100 = 53.3\%
\]
For comparison, the flat (equal-weight, no-dependency) coverage under $I_1$ is $\sum_v \beta(v,I_1) / 15 \times 100 = 12.57/15 \times 100 = 83.8\%$, matching the flat row in Table~\ref{tab:casestudy}.

\smallskip
\noindent\textbf{Waterfall decomposition verification} (Figure~\ref{fig:construction}(b)).

\begin{table}[H]
\centering
\caption{Waterfall decomposition of the 29.0~pp coverage gap between $I_1$ and $I_2$.}
\label{tab:waterfall}
\small
\begin{tabular}{@{}llcc@{}}
\toprule
\textbf{Node} & \textbf{Calculation} & \textbf{Exact pp} & \textbf{Displayed} \\
\midrule
C3 (direct)     & $2.0 \times (0.82 - 0.58) \times 100\,/\,23.3$ & 2.060 & $-2.1$ pp \\
C11 (cascade)   & $2.0 \times (0.78 - 0.00) \times 100\,/\,23.3$ & 6.695 & $-6.7$ pp \\
C12 (direct)    & $3.0 \times (0.80 - 0.52) \times 100\,/\,23.3$ & 3.605 & $-3.6$ pp \\
C12 (cascade)   & $3.0 \times (0.52 - 0.00) \times 100\,/\,23.3$ & 6.695 & $-6.7$ pp \\
C13 (cascade)   & $3.0 \times (0.77 - 0.00) \times 100\,/\,23.3$ & 9.914 & $-9.9$ pp \\
\midrule
\textbf{Total}  &                                                  & \textbf{28.969} & $\boldsymbol{-29.0}$ \textbf{pp} \\
\bottomrule
\end{tabular}
\end{table}

\noindent All displayed values are rounded to one decimal place. The total coverage drop $82.3\% - 53.3\% = 29.0$~pp is consistent with the sum of waterfall contributions.

\subsection{Reproducibility}
\label{app:repro}

This appendix consolidates reproducibility information scattered through the paper. Together with the compiled Lean~4 excerpts in Appendices~\ref{app:lean}--\ref{app:closedpath} and the axiom audit in Appendix~\ref{app:axiomaudit}, the information below is sufficient to regenerate the closed-path certificate (Proposition~\ref{prop:gensound}) and independently verify its axiom profile.

\noindent\textbf{Software stack.}
\begin{itemize}[leftmargin=2em, topsep=2pt, itemsep=1pt]
\item \emph{Lean:} Lean~4 v4.30.0-rc1 (as declared in \texttt{lean-toolchain}; shown explicitly in Appendix~\ref{app:axiomaudit} and Section~\ref{sec:implementation}).
\item \emph{Mathlib:} pinned via \texttt{lake-manifest.json} to the version compatible with the above Lean toolchain. The audited results depend only on Mathlib's \texttt{Logic}, \texttt{Data.Rat.Defs}, \texttt{Data.Fintype.Basic}, \texttt{Data.List.Defs}, \texttt{Algebra.Order.BigOperators.Group.Finset}, and \texttt{Tactic} modules.
\item \emph{Build system:} Lake~5.x (bundled with the toolchain).
\item \emph{ML pipeline:} Python~3.11, PyTorch; concrete embedding model and prompt templates are not load-bearing on the axiom audit (they set the numerical $\beta$ inputs but do not affect the generator's axiom dependencies).
\end{itemize}

\noindent\textbf{Repository layout.}
\begin{itemize}[leftmargin=2em, topsep=2pt, itemsep=1pt]
\item \texttt{patent\_verification/lean-toolchain}: pinned Lean version.
\item \texttt{patent\_verification/lakefile.toml}: Lake project configuration.
\item \texttt{patent\_verification/PatentVerification/ClaimNode.lean}: atomic limitations, dependency map, \texttt{claimWeight}.
\item \texttt{patent\_verification/PatentVerification/MatchStrength.lean}: bounded-rational match-strength type and discretization.
\item \texttt{patent\_verification/PatentVerification/DAGAcyclicity.lean}: \texttt{topoDepth} and the machine-checked \texttt{dag\_acyclic} theorem.
\item \texttt{patent\_verification/PatentVerification/Coverage.lean}: \texttt{computeEff}, \texttt{weightedCoverage}, \texttt{coverage\_in\_range}, \texttt{propag\_proof}, \texttt{eff\_cases}, \texttt{eff\_le\_score}.
\item \texttt{patent\_verification/PatentVerification/Certificate.lean}: \texttt{ProofCertificate} structure, \texttt{generateCertificate} compiled generator.
\item \texttt{patent\_verification/PatentVerification/CaseStudy.lean}: per-construction score tables (\texttt{scores\_I1\_broad}, \texttt{scores\_I2\_narrow}), their \texttt{ScoreValid} witnesses, and the two concrete certificate instantiations (\texttt{cert\_I1\_broad}, \texttt{cert\_I2\_narrow}), together with the three \texttt{\#print axioms} invocations whose verbatim output is reproduced in Appendix~\ref{app:axiomaudit}.
\item \texttt{patent\_verification/generate\_lean.py}: the Python driver that turns a JSON claim decomposition into the auto-generated Lean modules.
\end{itemize}

\noindent\textbf{What the build reproduces.} Running \texttt{lake build} from the \texttt{patent\_verification/} directory builds all modules and emits the three \texttt{\#print axioms} information messages, which should each read\\
\texttt{\ \ [propext, Classical.choice, Quot.sound]}\\
verbatim (Appendix~\ref{app:axiomaudit}). Any deviation (in particular, \texttt{sorryAx} or \texttt{Lean.ofReduceBool} appearing in the output) causes the certificate to be rejected under Definition~\ref{def:validcert}(ii).

\noindent\textbf{What is illustrative only.} All Lean~4 code blocks in the body of the paper shaded \colorbox{orange!20}{orange} are \emph{illustrative} and are not part of the compiled development; only the blocks shaded \colorbox{green!15}{green} correspond to compiled code. The prototype's ML layer and the generators for Algorithms~2--6 are not part of the reproducibility artifact: they are prototype components whose certificate outputs (when provided) are independently kernel-checked and $\Omega$-audited, but whose source is not the subject of this paper's formal guarantees.

\end{document}